\documentclass[twoside]{article}

\usepackage[accepted]{aistats2024}

\usepackage[round]{natbib}
\renewcommand{\bibname}{References}
\renewcommand{\bibsection}{\subsubsection*{\bibname}}

\usepackage{amsmath,amsfonts,bm}

\def\eqref#1{equation~\ref{#1}}

\def\1{\bm{1}}

\DeclareMathAlphabet{\mathsfit}{\encodingdefault}{\sfdefault}{m}{sl}
\SetMathAlphabet{\mathsfit}{bold}{\encodingdefault}{\sfdefault}{bx}{n}

\newcommand{\E}{\mathbb{E}}

\newcommand{\R}{\mathbb{R}}

\usepackage{amsmath,amssymb,amsthm,amsfonts,booktabs,color,doi,graphicx,latexsym,url,xcolor,xspace,algpseudocode,algorithm}
\usepackage{enumitem}

\newcommand{\squeeze}{\textstyle}

\usepackage{todonotes}
\theoremstyle{plain}
\newtheorem{theorem}{Theorem}[section]

\newtheorem{lemma}[theorem]{Lemma}
\newtheorem{corollary}[theorem]{Corollary}
\theoremstyle{definition}

\theoremstyle{remark}

\newtheorem*{rep@theorem}{\rep@title}
\newcommand{\newreptheorem}[2]{%
\newenvironment{rep#1}[1]{%
 \def\rep@title{#2 \ref{##1}}%
 \begin{rep@theorem}}%
 {\end{rep@theorem}}}

 \newreptheorem{theorem}{Theorem}
\newreptheorem{lemma}{Lemma}
\newreptheorem{proposition}{Proposition}
\newreptheorem{claim}{Claim}
\newreptheorem{fact}{Fact}

\newcommand{\tg}{\widetilde{g}}
\newcommand{\tih}{\widetilde{h}}
\usepackage[font=scriptsize,labelfont=bf]{caption}
\usepackage[noindentafter]{titlesec}

\newcommand{\Hide}[1]{{}}

\renewcommand{\paragraph}[1]{ \noindent \textbf{#1}}
\usepackage{hyperref}
\hypersetup{hidelinks}
\usepackage[capitalize,noabbrev]{cleveref}

\begin{document}

\twocolumn[
\runningtitle{Communication-Efficient Federated Learning With Data and Client Heterogeneity} 

\aistatstitle{Communication-Efficient Federated Learning \\ With Data and Client Heterogeneity}

\aistatsauthor{ Hossein Zakerinia\textsuperscript{1} \And Shayan Talaei\textsuperscript{2} \And Giorgi Nadiradze \textsuperscript{3} \And Dan Alistarh\textsuperscript{1} }

\runningauthor{ Hossein Zakerinia, Shayan Talaei, Giorgi Nadiradze, Dan Alistarh}

\aistatsaddress{ \textsuperscript{1} Institute of Science and Technology Austria (ISTA) \\
 \textsuperscript{2} Stanford University\\
  \textsuperscript{3} Aptos Labs} ]

\begin{abstract}
  Federated Learning (FL) enables large-scale distributed training of machine learning models, while still allowing individual nodes to maintain data locally. 
    However, executing FL at scale comes with inherent practical challenges: 
    1) heterogeneity of the local node data distributions, 
    2) heterogeneity of node computational speeds (asynchrony), 
    but also 3) constraints in the amount of communication between the clients and the server.   
    In this work, we present the first variant of the classic federated averaging (FedAvg) algorithm 
    which, at the same time, supports data heterogeneity, partial client asynchrony, and communication compression. 
    Our algorithm comes with a novel, rigorous analysis showing that, in spite of these system relaxations, 
    it can provide similar convergence to FedAvg in interesting parameter regimes. 
    Experimental results in the rigorous LEAF benchmark on setups of up to $300$ nodes show that our algorithm ensures fast  convergence for standard federated tasks, improving upon prior quantized and asynchronous approaches. 
\end{abstract}

\section{INTRODUCTION}

In Federated learning (FL)~\citep{konevcny2016federated, mcmahan2017communication}, multiple clients, orchestrated by a central authority, cooperate to jointly optimize a machine learning model given their local data. 
The basic FL algorithm is \emph{federated averaging (FedAvg)}~\citep{mcmahan2017communication}, in which a central authority periodically communicates a model to all clients; clients optimize this model locally, and send back the resulting updates to the server, which incorporates them, proceeding to the next iteration. This approach is effective in practice~\citep{li2020federated}, and motivates a rich line of research analyzing its convergence~\citep{stich2018local, haddadpour2019convergence}, as well as improved variants~\citep{reddi2020adaptive, karimireddy2020scaffold, li2021canita}. 

Despite its popularity, scaling federated learning runs into a number of fundamental challenges~\citep{kairouz2021advances}. 
One such obstacle is \emph{data heterogeneity}: the fact that the clients' local data distributions may be different, 
which can lead to difficulties in optimization~\cite{fednova, karimireddy2020scaffold}. 
A second barrier is \emph{node heterogeneity}: as practical deployments contain large node counts, 
which may execute \emph{at different speeds}, it may be infeasible for a central server to orchestrate rounds across participants~\citep{bonawitz2019towards, wu2020safa, FedBuff}. 
The third barrier is the \emph{communication cost} of the parameter updates~\citep{kairouz2021advances}, which can overwhelm communication-limited clients~\citep{jhunjhunwala2021adaptive,li2021canita,wang2022communication}. 

In a scalable FL system, all these three barriers need to be mitigated: 
for instance, communication-reduction may not be effective if the server has to \emph{synchronously} wait for all of the clients to complete a communication round. 
Yet, supporting all these system relaxations jointly is extremely challenging: it is known~\citep{fednova} that one cannot allow both \emph{data heterogeneity} and \emph{client asynchrony} in full generality without impacting the objective; 
at the same time, none of the existing communication-efficient methods support asynchrony~\citep{jhunjhunwala2021adaptive,li2021canita,wang2022communication}. 
Thus, it is interesting to ask to what extent communication-compression, asynchrony, and heterogenous data \emph{can be jointly supported} in FL.

\paragraph{Contribution.} We address this question by proposing an algorithm for \textbf{Qu}antized \textbf{A}synchronous \textbf{F}ederated \textbf{L}earning called QuAFL, which is an extension of FedAvg supporting heterogeneous data, communication compression, and partial asynchrony. 
We provide a rigorous theoretical analysis of its convergence, showing that it asymptotically matches FedAvg in interesting parameter regimes, and experiments showing that it can also lead to practical gains. 
 
\paragraph{Overview.} 
At a high level, QuAFL follows the structure of FedAvg: in each ``logical round,'' the server samples $s$ clients uniformly at random, and sends them a (compressed) copy of its current model. 
As soon as a client receives the server's message, it performs two steps: first, it replies to the server with a (compressed) copy of its local progress, obtained via optimization steps on its local data, since its last server interaction.
Second, the client adopts the server's model and will proceed to perform up to $K \geq 1$ local optimization steps on it in the future, at its own local speed. 
At the end of the round, the server collects the clients' messages and integrates them into its parameter estimate. 

One key difference from FedAvg is that, in QuAFL, the entire process is \emph{partially-asynchronous}: 
clients perform local steps each at their own speed, independently of the server's round structure, on their local version of the parameters. 
Thus, when sampled by the server, a client may still be in the middle of performing its $K$ local steps since its last interaction, or may not yet have completed any local steps at all! 
In QuAFL, the clients' progress may be partial---as a client may not have completed its $K$ local steps---and is always computed on a stale version of the parameter, previously sent by the server. 
The second key difference is that, in QuAFL, all client-server communication is compressed using a fast customized quantizer.

There are two analytical challenges in this setting: the first is in showing that the optimization process can still converge in this highly-decoupled setting, in which clients proceed at different speeds, can be interrupted asynchronously by the server, and sometimes do not make any progress at all. 
The second challenge is to interface asynchrony with communication compression: as detailed later, using standard quantizers~\citep{alistarh2016qsgd, karimireddy2019error} induces error proportional to the second-moment gradient bound, which leads to both poor practical performance, and major difficulties in the analysis.   

Our analysis circumvents these obstacles, and shows that QuAFL can provide surprisingly strong convergence guarantees, which match those of FedAvg in certain parameter regimes. 
We achieve this via a new potential argument, which roughly shows that, under standard assumptions, the discrepancy between the client and server models is always bounded, and by leveraging an instance of position-aware quantization~\citep{davies2021new}, which has the property that the compression error only depends on the \emph{distance} between the models at the server and the clients.  
Our analysis approach controls the ``noise'' due to model inconsistency precisely, 
ensuring that local models are close enough to allow correct encoding and decoding using the customized positional quantizer. 

We validate our algorithm experimentally in the rigorous LEAF benchmark~\citep{caldas2018leaf}, on a range of standard tasks. 
We show that QuAFL can compress updates by more than $3\times$ without significant loss of convergence, 
and can even withstand a large fraction of ``slow'' clients submitting infrequent or even no updates, also in non-i.i.d. data settings.  
Moreover, in settings where client computation speeds are heterogenous, QuAFL provides end-to-end speedup in terms of iteration times, since the server can progress without waiting for all clients to complete their local computation, and is also competitive with asynchronous FL approaches such as FedBuff~\citep{FedBuff}. 

\paragraph{Related Work.} There has been significant work on communication-compression for FedAvg~\citep{philippenko2020bidirectional, reisizadeh2020fedpaq, jin2020stochastic, haddadpour2021federated}.  
However, virtually all prior work considers \emph{synchronous} iterations.  
~\cite{reisizadeh2020fedpaq} introduced a variant of FedAvg which supporting standard compressors, and provides convergence bounds under the assumption of i.i.d. client data.~\cite{jin2020stochastic} examined the signSGD quantizer~\citep{seide2014sgd1bit} for FedAvg, providing convergence guarantees; however, the rates are polynomial in the 
\emph{model dimension} $d$, rendering them less practically meaningful.~\cite{haddadpour2021federated} proposed a family of algorithms with communication-compression; yet, to prove convergence in the challenging heterogeneous-data setting, they require very strong technical assumptions on quantized gradients~\citep[Assumption 5]{haddadpour2021federated}.~\cite{chen2021communication} also considered update compression, but under convex losses, coupled with a strong second-moment bound assumption on the gradients. 
Finally,~\cite{jhunjhunwala2021adaptive} adapt the degree of compression during the execution, proving convergence only under i.i.d. data sampling. 
 \emph{In sum, all prior work on compression for FedAvg requires at least one non-standard assumption. By contrast, our analysis works for non-convex losses, non-i.i.d. data, without second-moment gradient bounds. In addition, we support partial client asynchrony.} 

A complementary approach has been to investigate FL optimizers with faster convergence~\citep{mishchenko2019distributed, karimireddy2020scaffold},  or adaptive optimizers~\citep{reddi2020adaptive, tong2020effective}. 
These approaches can be compatible with communication-compression~\citep{gorbunov2021marina, li2021canita, wang2022communication}. 
Specifically, for non-convex losses, MARINA~\citep{gorbunov2021marina} offers theoretical guarantees both in terms of convergence and bits transmitted. However, MARINA is synchronous; moreover, it periodically computes full gradients and transmits uncompressed model updates, and requires complex synchronization and variance-reduction to compensate for quantization noise. 
DASHA~\citep{tyurin2022dasha} proposed a family of theoretical methods which extend MARINA with Momentum Variance Reduction (MVR)~\citep{cutkosky2019momentum}, partially relaxing the coupling between the server and the workers. 
\emph{By contrast to this work, we focus on obtaining a practical algorithm in a highly-decoupled model, with competitive convergence relative to vanilla FedAvg: we  always transmit compressed, low-precision messages, and consider  asynchronous communication and client progress.}

FedBuff~\citep{FedBuff} is a state-of-the-art \emph{practical} approach for asynchronous FL, 
where nodes aggregate their updates asynchronously in a shared buffer; 
once the buffer is full, the server updates the global model and communicates it. 
\emph{Our convergence bounds are competitive to FedBuff, but in a more general setting, as we do not assume a gradient bound.}  
Experimentally, QuAFL achieves better performance relative to FedBuff in the non-i.i.d. case: 
the intuitive reason is that, in this case, slower clients will consistently contribute less frequently to the buffer, 
meaning that convergence could be ``skewed'' towards fast clients. 
Parallel work by~\cite{koloskova2022sharper} and~\cite{mishchenko2022asynchronous} provide sharp convergence bounds for asynchronous SGD in a related but different model, with arbitrary worst-case delays. 
Specifically, they prove convergence rates that are similar to ours in the case of a single client sampled at a time. 
By contrast, our work considers a \emph{probabilistic model} on the delays, similar to~\cite{cannelli2020asynchronous}. 
The two models are incomparable: we allow the worst-case delay to be unbounded, but assume that each client $i$ proceeds at an expected speed $H_i$. 
In addition, in our algorithm, the clients can be interrupted by the server during their local computation, which leads to further difficulties in the analysis, and practical improvements in terms of waiting times.  
FedNova \citep{fednova} works in a \emph{synchronous} model, but it allows each node to perform a different number of local steps at each round. One main difference between our work and FedNova is that in the asynchronous setting, the $(n-s)$ clients which are not participating in a round can still perform local updates at their own speed, while in the synchronous scenario of FedNova, these $(n-s)$ clients will be idle. 
We show that asynchrony can result in faster wall-clock time convergence. Similarly, FedNova requires that, at each round, all participating nodes perform at least one local step, which results in all nodes having to wait for the slowest client; we allow nodes to perform zero steps at a round. Additionally, the bounded dissimilarity assumption in FedNova is stronger than the standard literature definition, which we use in our analysis.

\section{THE ALGORITHM}
\subsection{System Overview}
\label{sec:model}

\paragraph{Optimization Setting.} 
We assume a distributed system with one coordinator and $n$ workers, jointly minimizing a $d$-dimensional,   differentiable  function $f: \R^d \rightarrow \R$.  
We consider empirical risk minimization (ERM), where data samples are located at the $n$ nodes.
Each agent $i$ has a local function $f_i$ associated to its own data partition, i.e   $\forall x \in \mathbb{R}^d$: $f(x)=\sum_{i=1}^n {f_i(x)}/{n}.$
The goal is to converge on a model $x^*$ which minimizes the empirical loss. 
Clients run a distributed variant of SGD, coordinated by the central node. 
Each client $i$ is able to obtain  \emph{unbiased stochastic gradients} 
$\tg_i$ of its own local function $f_i$, i.e. $\E [\tg_i(x)] = \nabla f_i(x)$, sampling i.i.d. from its local distribution.  

\paragraph{System Model.} 
Our algorithm will follow the general pattern of FedAvg, in that the server periodically polls a subset of clients, sending them its model. 
However, the interaction pattern is \emph{(partially) asynchronous}: when contacted, clients immediately communicate their local progress since the last interaction, even though it may be partial, and afterwards proceed to take local steps on the new model communicated by the server, until their next interaction. 
Moreover, the clients themselves may progress at \emph{heterogeneous speeds}: 
the number of local steps taken by client $i$ between server interactions $t$ and $t + 1$ is a random variable $\mathcal{H}_i$, 
taking values in $\{0,1,2,\ldots, K\}$, where $K$ is an upper bound on how many steps a client can take in isolation.  
 \emph{Our only assumption is that, for each client $i$ and every server interaction, the \emph{expected value} of $\mathcal{H}_i$, denoted by $H_i$, exists and is $> 0$.}  
 That is, on average, each client has a fixed speed, and makes non-zero progress. 
 \emph{However, clients may progress at different speeds, and the individual step distributions $\mathcal{H}_i$ can be completely different.} 
\emph{We emphasize that $\mathcal{H}_i$ \emph{can be $0$}, meaning that $i$ has taken no steps between  interactions.}

\begin{algorithm*}[!ht]
\caption{Pseudocode for the QuAFL Algorithm at Server and Clients.}
\label{algo:quafl}
\footnotesize
\% Initial models $X_0=X^1=X^2=...=X^n = 0^d$, number of local steps $K$\\
\% Encoding ($Enc(A)$) and decoding ($Dec(B, Enc(A))$) functions, with common parametrization.\\ 
\% For each client $i$, the expected local steps  between two consecutive server interactions is $H_i$. \\ 
\% For each client $i$, we define weights $\eta_i = \frac{H_{\min}}{H_i}$ where $H_{\min}$ is the minimum speed among $H_i$.\\
\% \textbf{At the Server}:
\begin{algorithmic}[1]
\For{$t=0$ \textbf{to} $T-1$} {\color{blue} \qquad \qquad  \qquad \qquad \qquad \space \% Each round takes a fixed amount of time.}
    \State {Server chooses $s$ clients uniformly at random, let $S$ be the resulting set}.
    \ForAll{clients $i \in S$}  
        \State Server sends $Enc(X_t)$ to the client $i$.
        \State Server receives $Enc(Y^i)$ from client $i$ {\color{blue} \quad \% $Y^i$ is the client's model with the progress since the last interaction.}
        \State $Q(Y^i) \gets Dec(X_t, Enc(Y^i))$ {\color{blue} \qquad \qquad  \quad \% Decodes quantized client messages relative to  $X_t$} 
    \EndFor
    \State $X_{t+1} = \frac{1}{s+1} X_t + \frac{1}{s+1} \sum_{i \in S} Q(Y^i)$
\EndFor
\end{algorithmic}

\% \textbf{At Client} $i$:\\
\% Upon (asynchronous) contact from the server run \Call{InteractWithServer}{} \\
\% \textbf{Local variables:} \\
    \%  $X^i$ stores the base client model, following the last server interaction. Initially $0^d$.\\ 
    \% $\tih_i$ accumulates local gradient steps since last server interaction, initially $0^d$.  

\begin{algorithmic}[1]
\footnotesize
\Function{InteractWithServer}{}
    \State ${MSG}_i
    \gets Enc(X^i - \eta \eta_i \tih_i)$ {\color{blue}  \qquad \qquad  \qquad \quad  \% Client $i$ compresses its local progress since last contacted.}
    \State Client sends $MSG_i$ to the server.
    \State Client receives $Enc(X_t)$ from the server, where $t$ is the current server time.
    \State $Q(X_t) \gets Dec(X^i, Enc(X_t))$ {\color{blue}  \qquad \qquad  \qquad \space \% Client decodes the message \emph{using its own model} as reference point.}
    \State $X^i = \frac{1}{s+1} Q(X_t) + \frac{s}{s+1} (X^i - \eta \eta_i \tih_i)$ {\color{blue} \qquad \quad \space \% The client then updates its local model}
    \State $ \Call{LocalUpdates}{X^i, K}$   {\color{blue} \qquad \qquad  \qquad \qquad \space \% Finally, The client goes back to compute new local updates.}
    \State \Call{Wait}{~}
\EndFunction
\end{algorithmic}
\begin{algorithmic}[1]
\footnotesize
\Function{LocalUpdates}{$X^i$, $K$}
    \State $\tih_i = 0$ {\color{blue}  \qquad \qquad   \qquad \qquad\qquad \qquad  \qquad \qquad\% local gradient accumulator}
    \For{$q=0$ to $K-1$}
            \State $\tih_i^q = \tg_i(X^i - \eta \sum_{\ell=0}^{q-1}\tih_i^{\ell}) $ {\color{blue} \qquad \qquad\qquad \space \space \% compute the $q$th local gradient}
            \State $\tih_i =  \tih_i+ \tih_i^q$ {\color{blue} \qquad \qquad\qquad \qquad\qquad \qquad\% add it to the accumulator}
    \EndFor
\EndFunction

\end{algorithmic}
\end{algorithm*}

\subsection{Algorithm Description}
\label{sec:algo-description}

\paragraph{Overview.} 
The pseudocode for QuAFL is given in Algorithm~\ref{algo:quafl}.  
From the server's perspective, the execution is similar to FedAvg: 
we execute logical ``rounds,'' where in each round $t$ the server polls a subset of $s$ workers, sending them a compressed version of its model $Enc(X_t)$. 
However, the server \emph{does not wait for workers to perform local steps over $X_t$ in this round}: 
instead, it immediately receives each worker's local progress $Enc(Y^i)$ \emph{since worker $i$'s last server interaction}. 
(Thus, the server will observe progress on $X_t$ from clients only on their \emph{next} interaction.) 
The received progress is de-quantized, and integrated into the server's local model \emph{via weighted averaging}. 

From the other perspective, a contacted worker $i$ could be either idle when polled by the server, having completed its $K$ steps since the last server interaction, 
or still in the process of performing local steps. In either case, the worker \emph{immediately} quantizes its possibly-partial local progress $Y^i$ since the last server contact, and sends  it in quantized form $Enc(Y^i)$  to the server. 
The worker then decodes the server's quantized message $Q(X_t)$, and updates its own local model $X^i$ correspondingly via weighted averaging.  
The worker then starts performing $K$ local steps on top of the updated local parameters, until its next server interaction. 
QuAFL relaxes the FedAvg pattern as follows.

\paragraph{Non-blocking Communication.} 
A key limitation of standard FedAvg is that the server has to wait for all contacted workers 
to compute their $K$ updates and transmit them each round, before moving on. 
In QuAFL, communication is \emph{non-blocking}: the contacted worker node $i$ \emph{immediately returns (a quantized version of) its local progress} since the last interaction to the server, without performing any computation. 
We emphasize that this progress may be incomplete, or even zero for some clients, and that it is computed with respect to the server's previously-communicated model, not the one just received. 
Conversely, the server does not wait for clients to take steps on its current model $X_t$; it will observe these updates in future interactions. 
This significantly reduces the server's waiting times, and allows QuAFL to ``pipeline'' several communication rounds over a fixed wall-clock time,
at the cost of supporting inconsistent client models in the analysis and in practice. 

\paragraph{Model Averaging.} 
This parameter inconsistency is handled via the server- and client-side averaging mechanism, which differs from the standard FedAvg iteration. 
Specifically, in each round, the server model ``weight'' is $\frac{1}{s+1}$, and it gets averaged with $s$ clients. 
Each of them, plus the server itself, gets a $\frac{1}{s+1}$ ``fraction'' of the server model. Thus, the server's model is evenly distributed among $s+1$ participants; in turn, the server receives a $\frac{1}{s+1}$ fraction of the local models of each client node it interacts with.  Crucially, this average does not change, but the local models \emph{move closer to the mean}. 
This idea will be reflected in our analysis, which works by first tracking convergence \emph{at the mean of client models}, since the change in average only depends on the stochastic gradient updates at the round, and on the prior mean. We show that mean convergence  implies  convergence at the server.

\paragraph{Partial Client Asynchrony.} 
As discussed in Section~\ref{sec:model}, clients progress at different speeds. Specifically, each client $i$ is assumed to take $H_i$ steps between two server interactions, 
in expectation. 
We support this in QuAFL as follows: the only shared information, maintained by the server, is $H_{\min}$, the lowest ``speed'' of any participating client. 
To address the difference in average speeds, each client will ``dampen'' its transmitted progress before transmitting it to the server, by a factor of $\eta_i = H_{\min} / H_i$.   
We will show that this is sufficient to maintain balance in the optimization objective, without losing any client privacy, 
as the server does not need to be aware of client speeds.

\paragraph{Fully-Quantized Communication.} 
For quantization, we employ a customized version of the lattice-based quantizer of~\cite{davies2021new}; its parametrization is described formally in Section~\ref{sec:analysis}. 
Quantization works via an encoding function $Enc(A)$, which maps $A$ to its quantized representation. 
To ``read'' message $Enc(A)$, a node calls the symmetric $Dec(B, Enc(A))$ function, 
which allows ``decoding'' of $Enc(A)$ with respect to a reference $B$, returning output $Q(A)$. 
As evident from the pseudocode, server-client communication in QuAFL is always quantized, as opposed to  prior methods~\citep{gorbunov2021marina}, 
where the server still transmits full-precision updates. 

\paragraph{The Issue with Standard Quantizers.} We emphasize that compressing via standard quantizers~\citep{alistarh2016qsgd}, 
induces error proportional to the model's norm, which is in principle unbounded. 
Thus, applying direct quantization to existing methods is not theoretically-justified, and experimentally it does not always lead to good results (see Section~\ref{sec:experiments}). 
Prior work in decentralized optimization~\cite{lu2020moniqua, nadiradze2021asynchronous} 
addressed this by only transmitting \emph{model updates}. 
This requires both additional memory at the client, or an unrealistic second-moment gradient bound. 
Our approach avoids both issues: it requires no extra memory, 
and makes no extra assumptions.

\section{CONVERGENCE ANALYSIS}
\label{sec:analysis}
\subsection{Analytical Assumptions}
We begin by stating the assumptions we make in the theoretical analysis of our algorithm. 
We assume the following for the global loss function $f$, the client losses $f_i$, and their stochastic gradients $\tg_i$:
\begin{enumerate}[itemsep=0.5pt, topsep=0.5pt, partopsep=-1pt, leftmargin=2em]
\item \textbf{Uniform Lower Bound:} 
    There exists $f_* \in \mathbb{R}$ such that $f(x) \geq f_*$ for all $x \in \mathbb{R}^d$. 

\item \textbf{Smooth Gradients}: 
    For any client $i$, the  gradient $\nabla f_i(x)$ is $L$-Lipschitz continuous for some $L>0$, i.e. for all $x, y\in \mathbb{R}^d$:
    $    \label{eqn:lipschitz_assumption_f_i}
        \|\nabla f_i(x) - \nabla f_i(y)\| \leq L\|x-y\|. 
    $
    
    \item \textbf{Bounded Variance}: 
    For any client $i$, the variance of the stochastic gradients is bounded by some $\sigma^2>0$, i.e. for all $x \in \mathbb{R}^d$: 
    $   \label{eqn:variancebound_i}
        \E \, \Big \|\tg_i\left(x\right)-\nabla f_i\left(x\right)\Big\|^2 \le {\sigma}^2.$
    
    \item \textbf{Bounded Dissimilarity}: There exist  constants $G^2\ge0$ and $B^2 \ge 1$, s.t. $ \forall x \in \mathbb{R}^d$: 
    \begin{equation*} 
        \squeeze
        \label{eqn:varsigmabound}
        \sum_{i=1}^n \frac{\|\nabla f_i\left(x\right)\|^2}{n} \le G^2 + B^2 \|\nabla f\left(x\right) \|^2.
    \end{equation*}
\end{enumerate}

 The first three conditions are  universal in distributed non-convex stochastic optimization, whereas the fourth encodes the fact that there must be a bound on the amount of divergence between the local distributions at the nodes in order to allow for joint optimization~\citep{karimireddy2020scaffold, jin2020stochastic, gorbunov2021marina}. 

In addition, we make the following assumption on the local progress performed by each node: 
\begin{enumerate}[itemsep=0.5pt, topsep=0.5pt, partopsep=-1pt]
\setcounter{enumi}{4}
\item \textbf{Probabilistic Progress:} 
For each client $i$, the expected number of local steps taken since the last interaction when contacted by the server is $0 < H_i \leq K$. 
\end{enumerate}

Clearly, a condition of this type is \emph{necessary}: if a client makes zero progress on average, we cannot converge to a consistent objective in the heterogeneous setting. 

\paragraph{Quantization Procedure.} 
Please recall the semantics of our quantization procedure, as described in Section~\ref{sec:algo-description}.
In this context, the quantizer has the following guarantees~\citep{davies2021new} (Lemma 23):
\begin{lemma}{(Lattice Quantization)}
\label{lem:quant}
Fix parameters $R$ and $\gamma > 0$. There exists a  quantization procedure defined by an encoding function $Enc_{R,\gamma} : \mathbb{R}^d \rightarrow {\{0,1\}}^*$ and a decoding function $Dec_{R,\gamma}=\mathbb{R}^d \times {\{0,1\}}^* \rightarrow \mathbb{R}^d$ such that, for any vector $x \in \mathbb{R}^d$ which we are trying to quantize, and any vector $y$ which is used by decoding, which we call the \emph{decoding key}, if $\|x-y\| \le R^{R^d}\gamma$ then
with probability at least $1-\log \log (\frac{\|x-y\|}{\gamma})O(R^{-d})$, the function $Q_{R, \gamma}(x)=Dec_{R, \gamma}(y, Enc_{R, \gamma}(x))$ has the following properties:
\begin{enumerate}[topsep=0.5pt, itemsep=-2pt, leftmargin=2em]
    \item (Unbiased decoding)  \, $\E[Q_{R, \gamma}(x)]=\E[Dec_{R, \gamma}(y, Enc_{R, \gamma}(x))]=x$;
    \item (Error bound) \, $\|Q_{R, \gamma}(x)-x\| \le (R^2+7)\gamma$;
    \item (Communication bound) \, $O\left(d \log(\frac{R}{\gamma}\|x-y\|\right))$ bits are needed to send $Enc_{R, \gamma}(x)$.
\end{enumerate}
\end{lemma}

\subsection{Main Results}

Our proof strategy will be to show that the clients' local models stay close to the server's. 
This coupling is used both to show that the models converge jointly, but also that we can successfully apply the quantizer. 
Let $\mu_t={(X_t + \sum_{i=1}^n {X^i})}/{(n+1)}$ be the mean over all the node models in the system at a given $t$. Our main result shows the following: 

\begin{theorem} \label{thm:quantized}
Let $H_i > 0$ be client $i$'s speed, and $H_{\min}$ be the minimum client speed. 
Assume the total number of server rounds is $T \ge \Omega(n^3)$, the learning rate $\eta=\frac{n+1}{H_{\min}\sqrt{sT}}$, and the quantization parameters $R$, and $\gamma$ satisfy that 
$\gamma^2=\frac{\eta^2}{(R^2+7)^2}\left((\sum_{i=1}^{n}\frac{H_i^2}{n H^2})\sigma^2 + 2KG^2 + \frac{f(\mu_0) - f_*}{L}\right)$, and $R=2+T^{\frac{3}{d}}$. 
Then, we have that  Algorithm \ref{algo:quafl} converges at the following rate: 
\begin{align*}
\squeeze
  \frac{1}{T} \sum_{t=0}^{T-1} &\E\|\nabla f(\mu_t)\|^2 \le \frac{4(f(\mu_0)-f_*)}{\sqrt{sT}}
     \\& + \frac{36KL(\sum_{i=1}^{n} \frac{\sigma^2}{n H_i^2} + \frac{2KG^2}{{H_{\min}}^2})}{\sqrt{sT}}
      \\& 
    + O\left(\frac{n^3K^2L^2 ((\sum_{i=1}^{n}\frac{{H_{\min}}^2}{n H_i^2})\sigma^2 + 2KG^2)}{{H_{\min}}^3T}\right).
\end{align*}
The algorithm uses $O\left(sT(d\log{n} + \log T\right))$ expected communication bits.
\end{theorem}

\paragraph{Discussion.} 
The result shows non-trivial trade-offs between the convergence speed of the algorithm, the variance of the local distributions (given by $\sigma$ and $G$), the sampling set size $s$, 
and individual client speeds $H_i$. We now examine it across relevant parameter regimes. 

Let us first consider the case where the clients are \emph{homogenous}, i.e. $H = H_i = H_{\min}, \forall i$.
If assume that the maximum number of local steps $K$ is constant, then we also get constant $H$ since $K \geq H$.
Then, our bound is asymptotically-optimal. 
Specifically, the first two upper bound terms achieve the ``optimal'' speedup $\sqrt{sT}$ with respect to the sampling parameter $s$ and the number of iterations $T$ for the non-convex case. The  third term contains similar ``nuisance factors'' as the second term, with the addition of the $n^3$ factor, which is directly due to asynchrony. 
Crucially, this larger third term is divided by $T$, as opposed to $\sqrt T$; since $T$ is our asymptotic parameter, the whole third term is commonly assumed to be negligible~\citep{lu2020moniqua}. 
Thus, in this case essentially the entire overhead of asynchrony and quantization is ``offloaded'' onto the third term, which becomes negligible as $T$ grows. 

For equal speeds $H_i = H$ and \emph{non-constant} local steps $K$, it is reasonable to assume that $H = \Theta(K)$: on average, each client $i$ will have completed a significant fraction of its local steps on the old version of the model $X^i$ when being contacted. (Otherwise, the sampling frequency of the server is too high, and the server should simply decrease it.) 
If $H = \Theta(K)$, our algorithm also gets linear speedup with respect to the parameter $H$ in terms of reducing the impact of the variance $\sigma^2$, in the second term. 
These dependencies match some of the best known bounds for standard FedAvg, in a similar learning rate regime: without asynchrony and quantization, we asymptotically recover the FedAvg bounds~\citep{karimireddy2020scaffold}. 

Finally, take the interesting case of \emph{heterogeneous client speeds} $H_i$. 
The first term is not affected, 
and the third term remains negligible as long as the ``slowdown'' of the slowest node is not asymptotic in the total steps $T$. 
In the second term, there is a remarkable difference between the \emph{variance terms} $\sigma^2 / H_i^2$,  
which decrease proportionally to the client speeds, and the \emph{discrepancy term} $2KG^2$, which only gets decreased proportionally to the square of $H_{\min}$, 
the speed of the lowest node. This is inherent: since the local data distributions are \emph{heterogenous}, 
the system cannot make progress in the discrepancy term without contributions from the slowest node.   
As such, our results are compatible with the trade-offs between asynchrony and data heterogeneity in FedNova~\citep{fednova}.

To simplify the discussion, the Theorem statement considers a parametrization of the quantizer under which the impact of quantization noise is subsumed into the third upper bound term. 
For a more fine-grained upper bound result, please see Appendix~\ref{appendix:full-convergence}.

\paragraph{Convergence at the Server.}
We can obtain a similar bound for convergence \emph{of the server's model}, as opposed to convergence of the mean of local models. 
 \begin{corollary} \label{cor:quant}
Assume that the total number of steps is $T \ge \Omega(n^4)$, whereas all the other parameter values are identical to Theorem~\ref{thm:quantized}. 
Then, with probability at least $1-O(\frac{1}{T})$, the server's model $X_t$ converges asymptotically at the same rate as the model average. 
\end{corollary}

This bound is very similar to our main result, except for the larger dependency between the parameters $T$ and $n$. 
Intuitively, this is required in the analysis due to the additional ``mixing time'' required for the server to converge to a similar bound to the mean $\mu_t$. 
However, we do not observe such a requirement in practical experiments, and this bound is realistic given that the optimization process is usually executed over a large number of iterations. 

Thus, the discussion of convergence interesting parameter regimes remains the same as for convergence of the average. 
In sum, QuAFL can match some of the best known rates for FedAvg. We find this surprising, since our algorithm executes in a highly-decoupled environment, in which communication is non-blocking and compressed, and clients are partially asynchronous.

\subsection{Overview of the Analysis} 

We provide an overview of the proofs, outlining the main intermediate results. (The full analysis is given in the Appendix.) 
The first step is bounding the deviation between the local models and their mean. For this, we define the following potential function:  $\Phi_t = \|X_t - \mu_t\|^2 + \sum_{i=1}^{n} \|X^i - \mu_t\|^2$. We can show that this potential has the following supermartingale-type property: 
\begin{lemma} \label{lem:PhiBoundPerStepAsyncT}
For any time step $t$ we have:
 \begin{align*} %
 \squeeze
\E[\Phi_{t+1}] & \le (1 - \frac{1}{4n})\E[\Phi_t]
\\& +8s\eta^2\sum_{i = 1}^n \eta_i^2 \E\|\tih_i\|^2 + 16n({R}^2+7)^2\gamma^2. 
\end{align*}
\end{lemma}
The intuition behind this result is that potential $\Phi_t$ will stay well-concentrated around its mean, except for influences from the variance due to local steps (second term) or quantization (third term). 
With this in place, the next lemma allows us to track the evolution of the average of the local models, with respect to local step and quantization variance: 
\begin{lemma} \label{lem:mudifferenceT}
For any step $t$,
$\E\|\mu_{t+1}-\mu_t\|^2 \le 
 \frac{2\eta^2}{(n+1)^2} \E  \|\sum_{i \in S} \eta_i \tih_{i}  \|^2  + \frac{2({R}^2+7)^2\gamma^2}{(n+1)^2}.$ 
\end{lemma}
In both cases, the upper bound depends on the second moment of the nodes' local progress ($\sum_{i} \eta_i^2
\E\|\tih_i\|^2$ or $\|\sum_{i \in S} \eta_i \tih_{i}  \|^2$). This is due to the fact that the server contacts $s$ clients, which are chosen uniformly at random. Then, our main technical lemma  uses properties 
(\ref{eqn:lipschitz_assumption_f_i}), (\ref{eqn:variancebound_i}) and (\ref{eqn:varsigmabound}), to concentrate 
these quantities around the true gradient $\E \| \nabla f(\mu_t) \|^2$, where the expectation is taken over the algorithm's randomness. 
 \begin{lemma} \label{lem:sumofselectedstochasticG}
For any step $t$
\begin{align*}
\squeeze
 &\E  \|\sum_{i \in S} \eta_i\tih_{i}  \|^2 \le \frac{16s^2K^2L^2}{n} \E[\Phi_t] \\&+ 18sK((\frac{1}{n}\sum_{i=1}^{n} \eta_i^2)\sigma^2 + 2KG^2)  +64s^2K^2B^2 \E\|\nabla f(\mu_t) \|^2  
 \end{align*}
 \end{lemma}
We can use  Lemmas \ref{lem:mudifferenceT} and \ref{lem:sumofselectedstochasticG} to bound $\E\|\mu_{t+1}-\mu_t\|^2$. Similarly to Lemma \ref{lem:sumofselectedstochasticG}, we can get an upper bound for $\sum_{i} \eta_i^2
\E\|\tih_i\|^2$, and by combining it with Lemma \ref{lem:PhiBoundPerStepAsyncT} we get the following upper bound on the potential with respect to $\E \| \nabla f(\mu_t) \|^2$. 
\begin{lemma} \label{lem:PhiBoundGlobalmain}
We have that:
\begin{align*} 
\squeeze
&\sum_{t=0}^T \E[\Phi_t]
\le O(Tn^2({R}^2+7)^2\gamma^2 +K \sum_{t=0}^{T-1} \E\|\nabla f(\mu_t)\|^2)) \\&+
n^2sK\eta^2(T(\frac{\sum_{i=1}^{n} \eta_i^2\sigma^2}{n}+2KG^2).
\end{align*}
\end{lemma}
\vspace{-1mm}
Next, using the $L$-smoothness of the function $f$ (\ref{eqn:lipschitz_assumption_f_i}), we can show that
\begin{align} \label{L-smoothness}
\squeeze
\begin{split}
\E[f(\mu_{t+1})] &\le \E[f(\mu_t)]+\E\langle\nabla f(\mu_t) , \mu_{t+1}-\mu_t\rangle
     \\& + \frac{L}{2} \E\|\mu_{t+1}-\mu_t\|^2.
\end{split}
\end{align} 
\paragraph{Weighting for heterogeneous clients.} Using (\ref{L-smoothness}), and given that $
\E[\mu_{t+1}-\mu_t]=-\frac{\eta}{n+1}\sum_{i \in S} \eta_i \tih_i(X_t^i)$, we observe that the sum 
$\sum_{i=1}^n \E \langle \nabla f(\mu_t), \mu_{t+1}-\mu_t \rangle$ can be concentrated around 
$\E\| \nabla f(\mu_t) \|^2$, in similar fashion as in Lemma \ref{lem:sumofselectedstochasticG}. While bounding this quantity we need to control  $\sum_{i=1}^n \eta_i H_i (- \E\langle\nabla f(\mu_t),\nabla f_i(\mu_t) \rangle)$. If data were homogeneous, for each $i$ we had $\langle\nabla f(\mu_t),\nabla f_i(\mu_t) \rangle = \|\nabla f(\mu_t)\|^2$ and we would not need weighting. However, in the heterogeneous setting, the $f_i(\mu_t)$s are different and each of the inner products can have arbitrary sign. Thus, we choose $\eta_i = \frac{H_{\min}}{H_i}$ such that all $\eta_i H_i$s are equal, and the full term is  $-nH_{\min}\|\nabla f(\mu_t)\|^2$.

\paragraph{Final argument.} Then we can place this upper bound and the upper bound for $\E\|\mu_{t+1}-\mu_t\|^2 $ in (\ref{L-smoothness}), summing over all $T$ steps, and use Lemma \ref{lem:PhiBoundGlobalmain} to cancel out the terms containing the potential $\sum_{t=0}^T \E[\Phi_t]$ based on $\sum_{t=0}^{T-1} \E\|\nabla f(\mu_t)\|^2$ , and modulo some careful wrangling, we obtain the convergence bound in Theorem \ref{thm:convergence}. Setting the quantization parameters as stated, we get the claimed convergence rate. We discuss the impact of the quantization parameters in detail in the Appendix.

\begin{figure*}[t]
\centering
    \begin{minipage}[c]{0.43\textwidth}
        \centering
        \includegraphics[width=\textwidth, clip]{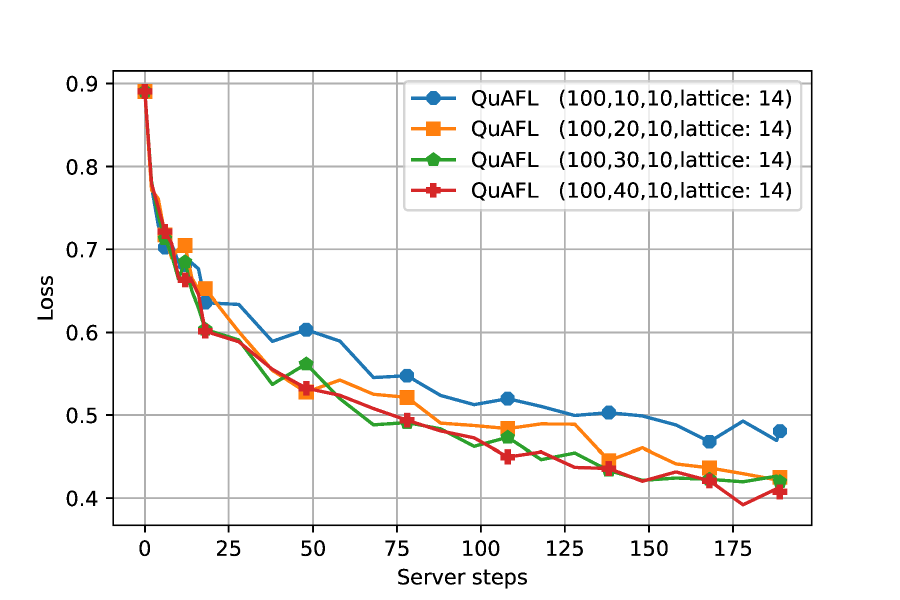}
        \vspace{-2em}
        \caption{Impact of the number of peers $s\in \{10, 20, 30, 40\}$ on convergence, for $n = 100$ clients, $14$-bit quantization, on CelebA, using non-i.i.d data.}
        \label{fig:s-impact}
    \end{minipage}~~
    \begin{minipage}[c]{0.43\textwidth}
    \centering
    \includegraphics[width=\textwidth, clip]{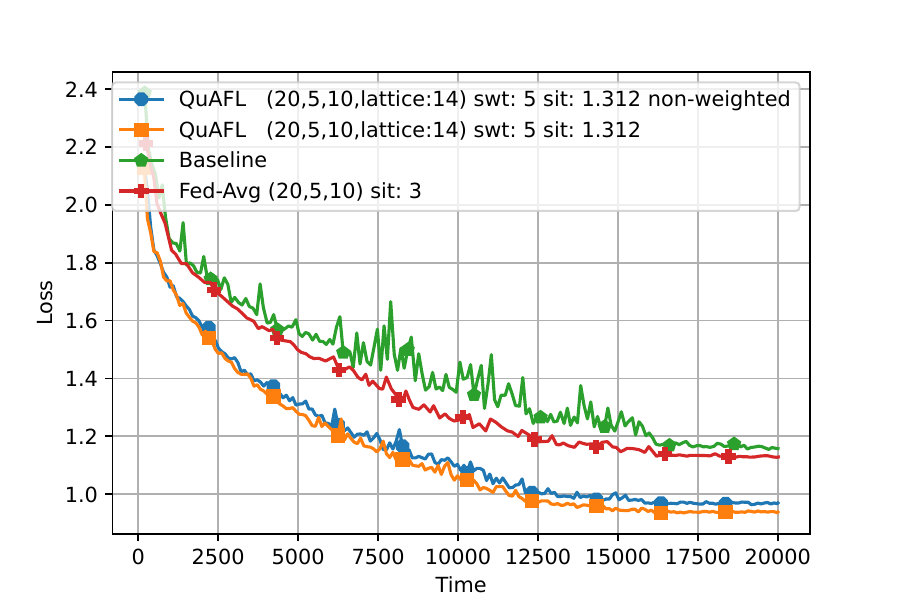}
    \vspace{-2em}
    \caption{Convergence comparison relative to simulated time between QuAFL and FedAvg for ResNet20/CIFAR10.}
    \label{fig:time}
    \end{minipage}
\vspace{-2em}

\end{figure*}

\begin{figure*}[t]
\centering
    \begin{minipage}[c]{0.43\textwidth}
        \centering
        \includegraphics[width=\textwidth, clip]{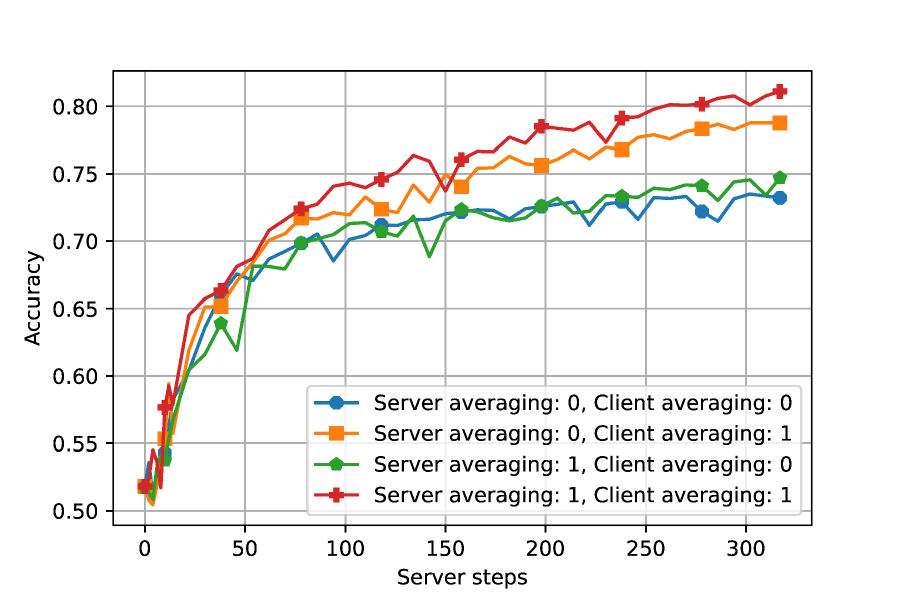}
        \caption{ The impact of averaging variants vs. validation accuracy on ResNet/CelebA, non-i.i.d data.}
        \label{fig:averaging}
    \end{minipage}~~
            \begin{minipage}[c]{0.43\textwidth}
        \centering
        \includegraphics[width=\textwidth, clip]{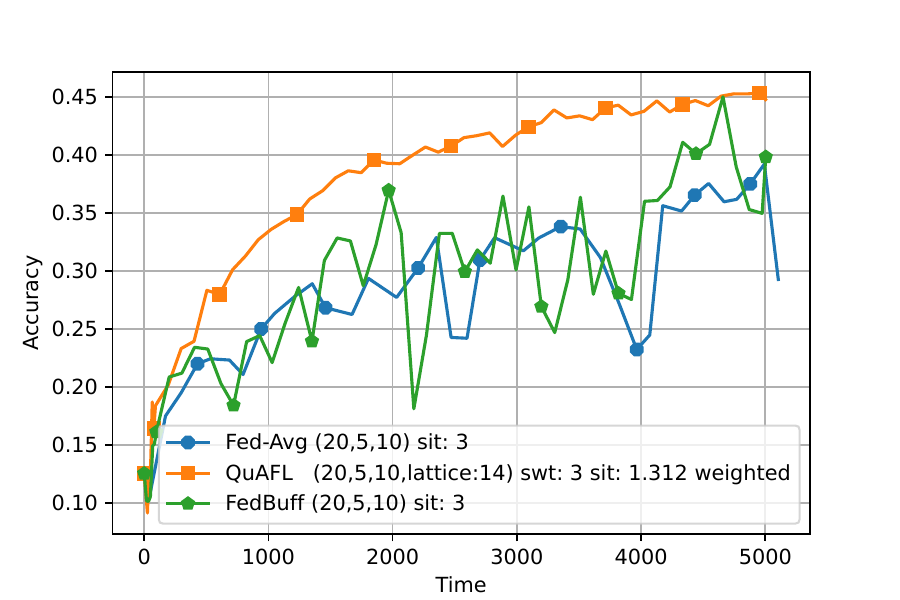}
        \vspace{-2em}
        \caption{ResNet20/CIFAR10 experiment where Fast and Slow clients have non-i.i.d. data from different classes.}
        \label{fig:non-iid}
    \end{minipage}~~

\end{figure*}
\begin{figure*}[t]
\centering
\centering
        \begin{minipage}[c]{0.43\textwidth}
        \centering
        \includegraphics[width=\textwidth,  clip]{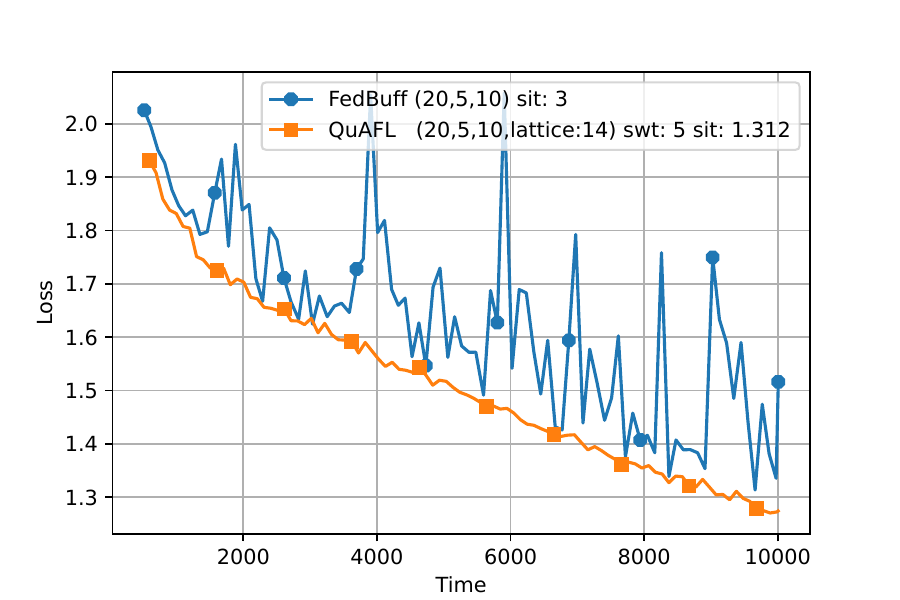}
        \caption{Experiment in which the average time per client per local step is uniformly random between 2 and 9, showing superior performance relative to FedBuff, which becomes unstable when both data and client speeds are heterogeneous (ResNet20/CIFAR10). }
        \label{fig:fedbuff_different_speeds}
    \end{minipage}~~
\centering
\begin{minipage}[c]{0.43\textwidth}
        \centering
        \includegraphics[width=\textwidth,  clip]{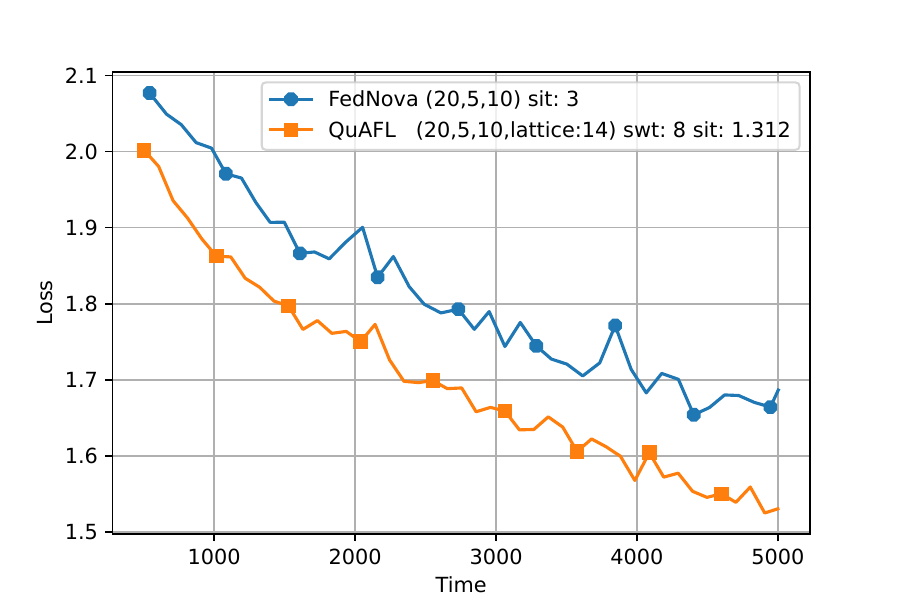}
        \caption{Comparison with FedNova in the same setup as Figure \ref{fig:fedbuff_different_speeds}. FedNova can handle different client speeds, but is synchronous, therefore there are many idle clients during training, leading to slower convergence vs time.}
        \label{fig:fednova_different_speeds}
    \end{minipage}

\end{figure*}

\vspace{-1em}
\section{EXPERIMENTAL RESULTS}
\vspace{-1em}
\label{sec:experiments}

\paragraph{Experimental Setup.} 
We implement QuAFL in Pytorch train neural networks for image classification tasks, specifically  residual CNNs~\citep{he2016deep} on the MNIST~\citep{lecun2010mnisthandwrittendigit}, Fashion MNIST~\citep{xiao2017fashion}, CIFAR-10~\citep{krizhevsky2009learning} and CelebA~\citep{liu2015faceattributes} datasets, in the rigorous benchmarking setup of LEAF~\citep{caldas2018leaf}. 
Experiments on MNIST, FMNIST and CIFAR use a fixed random split of the training set among the nodes, while the CelebA experiments are executed in a pure non-i.i.d. setting, in which the samples are split across classes, so that each client receives a non-overlapping subset of classes. 
Full experimental details are presented in the Appendix. 
We omit error bars for readability, as the variance between experiments is low and does not impact our conclusions.

 \paragraph{Goals and Metrics.} 
Experiments are described by $(n, s, K, b)$, where $b$ is the number of bits used. 
In addition, $swt$ is the server waiting time between two consecutive calls, and the server interaction time, $sit$, as the amount of time that server needs to send and receive necessary data. 
We assume  a \emph{server} and $n$ \emph{clients}, of which $s$ are chosen randomly to be sampled in a round. 
The training data is distributed among clients so that each has access to a fixed $1/n$ partition. 
We track the accuracy of the server's model on an unseen validation dataset. 
We measure \emph{loss} and \emph{accuracy} of the model with respect to \emph{simulation time}.%
     We update the both client and server models following QuAFL, and then increase the server time by $sit$. The server then waits for another interval of \emph{server waiting time (swt)} to make its next call. Unless otherwise stated, all communication is quantized using the lattice quantizer of~\cite{davies2021new}, which is simply implemented via a random rotation followed by direct quantization/dequantization.

We  run two timing experiments: 
\emph{uniform} experiments assume all clients take the same amount of time for a gradient step; 
\emph{non-uniform} timing experiments  differentiate between \emph{fast} or \emph{slow} clients. Specifically, the length of each client step is taken to be a random variable $X \sim exp(\lambda)$, where $\lambda$ is $1/2$ for \emph{fast clients} and $1/8$ for \emph{slow} clients; the expected runtime $\E(X)$ would be $2$ and $8$, respectively. In each experiment, we assumed 30\% of clients to be \emph{slow}. Unless otherwise noted, we employ the \emph{unweighted} version of QuAFL, meaning that clients set their weighting parameter to $\eta_i = 1$ in the algorithm. 
All code is available at \url{https://github.com/ShayanTalaei/QuAFL}. %

\paragraph{Results.}
Figure~\ref{fig:s-impact} examines 
the impact of the number of sampled peers $s$ when training ResNet18 on the (non-i.i.d) CelebA dataset, where  30\% of clients are slow. 
Observe that convergence speed clearly follows the ordering of the number of peers $s$, confirming our analysis. Interestingly, timings in this experiment lead to a 27\% probability that a slow client \emph{will not have taken any steps} when interacting with the server, i.e. \emph{its progress $Y_i$ is zero}. 
This shows that QuAFL is robust to such slow clients, although their proportion impacts convergence. 

In Figure~\ref{fig:time} we examine the convergence of FedAvg and QuAFL in \emph{simulated execution time}, in a system with $20$ clients, out of which $25\%$ are slow. (The \emph{Baseline} is a single slow node that performs an optimization step per round.) Here, it is evident that non-blocking communication in QuAFL leads to faster convergence in terms of wall-clock time. The figure also shows better convergence for the \emph{weighted} version of QuAFL, where agents set $\eta_i$ according to speed. 

In Figure~\ref{fig:averaging}, we examine the impact of different types of averaging on the convergence of the basic QuAFL pattern, on the non-i.i.d. CelebA dataset, with $n = 100$ clients. We clearly observe that the variant where averaging is applied \emph{both at the server and at the client performs the best}. 

In Figure~\ref{fig:non-iid}, we compare QuAFL with FedBuff, a SOTA \emph{asynchronous} FL protocol~\citep{FedBuff} and FedAvg in a non-i.i.d. experiment, in which the clients' data is created in a heterogeneous way. The experiment consists of 20 clients, of which 5 of them are slow ones, and the rest are fast ones. The task is  CIFAR10 classification, in which slow clients have access to three specific classes and the rest have the other classes. FedAvg has to wait for the slow clients to perform all of their local steps, and therefore in the same amount of time as others, it performs a small number of global steps. FedBuff is asynchronous and does not have to wait for slow clients, however, the fast clients contribute more often to the optimization, and the algorithm observes the seven classes from fast clients compared more frequently, relative to the three classes from slow clients. QuAFL on the other hand is asynchronous and balances the imbalance between different clients. As you can see in the figure, QuAFL outperforms both models in this comparison.

Figures~\ref{fig:fedbuff_different_speeds} and~\ref{fig:fednova_different_speeds} provide comparisons with FedBuff and FedNova, in a setup where both client speeds and client data are heterogenous. 
QuAFL provides faster convergence vs. time relative to FedBuff since it can handle heterogenous client speeds, whereas in FedBuff convergence is biased towards faster clients, leading to loss spikes; FedNova does not have this problem, but is \emph{synchronous} and thus incurs slow-downs due to delays required for all clients to synchronize at every round. 
We present additional results in the Appendix, specifically higher node counts (up to $300$), full-convergence experiments, as well as other tasks. 

\vspace*{-1mm}
\vspace{-0.5em}
\section{CONCLUSION AND LIMITATIONS}
\vspace{-0.5em}
We have provided the first variant of FedAvg which incorporates both asynchronous and compressed communication, and have shown that this algorithm can still provide good convergence guarantees in a setting with heterogeneous data and client speeds. 
Our analysis is extensible to more complex federated optimizers, such as gradient tracking, e.g.~\citep{haddadpour2021federated}, controlled averaging ~\citep{karimireddy2020scaffold}, heterogeneous clients~\citep{diao2020heterofl} or variance-reduced variants~\citep{gorbunov2021marina}. 
In future work, we plan to investigate the limitation of our analysis in terms of the relationship between time $T$ and the number of nodes $n$, applications to additional federated optimizers, as well as larger-scale deployments.

\bibliographystyle{apalike}
\bibliography{references.bib}

\section*{Checklist}

 \begin{enumerate}

 \item For all models and algorithms presented, check if you include:
 \begin{enumerate}
   \item A clear description of the mathematical setting, assumptions, algorithm, and/or model. [Yes]
   \item An analysis of the properties and complexity (time, space, sample size) of any algorithm. [Yes]
   \item (Optional) Anonymized source code, with specification of all dependencies, including external libraries. [Yes]
 \end{enumerate}

 \item For any theoretical claim, check if you include:
 \begin{enumerate}
   \item Statements of the full set of assumptions of all theoretical results. [Yes]
   \item Complete proofs of all theoretical results. [Yes]
   \item Clear explanations of any assumptions. [Yes]     
 \end{enumerate}

 \item For all figures and tables that present empirical results, check if you include:
 \begin{enumerate}
   \item The code, data, and instructions needed to reproduce the main experimental results (either in the supplemental material or as a URL). [Yes]
   \item All the training details (e.g., data splits, hyperparameters, how they were chosen). [Yes]
         \item A clear definition of the specific measure or statistics and error bars (e.g., with respect to the random seed after running experiments multiple times). [No]
         \item A description of the computing infrastructure used. (e.g., type of GPUs, internal cluster, or cloud provider). [No]
 \end{enumerate}

 \item If you are using existing assets (e.g., code, data, models) or curating/releasing new assets, check if you include:
 \begin{enumerate}
   \item Citations of the creator If your work uses existing assets. [Yes]
   \item The license information of the assets, if applicable. [Not Applicable]
   \item New assets either in the supplemental material or as a URL, if applicable. [Yes]
   \item Information about consent from data providers/curators. [Not Applicable]
   \item Discussion of sensible content if applicable, e.g., personally identifiable information or offensive content. [Not Applicable]
 \end{enumerate}

 \item If you used crowdsourcing or conducted research with human subjects, check if you include:
 \begin{enumerate}
   \item The full text of instructions given to participants and screenshots. [Not Applicable]
   \item Descriptions of potential participant risks, with links to Institutional Review Board (IRB) approvals if applicable. [Not Applicable]
   \item The estimated hourly wage paid to participants and the total amount spent on participant compensation. [Not Applicable]
 \end{enumerate}

 \end{enumerate}

\appendix

\onecolumn
\aistatstitle{Communication-Efficient Federated Learning  With Data and Client Heterogeneity: \\
Supplementary Materials}

\section{Experimental setup}
In this section, we describe our experimental setup in detail. We begin by defining the hyper-parameters which control the behavior of QuAFL and FedAvg. Then, we proceed by carefully describing the way in which we simulated each of the algorithms. Finally, we detail the datasets, tasks, and models used for our experiments. The experiments in this section are done with the non-weighted version of the QuAFL.

\subsection{Hyper-parameters}

We first define our hyper-parameters; in the later sections, we will examine their impact on algorithm behavior through ablation studies. 
     \begin{description}
     \item[$n$:] Number of the clients.
     \item[$s$:] Number of clients interacting with the server at each step.
     \item[$K$:] \emph{In QuAFL}, this is the maximum number of allowed local steps by each client between two server calls. \emph{In FedAvg}, this is the number of local steps performed by each client upon each server call.
     \item[$b$:] Number of bits used to send a coordinate after quantization. 
     \item[$swt$:] Server waiting time, i.e.\ the amount of time that server waits between two consecutive calls.
     \item[$sit$:] Server interaction time, i.e.\ the amount of time that server needs to send and receive necessary data (excluding computation time).
     \end{description}
\subsection{Simulation}

We attempt to simulate a realistic FL deployment scenario, as follows.  We assume  a \emph{server} and $n$ \emph{clients}, each of which initially has a model copy. The training dataset is distributed among the clients so that each of them has access to $1/n$ of the training data. We track the performance of each algorithm by evaluating the server's model, on an unseen validation dataset. We measure \emph{loss} and \emph{accuracy} of the model with respect to \emph{simulation time}, \emph{server steps}, and \emph{total local steps performed by clients}. These setups so far were common between QuAFL and FedAvg. In the following, we are going to describe their specifications and differences.
     \begin{description}
     \item[QuAFL:] Upon each server call, the server chooses $s$ clients uniformly at random. It then sends its model to those clients and asks for their current local models. (Recall that clients send their model immediately to the server.) 
     Each of the clients will have taken a maximum of $K$ local steps by the time it is contacted by the server. The server then replaces its model with a carefully-computed average over the received models and its current model. This process increases time on the server by the \emph{server interaction time (sit)}. The server then waits for another interval of \emph{server waiting time (swt)} to make its next call. The $s$ receiving clients replace their model with the weighted average between their current model and the received server's model.
     Since each client performs local steps from its last interaction time until the current server time, nodes are effectively executing asynchronously. Moreover, note that communication is compressed, as all the models get encoded in their source and decoded in their destination.  
     \begin{description}
     \item[Quantization:] To have a lightweight but efficient communication between clients and the server, we use the well-known lattice quantization \citep{davies2021new}. Using this method, we send \emph{b} instead of 32 bits for each scalar dimension. Informally, each 32-bit number maps to one of the $2^{\emph{b}}$ quantized levels and can be sent using \emph{b} bits only. The encoded number can then be decoded to a sufficiently close number at the destination, following the quantization protocol.  
     \end{description}
     
     \item[FedAvg:] In the beginning of each round, server chooses $s$ clients randomly, and sends its current model to them. Each of those clients receives the model, uncompressed, and performs exactly $K$ local steps using this model as the starting point, and then sends back the resulting model to the server. The server then computes the average of the received models and adopts it as its model. By this synchronous structure, in each round, the server must wait for the \emph{slowest} client to complete its local steps plus an extra \emph{sit} for the communication time. After completing each round, the server starts the next call immediately, that means $\emph{swt}=0$ in FedAvg.   
     \end{description}

\paragraph{Timing Experiments.} We  differentiate between two types of timing experiments. \emph{Uniform timing} experiments, presented in the paper body, assume all clients take the same amount of time for a gradient step. 
However, in real-world setups, different devices may require different amounts of time to perform a single local step. This is one of the main disadvantages of synchronous federated optimization algorithms. To demonstrate how this fact affects the experiments, in our \emph{Non-uniform} timing experiments we differentiate clients to be either \emph{fast} or \emph{slow}. The length of each local step can be characterized as a memoryless time event. Therefore, the length of each local step can be defined by a random variable $X \sim exponential(\lambda)$. The parameter $\lambda$ is $1/2$ for \emph{fast clients} and $1/8$ for \emph{slow} clients; the expected runtime $\E(X)$ would be $2$ and $8$, respectively. In each timing experiment, we assumed only one fourth of clients to be \emph{slow}. 

\subsection{Datasets and Models}
 We used Pytorch  to manage the training process in our algorithm. We have trained neural networks for image classification tasks on three well-known datasets, \textbf{MNIST}, \textbf{Fashion MNIST}, and \textbf{CIFAR-10}. For all the datasets, we used the default train/test split of the dataset for our training/validation dataset. In the following, we describe the model architecture and the training hyper-parameters used to train on each of these datasets.  

\begin{description}
     \item[\textbf{MNIST}:] We used SGD optimizer with constant $lr = 0.5$ in all the training process. We used a two-layer MLP architecture with (784,32,10) nodes in its layers respectively. We used batch size 128 in each client's SGD step. 
     
     \item[\textbf{Fashion MNIST}:] Although this dataset has the same sample size and number of classes as \textbf{MNIST}, obtaining competitive performance on it requires a more complicated architecture. Therefore, we used a CNN model to train on this model and demonstrated the performance of our algorithm in a non-convex task. To optimize the models, we used Adam optimizer with constant $lr = 0.001$ and batch size 100.
     
     \item[\textbf{CIFAR-10}:] To load this dataset, we used data augmentation and normalization. For this task, we trained ResNet20 models. Moreover, the SGD optimizer with constant $lr = 0.03$ is used to in the training process. The batch size 64/200 is used for training/validation.
     
\end{description}

\subsection{Results on Fashion MNIST (FMNIST)} 

We begin by validating our earlier results, presented in the paper body, for the slightly more complex FMNIST dataset, and on a convolutional model. 

\begin{figure}
\centering
    \begin{minipage}[c]{0.43\textwidth}
        \centering
        \includegraphics[width=\textwidth, clip]{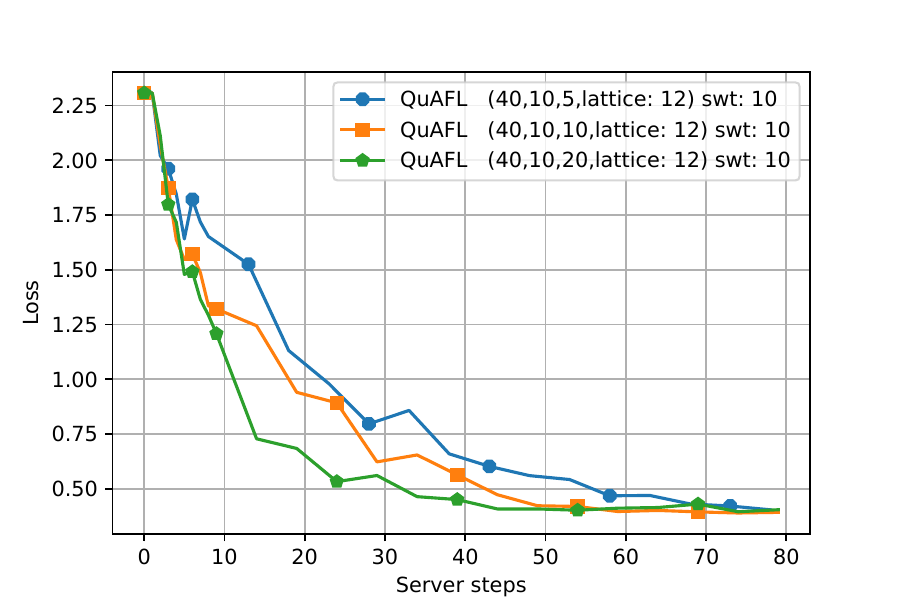}
        \vspace{-2em}
        \caption{Impact of the maximum number of local steps $K \in \{ 5, 10, 20 \}$ on the QuAFL algorithm / Fashion MNIST.}
        \label{fig:k-impact-fmnist}
    \end{minipage}~~
    \begin{minipage}[c]{0.43\textwidth}
        \centering
        \includegraphics[width=\textwidth, clip]{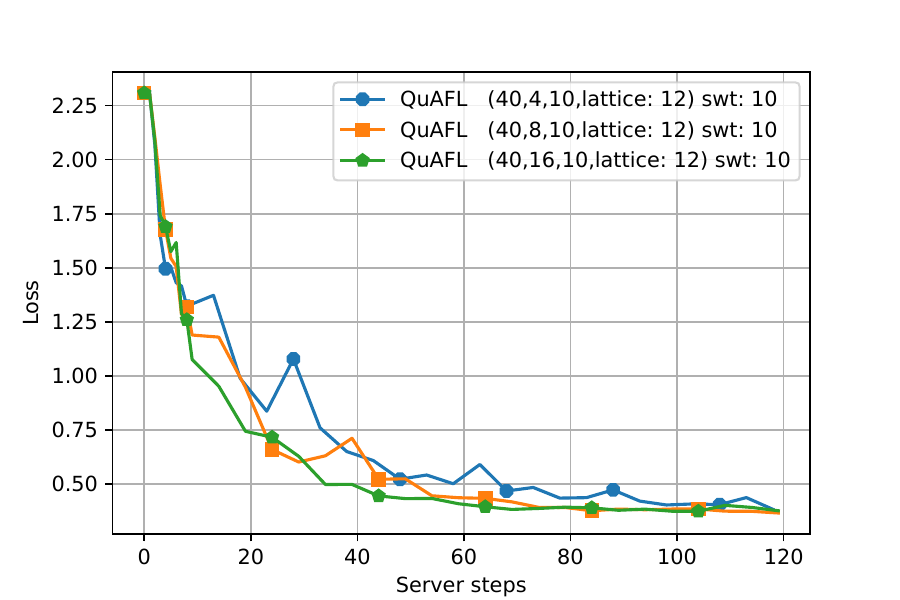}
        \vspace{-2em}
        \caption{Impact of the number of interacting peers $s \in \{4, 8, 16\}$ on the convergence of the algorithm.}
        \label{fig:s-impact-fmnist}
    \end{minipage}
    \begin{minipage}[c]{0.43\textwidth}
        \centering
        \includegraphics[width=\textwidth, clip]{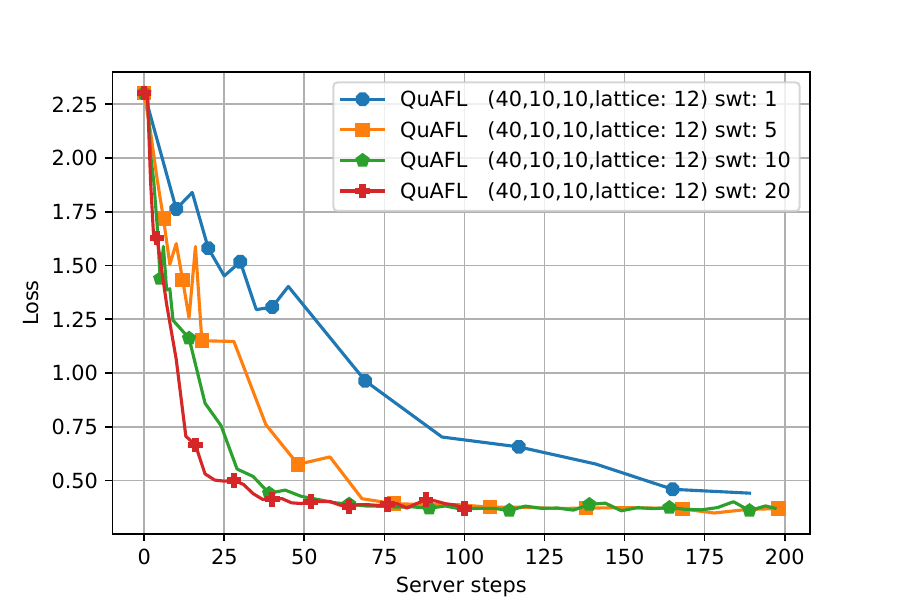}
        \vspace{-2em}
        \caption{Impact of the server contact frequency (controlled via server timeout \emph{swt}) on the convergence of the algorithm.}
        \label{fig:swt-impact-fmnist-rounds}
    \end{minipage}~~
    \begin{minipage}[c]{0.43\textwidth}
    \centering
    \includegraphics[width=\textwidth, clip]{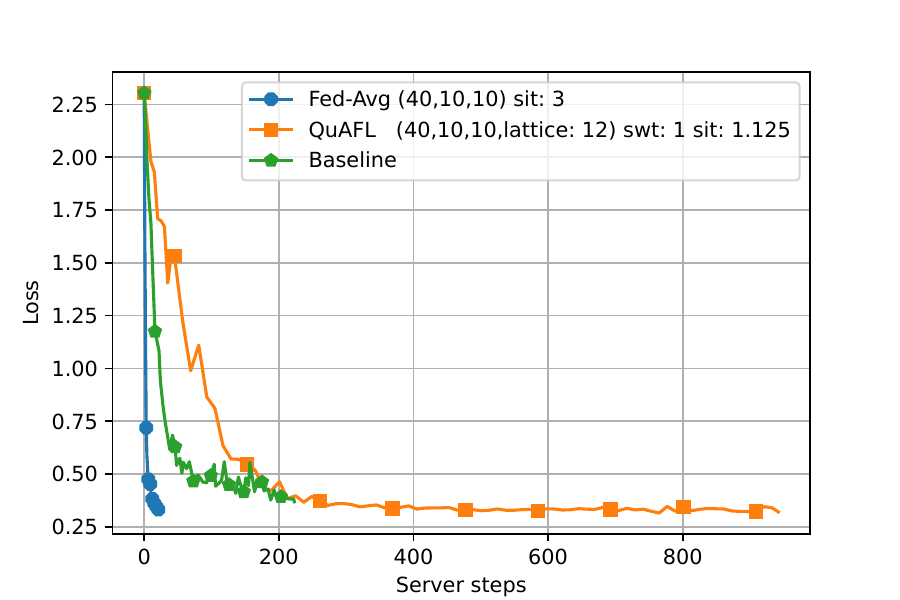}
    \vspace{-2em}
    \caption{Convergence comparison relative to total number of rounds, between QuAFL, FedAvg, and the sequential baseline.}
    \label{fig:fedavg}
\end{minipage}
\end{figure}

In Figures~\ref{fig:k-impact-fmnist} and~\ref{fig:s-impact-fmnist} we examine the impact of the parameters $K$ and $s$, respectively, on the total number of interaction rounds at the server, to reach a certain training loss. As expected, we notice that higher $K$ and $s$ improve the 
convergence behavior of the algorithm. 
In Figures~\ref{fig:swt-impact-fmnist-rounds} we examine the impact of the server waiting time on the convergence of the algorithm relative to the number of server rounds. Again, we notice that a higher server waiting time improves convergence, as it allows the server to take advantage of additional local steps performed at the clients, as predicted by our analysis. (Higher $swt$ means higher average number of steps completed $H$.)

Next, we examine the convergence, again in terms of number of optimization ``rounds'' at the server, between the sequential Baseline, FedAvg, and QuAFL. As expected, the Baseline is faster to converge than FedAvg, which in turn is faster than QuAFL in this measure. Specifically, the difference between QuAFL and the other algorithms comes because of the fact that, in our algorithm, nodes operate on \emph{old} variants of the model at every step, which slows down convergence. 
Next, we examine convergence in terms of actual time, in the heterogeneous setting in which 25\% of the clients are slow.

\begin{figure}
\centering
    \begin{minipage}[c]{0.43\textwidth}
        \centering
        \includegraphics[width=\textwidth, clip]{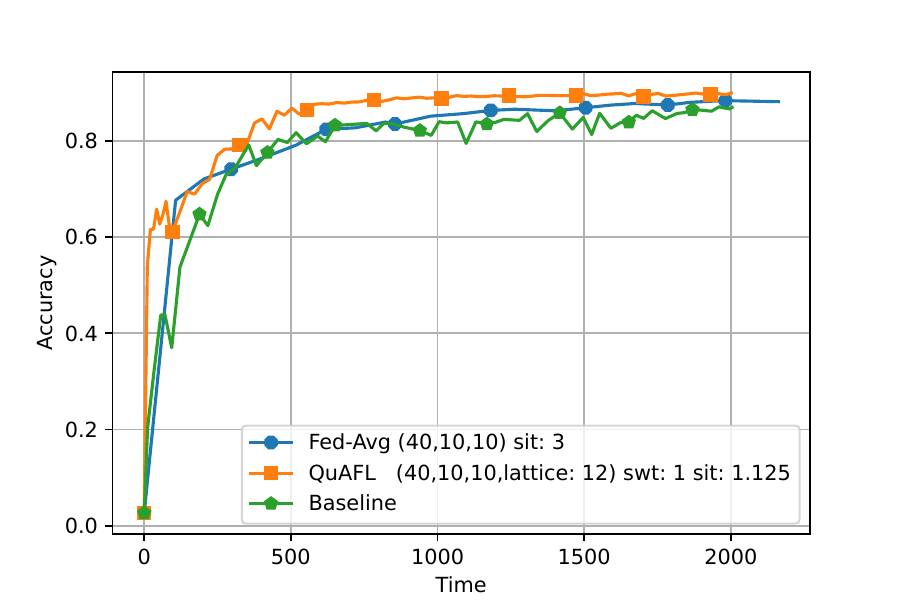}
        \vspace{-2em}
        \caption{Time vs. accuracy for various algorithm variants, on Fashion MNIST.}
        \label{fig:accuracy-vs-time-fmnist}
    \end{minipage}~~
    \begin{minipage}[c]{0.43\textwidth}
    \centering
    \includegraphics[width=\textwidth, clip]{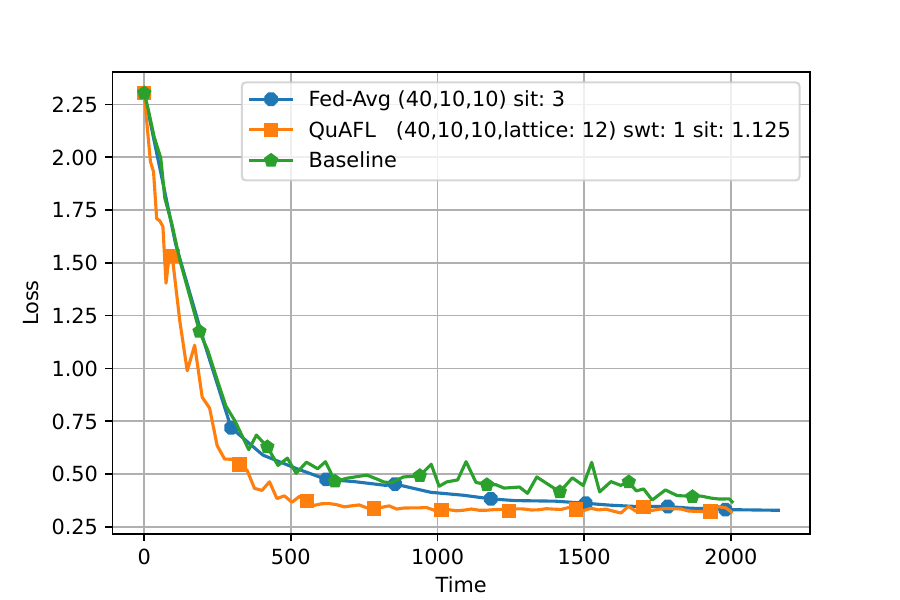}
    \vspace{-2em}
    \caption{Timing vs. loss for various algorithm variants, on Fashion MNIST.}
    \label{fig:loss-vs-time-fmnist}
\end{minipage}
\end{figure}

\begin{figure}
\centering
    \begin{minipage}[c]{0.43\textwidth}
        \centering
        \includegraphics[width=\textwidth, clip]{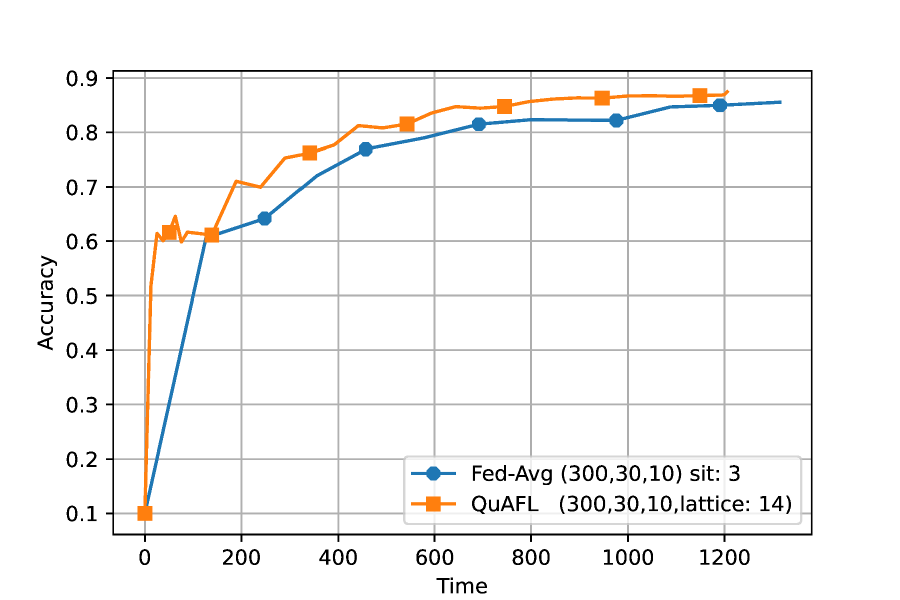}
        \vspace{-2em}
        \caption{Time vs. accuracy for n=300 clients, s=30 peers on Fashion MNIST.}
        \label{fig:300_nodes_accuracy-vs-time-fmnist}
    \end{minipage}~~
    \begin{minipage}[c]{0.43\textwidth}
    \centering
    \includegraphics[width=\textwidth, clip]{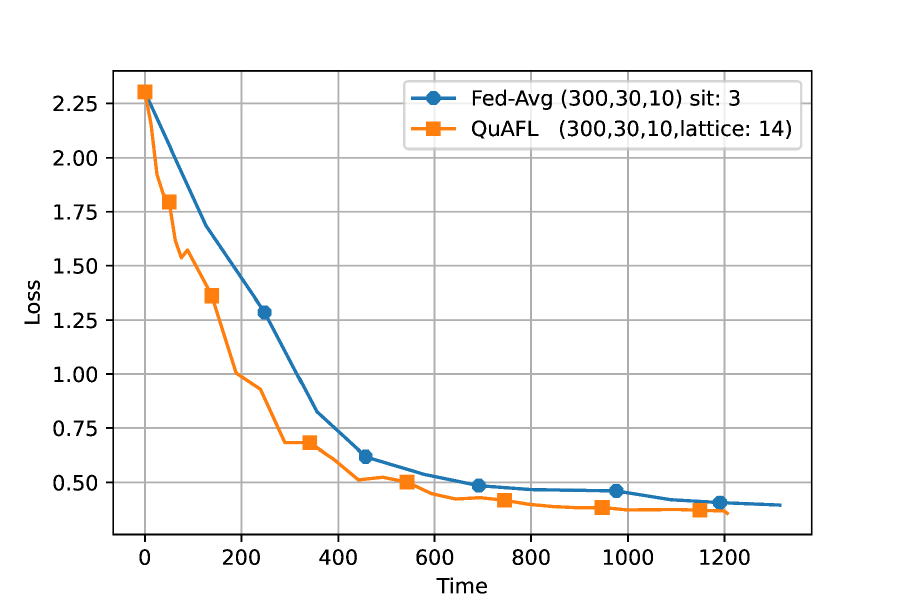}
    \vspace{-2em}
    \caption{Timing vs. loss for n=300 clients, s=30 peers on Fashion MNIST.}
    \label{fig:300_nodes_loss-vs-time-fmnist}
    \end{minipage}\\
        \begin{minipage}[c]{0.43\textwidth}
    \centering
    \includegraphics[width=\textwidth, clip]{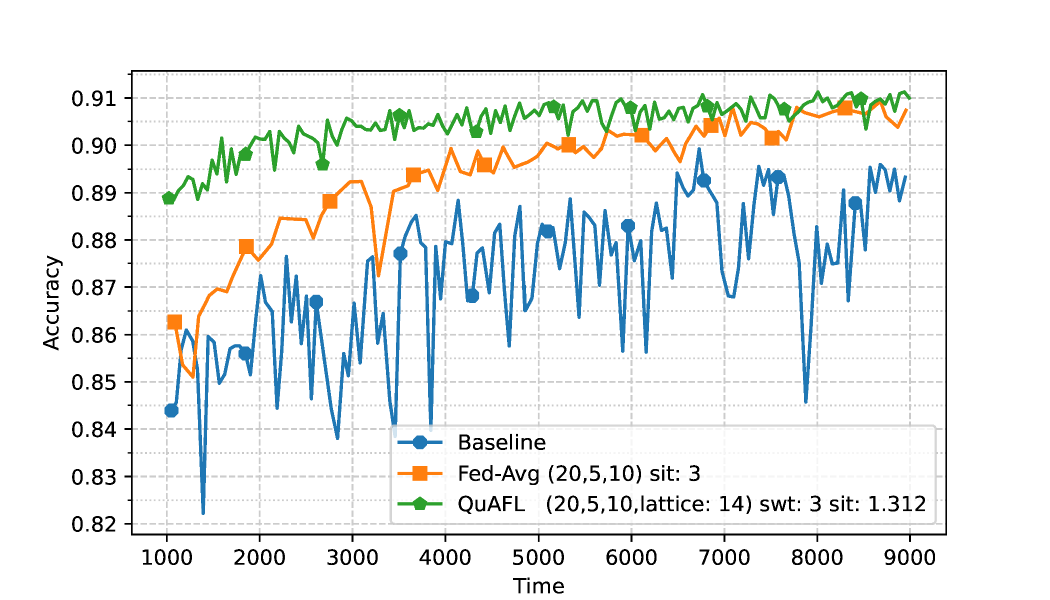}
    \vspace{-2em}
    \caption{Full convergence result for $n = 20$ clients and $s = 5$ on F-MNIST. All methods eventually reach the sequential $\sim 91\%$ top-1 accuracy on this task, but QuAFL is the fastest to do so in terms of wall-clock time.}
    \label{fig:full_convergence}
\end{minipage}~~
\end{figure}

In Figure~\ref{fig:accuracy-vs-time-fmnist}, we observe the validation accuracy ensured by various algorithms relative to the simulated execution time, whereas in Figure~\ref{fig:loss-vs-time-fmnist} we observe the training loss versus the same metric. 
(We assume that, in \emph{Baseline}, a single node acts as both the client and the server, and that this node is slow, i.e. has higher per-step times.)
To further support the robustness of our algorithm in regimes with large number of clients, we conducted an experiment with $n=300$ clients and $s=30$ peers interacting with the server at each step. The validation accuracy and loss versus time regarding the mentioned experiment plotted in Figure~\ref{fig:300_nodes_accuracy-vs-time-fmnist} and Figure~\ref{fig:300_nodes_loss-vs-time-fmnist} respectively.
We observe that, importantly, if \emph{time} is taken into account rather than the number of server rounds, QuAFL can provide notable speedups in these metrics.  This is specifically because of its asynchronous communication patters, which allow it to complete rounds faster, without having to always wait for the slow nodes to complete their local computation. While this behaviour is simulated, we believe that this reflects the algorithm's practical potential. 
Finally, Figure~\ref{fig:full_convergence} shows that all methods can reach the maximum accuracy for this dataset/model combination (for the SGD baseline, this occurs later), although QuAFL is the fastest to do so in terms of wall-clock time.

\subsection{Results on CIFAR-10}

We now present results for a standard image classification task on the CIFAR-10 dataset, using a ResNet20 model~\citep{he2016deep}. 

\begin{figure}
\centering
    \begin{minipage}[c]{0.43\textwidth}
        \centering
        \includegraphics[width=\textwidth, clip]{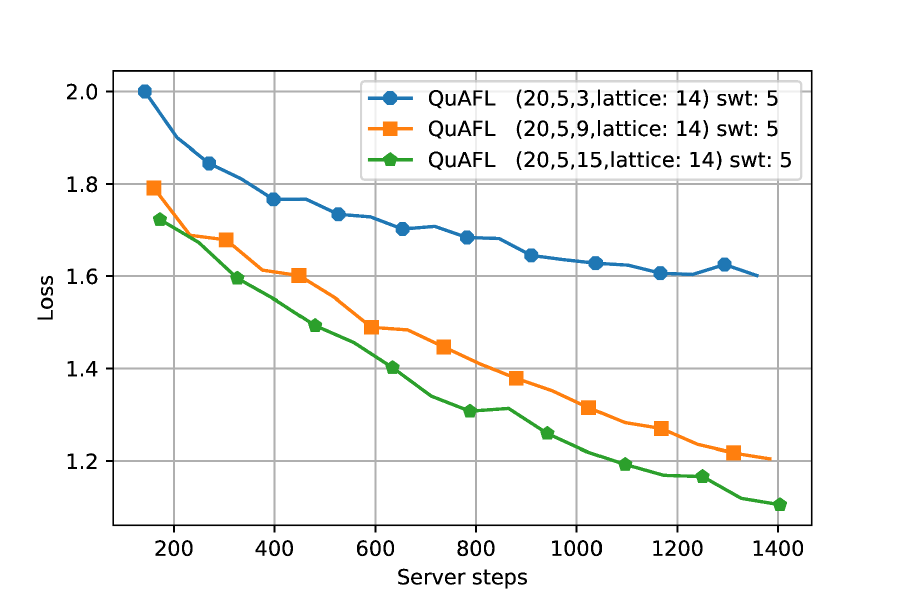}
        \vspace{-2em}
        \caption{Impact of maximum local steps $K \in \{ 3, 9, 15 \}$ on the QuAFL algorithm, on ResNet20/CIFAR-10.}
        \label{fig:k-impact-cifar10}
    \end{minipage}~~
    \begin{minipage}[c]{0.43\textwidth}
        \centering
        \includegraphics[width=\textwidth, clip]{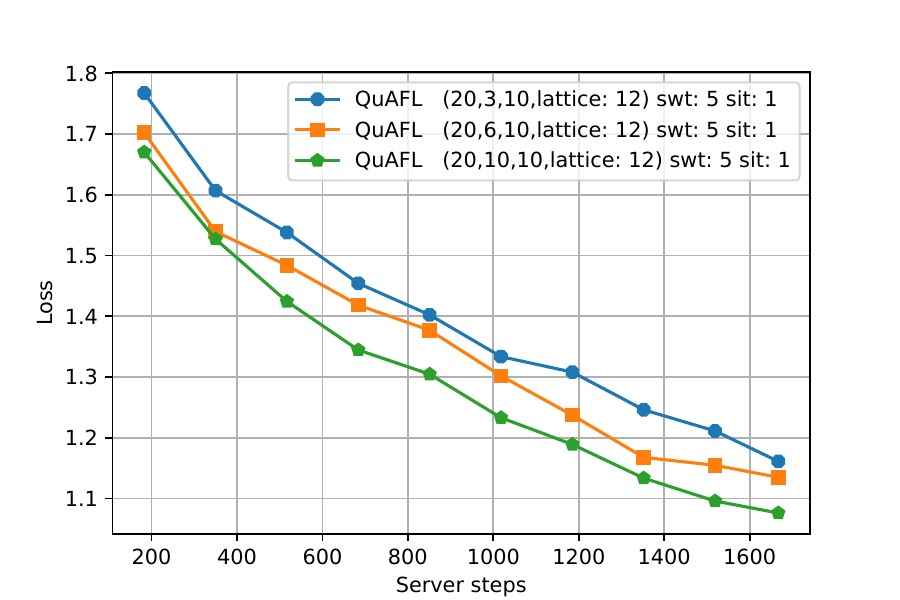}
        \vspace{-2em}
        \caption{Impact of the number of interacting peers $s \in \{3, 6, 10\}$ on the convergence of the algorithm.}
        \label{fig:s-impact-cifar10}
    \end{minipage}
    \begin{minipage}[c]{0.43\textwidth}
        \centering
        \includegraphics[width=\textwidth, clip]{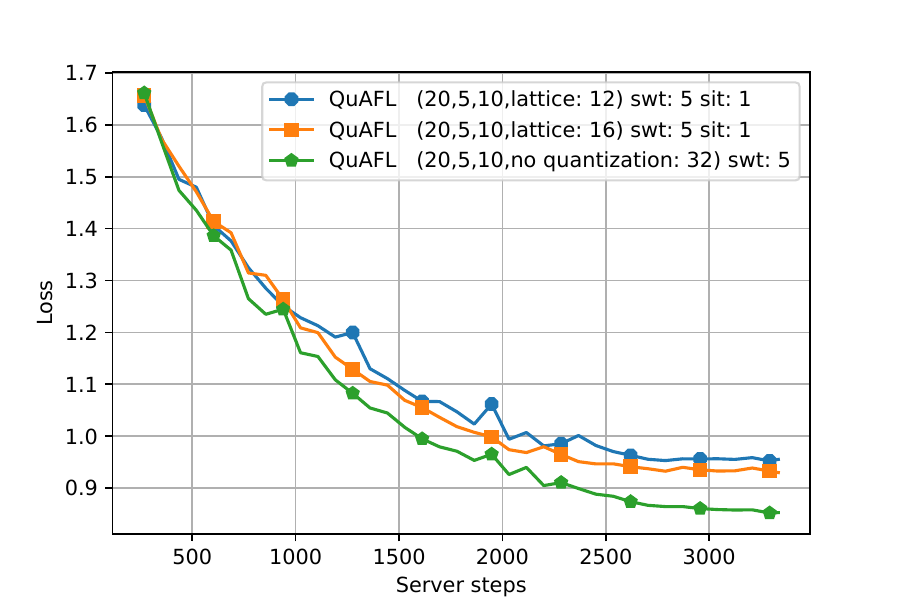}
        \vspace{-2em}
        \caption{Impact of the number of bits for quantization $b \in \{12, 16, 32\}$ on the convergence of the algorithm.}
        \label{fig:b-impact-cifar10}
    \end{minipage}~~
    \begin{minipage}[c]{0.43\textwidth}
        \centering
        \includegraphics[width=\textwidth, clip]{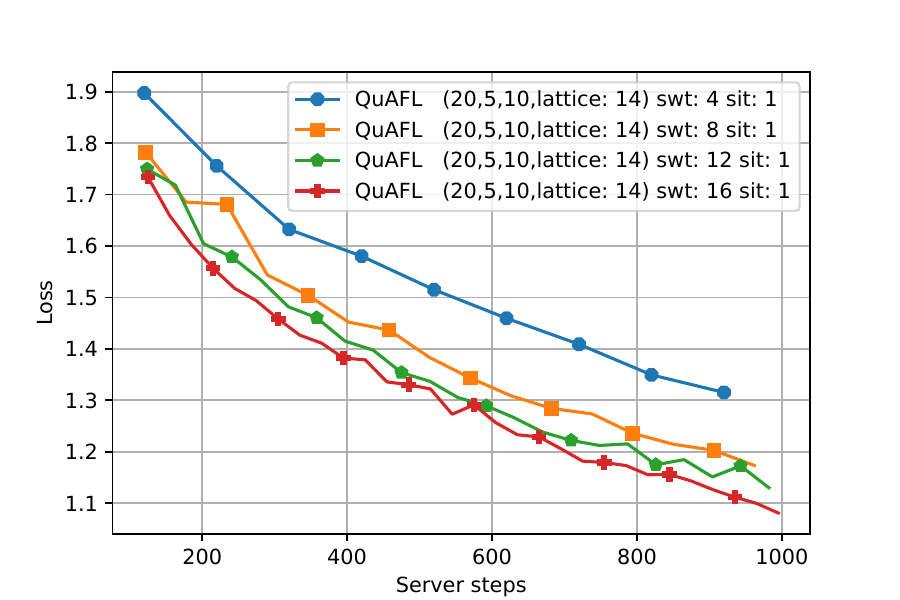}
        \vspace{-2em}
        \caption{Impact of the server contact frequency (controlled via server timeout \emph{swt}) on the convergence of the algorithm.}
        \label{fig:swt-impact-cifar10}
\end{minipage}
\end{figure}
Figures~\ref{fig:k-impact-cifar10} and \ref{fig:s-impact-cifar10} show the decrease in training loss versus the number of server steps (or rounds) for different values of $K$ and $s$ respectively. As our theory suggests, increasing $K$ and $s$  leads to an improvement in the convergence rate of the system. Figure~\ref{fig:b-impact-cifar10} demonstrates the impact of the number of quantization bits $b$, on the convergence behaviour of the algorithm. According to the definition of $b$, increasing the number of quantization bits improves the communication accuracy. Thus, as it can be seen in the graph, higher values of $b$ enhance the convergence relative to the number of server steps.       
Finally, Figure~\ref{fig:swt-impact-cifar10} shows the impact of the server interaction frequency, again controlled via the timeout parameter \emph{swt}, on the algorithm's convergence. It is apparent that a very high interaction frequency can slow the algorithm down, by not allowing it to take advantage of the clients' local steps.

\begin{figure}
\centering
    \begin{minipage}[c]{0.43\textwidth}
        \centering
        \includegraphics[width=\textwidth, clip]{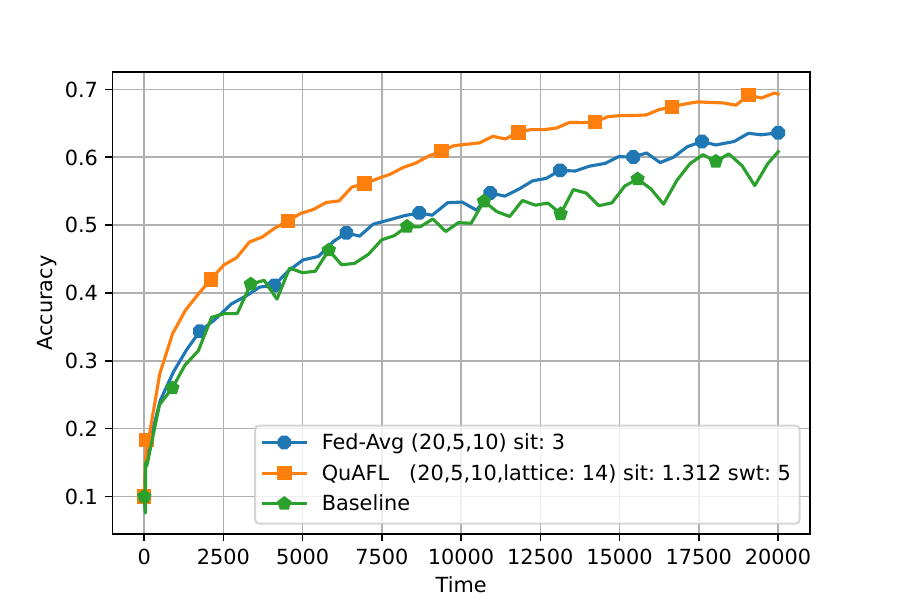}
        \vspace{-2em}
        \caption{Time vs. validation accuracy for various algorithm variants.}
        \label{fig:accuracy-vs-time-cifar10}
    \end{minipage}~~
    \begin{minipage}[c]{0.43\textwidth}
    \centering
    \includegraphics[width=\textwidth, clip]{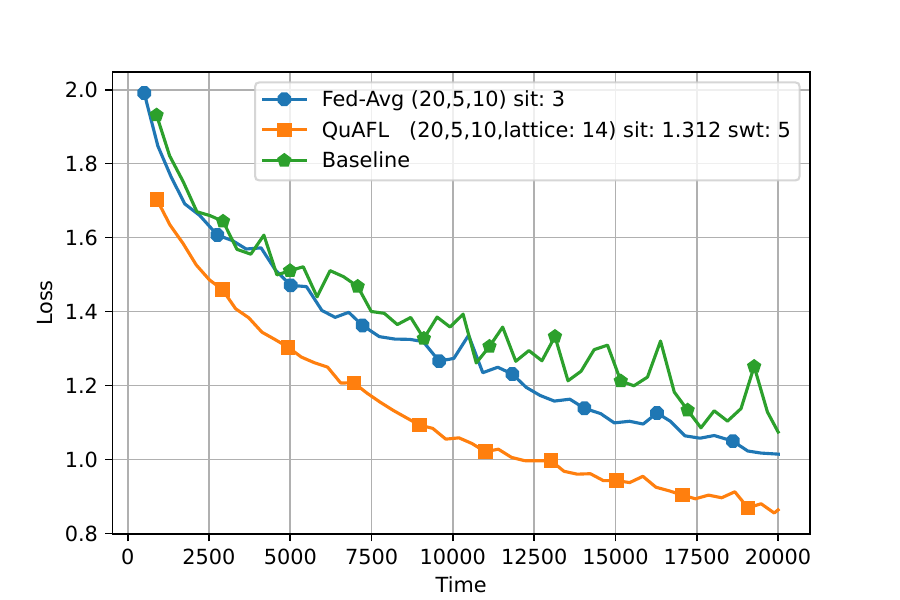}
    \vspace{-2em}
    \caption{Timing versus validation loss for various algorithm variants.}
    \label{fig:loss-vs-time-cifar10}
\end{minipage}
\\
    \begin{minipage}[c]{0.43\textwidth}
        \centering
        \includegraphics[width=\textwidth, clip]{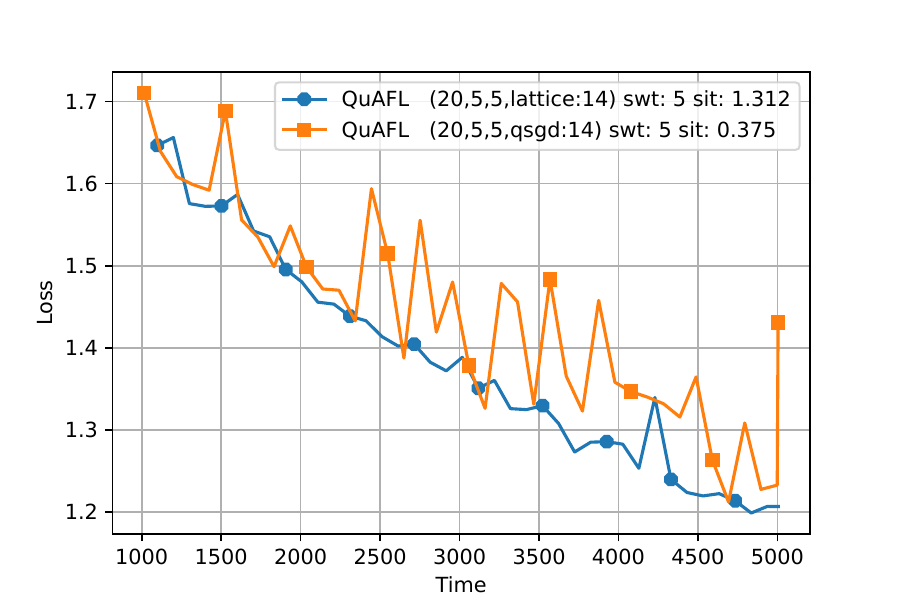}
        \vspace{-2em}
        \caption{Loss comparison of QuAFL with lattice quantization vs QSGD}
        \label{fig:qsgd_loss}
    \end{minipage}~~
    \begin{minipage}[c]{0.43\textwidth}
    \centering
    \includegraphics[width=\textwidth, clip]{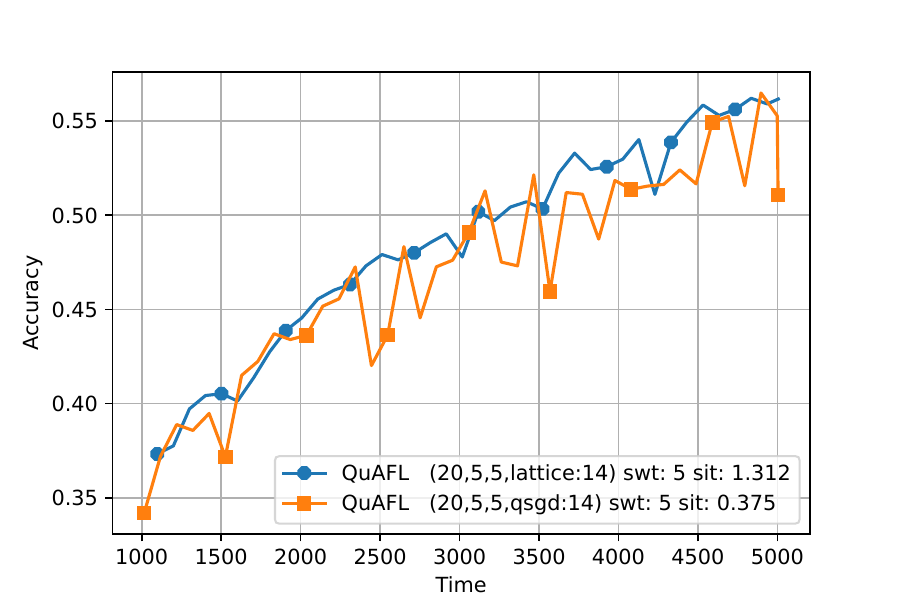}
    \vspace{-2em}
    \caption{Accuracy comparison of QuAFL with lattice quantization vs QSGD}
    \label{fig:qsgd_accuracy}
\end{minipage}
\\
    \begin{minipage}[c]{0.43\textwidth}
        \centering
        \includegraphics[width=\textwidth, clip]{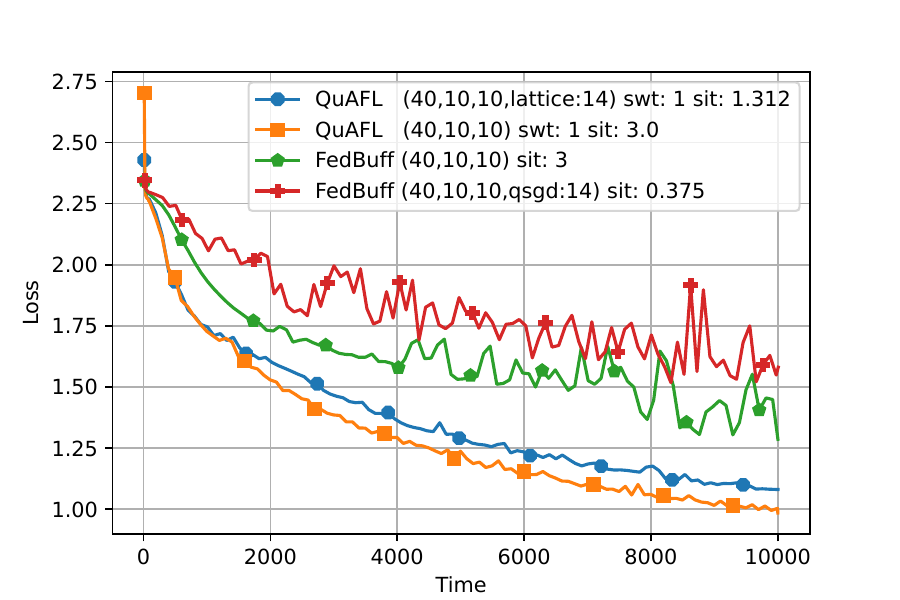}
        \vspace{-2em}
        \caption{QuAFL vs. SOTA asynchronous FL algorithm FedBuff with and without Quantization (ResNet20/CIFAR10).}
        \label{fig:fedbuff}
    \end{minipage}~~
    \begin{minipage}[c]{0.43\textwidth}
        \centering
        \includegraphics[width=\textwidth, clip]{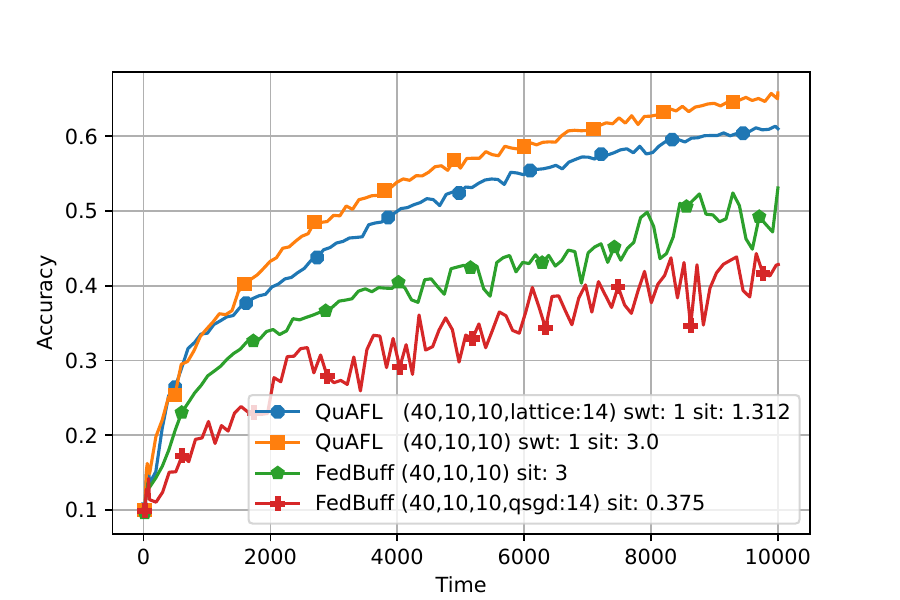}
        \vspace{-2em}
        \caption{Accuracy curves for Figure \ref{fig:fedbuff}. Convergence of QuAFL vs. SOTA asynchronous FL algorithm FedBuff with and without Quantization.}
        \label{fig:fedbuff_accuracy}
    \end{minipage}~~
    \\
    \begin{minipage}[c]{0.43\textwidth}
    \centering
    \includegraphics[width=\textwidth, clip]{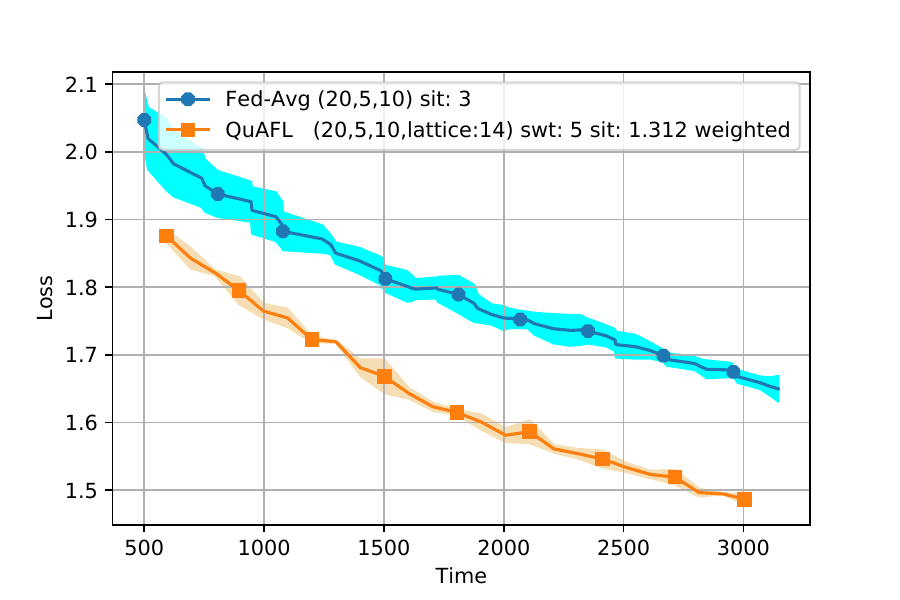}
    \vspace{-2em}
    \caption{Comparison of QuAFL vs FedAVG on ResNet20/CIFAR10, with 5 slow clients and 15 fast clients with error bars.}
    \label{fig:error_bar}
\end{minipage}
\end{figure}

In Figures~\ref{fig:accuracy-vs-time-cifar10} and \ref{fig:loss-vs-time-cifar10}, we examine the validation accuracy and loss, respectively, ensured by various algorithms versus the simulated execution time. (As in the F-MNIST experiments, we assumed the \emph{Baseline} to be a single slow node that performs an optimization step per round.) Again, the asynchronous nature of QuAFL provides a faster convergence rate than its synchronous counterparts; which can be clearly seen in the mentioned figures. 

As discussed earlier, using standard quantization techniques such as QSGD, is not theoretically justified, because in that case, the error would be proportional to the norm of the models, which can be unbounded. Therefore, there is no guarantee that these quantization techniques will always work. To show this experimentally, we ran an experiment on CIFAR10 with lr = 0.07. As shown in Figure \ref{fig:qsgd_loss} and \ref{fig:qsgd_accuracy}, QuAFL with lattice quantization works well in this setting, while QSGD quantization hurts the convergence. This experiment verifies our choice of lattice quantization.

In Figure~\ref{fig:fedbuff}, we compare QuAFL convergence (with and without quantization) relative to FedBuff, 
in terms of execution time. We have tuned each variant independently, and have also added QSGD quantization to FedBuff. 
(FedBuff is incompatible with lattice quantization, since nodes do not have a decoding key.)  
We observe that QuAFL converges faster, even with quantization: this is because QuAFL takes into account \emph{partial progress by slow clients,} whereas in FedBuff slower clients constantly contribute less significantly to the server updates. Moreover, we observe that quantization significantly increases the variance of FedBuff.

Figure \ref{fig:fedbuff_accuracy} shows the accuracy comparison of the experiment of Figure \ref{fig:fedbuff}. This figure shows that the accuracy follows the same pattern as the loss. Finally, as we mentioned before we omitted error bars in our experiments as the variance is low. We show one experiment in Figure \ref{fig:error_bar} with error bars to confirm this.

\section{The Complete Analysis}
\subsection{Overview and Notation}
Recall that $X_t$ denotes the model of the server at step $t$, and $X^i$ is the local model of client $i$ after its last interaction with the server. Also, $\tih_i$ is the sum of local gradient steps for model $X^i$ since its last interaction with the server.\\
For the convergence analysis, local steps of the clients that are not selected by the server don't have any effect on the server or other clients. Therefore we do not need to assume that clients are doing their local steps asynchronous,  and we can assume that all clients run their local gradient steps after the server contacts them. The only thing that we should consider is the randomness of the server selecting the clients, and the fact that the server can contact nodes before they have finished their $K$ steps. For this purpose, we assume that their number of steps is a random number $H_t^i$ with mean $H_i$.\\ 
To show the analysis in this setting, we introduce new notations that consider the server round. To this end, we use $X_t^i$ as the value of $X^i$ when the server is running its $t$th iteration, And $\tih_{i, t}$ for the sum of local steps at this time. We show each local step $q$ with a superscript. Formally, we have 
\begin{equation*}
\tih_{i, t}^{0}=0.
\end{equation*}
and for $1 \le q \le H_t^i$ let:
\begin{equation*}
\tih_{i, t}^q=\tg_i(X_t^i-\sum_{s=0}^{q-1}\eta \tih_{i, t}^{s}),
\end{equation*}
and 
\begin{equation*}
\tih_{i, t} = \sum_{q=0}^{H_t^i} \tih_{i, t}^{q}
\end{equation*}
Further , for $1\le q \le H_t^i$, let 
\begin{equation*}
h_{i, t}^q=\E[\tg_i(X_t^i-\sum_{s=0}^{q-1}\eta \tih_{i, t}^{s})]=\nabla f(X_t^i-\sum_{s=0}^{q-1}\eta \tih_{i, t}^{s})
\end{equation*}
be the expected value of $\tih_{i , t}^q$ taken over the randomness
of the stochastic gradient $\tg_i$. Also, we have:
\begin{equation*}
h_{i, t} = \sum_{q=0}^{H_t^i} h_{i, t}^{q}
\end{equation*}

\subsection{Properties of Local Steps}

\begin{lemma} \label{lem:sumofstochasticFixed}
For any agent $i$ and step $t$
\begin{equation*}
\E\|h_{i, t}^{q}\|^2 \le \frac{\sigma^2}{K^2}+8L^2 \E\|X_t^i-\mu_t\|^2 
+4\E\|\nabla f_i(\mu_t)\|^2.
\end{equation*}
\end{lemma}
\begin{proof}

\begin{align*}
\E\|h_{i, t}^{q}\|^2 &{\le}
\E\Bigg \|\Big(\nabla f_i(X_t^i-\sum_{s=0}^{q-1} \eta\tih_{i, t}^{s})-\nabla f_i(\mu_t)\Big)
+\nabla f_i(\mu_t)
\Bigg\|^2 
\\ &{\le} 
2\E\Bigg \|\nabla f_i(X_t^i-\sum_{s=0}^{q-1} \eta\tih_{i, t}^{s})-\nabla f_i(\mu_t)\Bigg \|^2
+2\E\|\nabla f_i(\mu_t)\|^2
\\
&{\le} 4L^2 \E\|X_t^i-\mu_t\|^2 + 4\eta^2L^2q \sum_{s=0}^{q-1} \E\| \tih_{i, t}^{s}\|^2
+2 \E \|\nabla f_i(\mu_t)\|^2
\\
&{\le} 4L^2 \E\|X_t^i-\mu_t\|^2 + 4\eta^2L^2q \sum_{s=0}^{q-1} (\E\| h_{i, t}^{s}\|^2 + \sigma^2)
+2 \E \|\nabla f_i(\mu_t)\|^2
\end{align*}

the rest of the proof is done by induction, and assuming $\eta < \frac{1}{4LK^2}$.
\end{proof}
\begin{lemma} \label{lem:sumofstochasticG}
For any  step $t$, we have that 
\begin{align*}
\sum_{i=1}^n \eta_i^2 \E\|\tih_{i, t}\|^2 \le 
2nK((\frac{1}{n} \sum_{i=1}^n \eta_i^2) \sigma^2 +2KG^2)+8L^2K^2 \E[\Phi_t]
+4nK^2B^2 \E \|\nabla f(\mu_t)\|^2.
\end{align*}
\end{lemma}
\begin{proof}
    Using lemma \ref{lem:sumofstochasticFixed}
\begin{align*}
\sum_{i=1}^n \eta_i^2\E\|\tih_{i, t}\|^2 &=
\sum_{i=1}^n \eta_i^2\sum_{h=0}^{K} Pr[H_{t}^i=h] \E\|\sum _{q=1}^h \tih_{i, t}^{q}\|^2 \\&
\le \sum_{i=1}^n \eta_i^2 \sum_{h=1}^{K} Pr[H_{t}^i=h] h \sum _{q=1}^h \E\| \tih_{i, t}^{q}\|^2 \\&
\le nK (\frac{1}{n} \sum_{i=1}^n \eta_i^2) \sigma^2 + \sum_{i=1}^n \eta_i^2 \sum_{h=1}^{K} Pr[H_{t}^i=h] h \sum _{q=1}^h \E\| h_{i, t}^{q}\|^2 \\&
\le nK (\frac{1}{n} \sum_{i=1}^n \eta_i^2)\sigma^2 + \sum_{i=1}^n \eta_i^2 K^2\Bigg(
\frac{\sigma^2}{K^2} + 8L^2 \E\|X_t^i-\mu_t\|^2 
+4\E\|\nabla f_i(\mu_t)\|^2
\Bigg)\\&
\le 2nK (\frac{1}{n} \sum_{i=1}^n \eta_i^2) \sigma^2 + \sum_{i=1}^n  \eta_i^2 K^2\Bigg(
8L^2 \E\|X_t^i-\mu_t\|^2 
+4\E\|\nabla f_i(\mu_t)\|^2
\Bigg)\\&
\le
2nK (\frac{1}{n} \sum_{i=1}^n \eta_i^2)\sigma^2+8L^2K^2 \E[\Phi_t]
+4nK^2G^2
+4nK^2B^2 \E \|\nabla f(\mu_t)\|^2.
\end{align*}
\end{proof}

\begin{lemma} \label{lem:gradientdifferencefixed}
For any local step $1 \le q$, and agent $1 \le i \le n$ and step $t$
\begin{align*}
&\E \|\nabla f_i(\mu_t)-h_{i, t}^q\|^2 \le 4L^2\eta^2q^2\sigma^2  +  4L^2 \E\|X_t^i-\mu_t\|^2 
+8L^2\eta^2q^2\E\|\nabla f_i(\mu_t)\|^2.
\end{align*}
\end{lemma}
\begin{proof}

\begin{align*}
\E \|\nabla f_i(\mu_t)-h_{i, t}^q\|^2&= \E\|\nabla f_i(\mu_t)-\nabla f_i(X_t^i-\sum_{s=0}^{q-1}\eta \tih_{i, t}^{s})\|^2 \\ &{\le}
L^2 \E \|\mu_t-X_t^i+\sum_{s=0}^{q-1} \eta \tih_{i, t}^{s}\|^2 \\&
{\le}  2L^2 \E \|X_t^i-\mu_t \|^2+2L^2\eta^2 \E\|\sum_{s=0}^{q-1}\tih_{i, t}^{s}\|^2
\\&
{\le}  2L^2 \E \|X_t^i-\mu_t \|^2+2L^2\eta^2q \sum_{s=0}^{q-1}\E\|\tih_{i, t}^{s}\|^2
\\&
\overset{\text{Lemma (\ref{lem:sumofstochasticFixed}})}{\le}  2L^2 \E \|X_t^i-\mu_t \|^2+2L^2\eta^2q^2 \bigg(2\sigma^2+8L^2 \E\|X_t^i-\mu_t\|^2 
 \\ &\quad \quad \quad \quad \quad \quad \quad \quad \quad \quad
+4\E\|\nabla f_i(\mu_t)\|^2 \bigg)
 \\ &\quad \quad \quad \quad  
= 4L^2\eta^2q^2\sigma^2  +   (2L^2 + 16L^4\eta^2q^2) \E\|X_t^i-\mu_t\|^2 
  \\ &\quad \quad \quad \quad \quad \quad 
+8L^2\eta^2q^2\E\|\nabla f_i(\mu_t)\|^2
 \\ &\quad \quad \quad \quad  
\le 4L^2\eta^2q^2\sigma^2  +  4L^2 \E\|X_t^i-\mu_t\|^2 
+8L^2\eta^2q^2\E\|\nabla f_i(\mu_t)\|^2
\end{align*}
and the last inequality comes from $\eta < \frac{1}{4LK}$.
\end{proof}

\begin{lemma} \label{lem:gradientdifference}
For any time step $t$
\begin{align*}
\sum_{i=1}^n \E \langle \nabla f(\mu_t), -\eta_i h_{i, t} \rangle \le 4KL^2\E[\Phi_t] &+ (-\frac{3H_{\min}n}{4} + 8B^2L^2\eta^2K^3n)\E\|\nabla f(\mu_t)\|^2 \\&+ 4nL^2\eta^2K^3(\sigma^2 + 2G^2).
\end{align*}

\end{lemma}
\begin{proof}
\begin{align*}
\sum_{i=1}^n &\E \langle \nabla f(\mu_t), -\eta_ih_{i, t} \rangle
=
\sum_{i=1}^n  \sum_{h=1}^{K} Pr[H_{t}^i=h] \E \langle \nabla f(\mu_t), -\eta_i \sum_{q=1}^h  h_{i, t}^q \rangle +  \sum_{i=1}^n  Pr[H_{t}^i=0] \E \langle \nabla f(\mu_t), 0\rangle \\ &= \sum_{i=1}^n  \eta_i \sum_{h=1}^{K} Pr[H_{t}^i=h] \sum_{q=1}^h 
\Big(
\E \langle \nabla f(\mu_t),\nabla f_i(\mu_t) -h_{i, t}^q \rangle - \E\langle\nabla f(\mu_t),\nabla f_i(\mu_t) \rangle \Big)
\end{align*}
Using Young's inequality we can upper bound 
$\E \langle \nabla f(\mu_t),\nabla f_i(\mu_t) -h_{i, t}^q\rangle$ by $$\frac{\E\|\nabla f(\mu_t)\|^2}{4}+\E\Big\|\nabla f_i(\mu_t)  -h_{i, t}^q\Big\|^2.$$
Plugging this in the above inequality we get:

\begin{align*}
\sum_{i=1}^n &\E \langle \nabla f(\mu_t), -\eta_ih_{i, t} \rangle \\
&\le \sum_{i=1}^n \eta_i \sum_{h=1}^{K} Pr[H_{t}^i=h] \sum_{q=1}^h 
\Big(
\E\|\nabla f(\mu_t) -h_{i, t}^q\|^2 + \frac{\E\|\nabla f(\mu_t)\|^2}{4} - 
\E\langle\nabla f(\mu_t),\nabla f_i(\mu_t) \rangle\Big)
\\&\overset{\text{Lemma \ref{lem:gradientdifferencefixed}}}{\le} \sum_{i=1}^n \eta_i  \sum_{h=1}^{K} Pr[H_{t}^i=h] \sum_{q=1}^h  
\Big(
4L^2\eta^2q^2\sigma^2  +  4L^2 \E\|X_t^i-\mu_t\|^2 
+8L^2\eta^2q^2\E\|\nabla f_i(\mu_t)\|^2
\\&\quad\quad\quad\quad\quad\quad\quad\quad\quad\quad
+ \frac{\E\|\nabla f(\mu_t)\|^2}{4} - 
\E\langle\nabla f(\mu_t),\nabla f_i(\mu_t) \rangle\Big)
\\&\le
\sum_{i=1}^n \eta_i \sum_{h=1}^{K} Pr[H_{t}^i=h] h \Big(
4L^2\eta^2h^2\sigma^2  +  4L^2 \E\|X_t^i-\mu_t\|^2 
+8L^2\eta^2h^2\E\|\nabla f_i(\mu_t)\|^2
\\&\quad\quad\quad\quad\quad\quad\quad\quad\quad\quad
 + \frac{\E\|\nabla f(\mu_t)\|^2}{4} - 
\E\langle\nabla f(\mu_t),\nabla f_i(\mu_t) \rangle \Big)
\\&\le
\sum_{i=1}^n \eta_i H_i \Big(
4L^2\eta^2K^2\sigma^2  +  4L^2 \E\|X_t^i-\mu_t\|^2 
+8L^2\eta^2K^2\E\|\nabla f_i(\mu_t)\|^2
\\&\quad\quad\quad\quad\quad\quad\quad\quad\quad\quad
 + \frac{\E\|\nabla f(\mu_t)\|^2}{4} - 
\E\langle\nabla f(\mu_t),\nabla f_i(\mu_t) \rangle \Big)
\\&\le
4KL^2 \E[\Phi_t] + 4nL^2\eta^2K^3((\frac{1}{n} \sum_{i=1}^n \eta_i^2)\sigma^2 + 2G^2)
+(8B^2nL^2\eta^2K^3 + \frac{H_{\min}n}{4} - H_{\min}n) \E \|\nabla f(\mu_t)\|^2\Big)
\end{align*}
Where in the last step we used that for each $i$, $\eta_i H_i = H_{\min}$, and $\sum_{i=1}^n \frac{f_i(x)}{n}=f(x)$, for any vector $x$.
\end{proof}

\paragraph{Discussion.} In the analysis, especially in this lemma, the value of $\eta_i H_i$s should be the same for all clients. However, we don't need them to be fixed in different rounds. Therefore, it is possible to extend the result to the case that clients' speeds would change, and they have different expected numbers of local steps in different rounds. Accordingly, the algorithm sets different values for $\eta_i^t$s. The final result would depend on the average number of local steps of each client during the whole training.

\subsection{Upper Bounding Potential Functions}
We proceed by proving the lemma \ref{lem:PhiBoundPerStepAsyncT} which upper bounds the expected change in potential:

\begin{replemma} {lem:PhiBoundPerStepAsyncT}
For any time step $t$ we have:
\begin{equation*} \label{lem:potentialperstep}
\E[\Phi_{t+1}] \le \left(1 - \frac{1}{4n}\right)\E[\Phi_t] +
8s\eta^2\sum_{i = 1}^n \E\|\tih_{i, t}\|^2 + 16n({R}^2+7)^2\gamma^2. 
\end{equation*}
\end{replemma}

\begin{proof}
First we bound change in potential $\Delta_t=\Phi_{t+1}-\Phi_t$ for some fixed time step $t>0$.

For this, let $\Delta_t^{S}$ be the change in potential when set $S$ of agents wake up.
for $i \in S$ define $S_t^i$ and $S_t$ as follows:
\begin{align*}
&S_t^i=-\frac{s}{s+1}\eta \eta_i \tih_{i, t}+\frac{Q(X_t)-X_t}{s+1} \\&
S_t=-\frac{1}{s+1}\eta \eta_i \sum_{i \in S}\tih_{i, t}+\frac{1}{s+1} \sum_{i \in S} (Q(X_t^i-\eta \eta_i \tih_{i, t}) - (X_t^i-\eta \eta_i \tih_{i, t}))
\end{align*}
We have that:
\begin{align*}
X_{t+1}^i&=\frac{s X_t^i+X_t}{s+1}+S_t^i \\
X_{t+1}&=\frac{\sum_{i \in S} X_t^i+X_t}{s+1}+S_t \\
\mu_{t+1}&=\mu_t+\frac{\sum_{j \in S} S_t^j+S_t}{n+1}
\end{align*}
This gives us that for $i \in S$:
\begin{align*}
X_{t+1}^i-\mu_{t+1}&=\frac{s X_t^i+X_t}{s+1}+S_t^i - \frac{\sum_{j \in S} S_t^j+S_t}{n+1}-\mu_t \\
X_{t+1}-\mu_{t+1}&=\frac{\sum_{i \in S} X_t^i+X_t}{s+1}+S_t -\frac{\sum_{j \in S} S_t^j+S_t}{n+1} -\mu_t \\
\end{align*}
For $k \not \in S$ we get that
\begin{equation*}
X_{t+1}^k-\mu_{t+1}=X_{t}^k-\frac{\sum_{j \in S} S_t^j+S_t}{n+1}-\mu_t.
\end{equation*}

Hence:
\begin{align*} \label{eq:PhiBound2V2}
\Delta_t^{S}&=\sum_{i \in S} \Big(\Big \|\frac{s X_t^i+X_t}{s+1}+S_t^i - \frac{\sum_{j \in S} S_t^j+S_t}{n+1}-\mu_t \Big\|^2 -\Big\|X_t^i-\mu_t\Big\|^2\Big) \\
&\quad +\Big \|\frac{\sum_{i \in S} X_t^i+X_t}{s+1}+S_t -\frac{\sum_{j \in S} S_t^j+S_t}{n+1} -\mu_t\Big \|^2-\Big\|X_t-\mu_t\Big\|^2 \\
&\quad +\sum_{k \not \in S}\Big(\Big\|X_{t}^k-\frac{\sum_{j \in S} S_t^j+S_t}{n+1}-\mu_t \Big \|^2-\Big\|X_t^k-\mu_t\Big\|^2 \Big) \\
&=\sum_{i \in S} \Big(\Big \|\frac{s X_t^i+X_t}{s+1} - \mu_t\Big \|^2 + \Big \| S_t^i + \frac{\sum_{j \in S} S_t^j+S_t}{n+1}\Big \|^2 \\
&\quad \quad +2\Big\langle \frac{s X_t^i+X_t}{s+1} - \mu_t,S_t^i - \frac{\sum_{j \in S} S_t^j+S_t}{n+1}\Big\rangle -\Big\|X_t^i-\mu_t\Big\|^2\Big) \\
&+\Big(\Big \|\frac{\sum_{i \in S} X_t^i+X_t}{s+1} - \mu_t\Big \|^2 + \Big \| S_t -\frac{\sum_{j \in S} S_t^j+S_t}{n+1}\Big \|^2 \\
&\quad \quad +2\Big\langle \frac{\sum_{i \in S} X_t^i+X_t}{s+1} - \mu_t, S_t -\frac{\sum_{j \in S} S_t^j+S_t}{n+1}\Big\rangle -\Big\|X_t-\mu_t\Big\|^2\Big) \\
& +\sum_{k \not \in S} 2\Big\langle X_{t}^k-\mu_t, -\frac{\sum_{j \in S} S_t^j+S_t}{n+1}\Big\rangle + \sum_{k \not \in S}\Big\|\frac{\sum_{j \in S} S_t^j+S_t}{n+1} \Big \|^2
\end{align*}

Observe that:
\begin{equation*}
\sum_{k=0}^n \Big\langle X_t^k-\mu_t, -\frac{\sum_{j \in S} S_t^j+S_t}{n+1} \Big\rangle=0. \end{equation*}

After combining the above  two equations, we get that:
\begin{align*} \label{eq:PhiBound2V2}
\Delta_t^{S}&=\sum_{i \in S} \Big(\Big \|\frac{s (X_t^i - \mu_t) + (X_t - \mu_t)}{s+1} \Big \|^2 - \frac{s}{s+1} \Big\|X_t^i-\mu_t\Big\|^2 - \frac{1}{s+1} \Big\|X_t-\mu_t\Big\|^2 \Big) \\
&+\Big(\Big \|\frac{\sum_{i \in S} (X_t^i - \mu_t)+(X_t-\mu_t)}{s+1}\Big \|^2 -\sum_{i \in S} \frac{1}{s+1} \Big\|X_t^i-\mu_t\Big\|^2 - \frac{1}{s+1} \Big\|X_t-\mu_t\Big\|^2 \Big)\\
&+\sum_{i \in S} \Big( \Big \| S_t^i - \frac{\sum_{j \in S} S_t^j+S_t}{n+1}\Big \|^2 +2\Big\langle \frac{s X_t^i+X_t}{s+1} - \mu_t,S_t^i\Big\rangle \Big) \\
&+\Big \| S_t -\frac{\sum_{j \in S} S_t^j+S_t}{n+1}\Big \|^2 +2\Big\langle \frac{\sum_{i \in S} X_t^i+X_t}{s+1} - \mu_t, S_t\Big\rangle \\
& + \sum_{k \not \in S}\Big\|\frac{\sum_{j \in S} S_t^j+S_t}{n+1} \Big \|^2
\end{align*}

By simplifying the above, we get:

\begin{align*} 
\Delta_t^{S}&=\frac{-s}{(s+1)^2}\sum_{i \in S}  \|X_t^i - X_t \|^2 -\frac{1}{(s+1)^2} \sum_{i \in S}  \|X_t^i - X_t \|^2 -\frac{1}{(s+1)^2} \sum_{i, j \in S}  \|X_t^i - X_t^j \|^2)  \\
&+\sum_{i \in S} \Big \| S_t^i - \frac{\sum_{j \in S} S_t^j+S_t}{n+1}\Big \|^2 + \frac{2s}{s+1}\sum_{i \in S}\Big\langle X_t^i - \mu_t,S_t^i\Big\rangle + \frac{2}{s+1}\sum_{i \in S}\Big\langle X_t - \mu_t,S_t^i\Big\rangle \\
&+\Big \| S_t -\frac{\sum_{j \in S} S_t^j+S_t}{n+1}\Big \|^2 +\frac{2}{s+1} \sum_{i \in S} \Big\langle X_t^i - \mu_t, S_t\Big\rangle
 +\frac{2}{s+1} \Big\langle X_t - \mu_t, S_t\Big\rangle \\
& + \sum_{k \not \in S}\Big\|\frac{\sum_{j \in S} S_t^j+S_t}{n+1} \Big \|^2
\end{align*}

Let $\alpha$ be a parameter we will fix later:
\begin{align*}
\Big\langle X_t^i-\mu_t, S_t^i\Big\rangle
\overset{\text{Young}}{\le} \alpha \Big\| X_t^i-\mu_t\Big\|^2+ \frac{\Big\|S_t^i\Big\|^2}{4\alpha}
\end{align*}
Finally, we get that
\begin{align*}
\Delta_t^{S}&\le\frac{-1}{s+1}\sum_{i \in S}  \|X_t^i - X_t \|^2 + 2\sum_{i \in S} \Big \| S_t^i \Big \|^2 + \frac{2s(s+1)}{(n+1)^2}\sum_{j \in S} \Big \| S_t^j \Big \|^2 + \frac{2s(s+1)}{(n+1)^2} \Big \| S_t \Big \|^2  \\&+ 
\sum_{i \in S} \frac{2s\alpha}{s+1} \Big\| X_t^i-\mu_t\Big\|^2+  \sum_{i \in S} \frac{s\Big\|S_t^i\Big\|^2}{2\alpha(s+1)} + \sum_{i \in S} \frac{2\alpha}{s+1} \Big\| X_t-\mu_t\|^2+  \sum_{i \in S} \frac{\Big\|S_t^i\Big\|^2}{2\alpha(s+1)} \\
&+2\Big \| S_t \Big \|^2 + \frac{2(s+1)}{(n+1)^2}\sum_{j \in S} \Big \| S_t^j \Big \|^2 + \frac{2(s+1)}{(n+1)^2} \Big \| S_t \Big \|^2 +\sum_{i \in S} \frac{2\alpha}{s+1} \Big\| X_t^i-\mu_t\|^2+  \sum_{i \in S} \frac{\Big\|S_t\Big\|^2}{2\alpha(s+1)}\\&
 +\frac{2\alpha}{s+1} \Big\| X_t-\mu_t\|^2+  \frac{\Big\|S_t\Big\|^2}{2\alpha(s+1)} + \sum_{j \in S} \frac{(n-s)(s+1)}{(n+1)^2}\Big\|S_t^j\Big \|^2 +  \frac{(n-s)(s+1)}{(n+1)^2}\Big\|S_t\Big \|^2
\\&=\frac{-1}{s+1}\sum_{i \in S}  \|X_t^i - X_t \|^2 +
(2 + \frac{2(s+1)^2}{(n+1)^2}  + \frac{1}{2\alpha} + \frac{(n-s)(s+1)}{(n+1)^2})\sum_{j \in S} \Big \| S_t^j \Big \|^2 + \\&
(2 + \frac{2(s+1)^2}{(n+1)^2}  + \frac{1}{2\alpha} + \frac{(n-s)(s+1)}{(n+1)^2}) \Big \| S_t \Big \|^2 +
\sum_{i \in S} 2\alpha \Big\| X_t^i-\mu_t\Big\|^2+ 
2\alpha \Big\| X_t-\mu_t\Big\|^2\\
&\le\frac{-1}{s+1}\sum_{i \in S}  \|X_t^i - X_t \|^2 +
(4 + \frac{1}{2\alpha})\sum_{i \in S} \Big \| S_t^i \Big \|^2 +
\\& \quad \quad \quad (4 + \frac{1}{2\alpha}) \Big \| S_t \Big \|^2 +
\sum_{i \in S} 2\alpha \Big\| X_t^i-\mu_t\Big\|^2+ 
2\alpha \Big\| X_t-\mu_t\Big\|^2\\
\end{align*}

 Using  definitions of $S_t^i$ and $S_t$, Cauchy-Schwarz inequality and properties of quantization
we get that 
\begin{align*}
\|S_t^i\|^2 &\le \frac{2s^2}{(s+1)^2}\eta^2 \eta_i^2 \|\tih_{i, t}\|^2+\frac{2({R}^2+7)^2\gamma^2}{(s+1)^2} .\\
\|S_t\|^2 &\le \frac{2s}{(s+1)^2}\eta^2 \eta_i^2 \sum_{i \in S} \|\tih_{i, t}\|^2 + \frac{2s^2({R}^2+7)^2\gamma^2}{(s+1)^2}
\end{align*}
Next, we plug this in the previous inequality:
\begin{align*}
\Delta_t^{S}& \le \frac{-1}{s+1}\sum_{i \in S}  \|X_t^i - X_t \|^2 + \sum_{i \in S} 2\alpha \Big\| X_t^i-\mu_t\Big\|^2+ 2\alpha \Big\| X_t-\mu_t\Big\|^2 \\&
+(4 + \frac{1}{2\alpha})\frac{2s^2+2s}{(s+1)^2}\eta^2 \eta_i^2 \|\tih_{i, t}\|^2+\frac{(2s^2 + 2s)({R}^2+7)^2\gamma^2}{(s+1)^2}\\&
\le \frac{-1}{s+1}\sum_{i \in S}  \|X_t^i - X_t \|^2 + \sum_{i \in S} 2\alpha \Big\| X_t^i-\mu_t\Big\|^2+ 2\alpha \Big\| X_t-\mu_t\Big\|^2 \\&
+(4 + \frac{1}{2\alpha})(\eta^2 \sum_{i \in S}  \eta_i^2 \|\tih_{i, t}\|^2+2({R}^2+7)^2\gamma^2)
\end{align*}

Next, we calculate the probability of choosing the set $S$
and upper bound $\Delta_t$ in expectation, for this we define $\E_t$ as expectation conditioned on the entire history up to and including step $t$
\begin{align*}
    \E_t[\Delta_t]&=\sum_S \frac{1}{{n \choose s}} \E_t[\Delta_t^{S}] 
    \\&\le
    \sum_S \frac{1}{{n \choose s}} \Bigg(\frac{-1}{s+1}\sum_{i \in S}  \|X_t^i - X_t \|^2 + \sum_{i \in S} 2\alpha \Big\| X_t^i-\mu_t\Big\|^2+ 2\alpha \Big\| X_t-\mu_t\Big\|^2 \\&
\quad \quad \quad \quad +(4 + \frac{1}{2\alpha})(\eta^2 \sum_{i \in S} \eta_i^2 \|\tih_{i, t}\|^2+2({R}^2+7)^2\gamma^2) \Bigg)  
\\ &= \frac{-{n-1 \choose s - 1}}{(s+1){n \choose s}}\sum_{i}  \|X_t^i - X_t \|^2 + \sum_{i} \frac{2\alpha {n-1 \choose s-1}}{{n \choose s}} \Big\| X_t^i-\mu_t\Big\|^2+ 2\alpha \Big\| X_t-\mu_t\Big\|^2 \\&
\quad \quad \quad \quad +(4 + \frac{1}{2\alpha})(\eta^2 \frac{{n-1 \choose s - 1}}{{n \choose s}} \sum_{i} \eta_i^2 \|\tih_{i, t}\|^2+2({R}^2+7)^2\gamma^2) 
\\ &\le 
-\sum_i \frac{s\|X_t^i-\mu_t\|^2}{(s+1)n} +
\sum_i 2\frac{s\alpha}{n}\Big\|X_t^i-\mu_t\Big\|^2 + 2\alpha\Big\|X_t-\mu_t\Big\|^2 \\&\quad\quad+
(8+\frac{1}{\alpha})({R}^2+7)^2\gamma^2 +
\sum_i \frac{s}{n}(4+\frac{1}{2\alpha})\eta^2 \eta_i^2\E_t\|\tih_{i, t}\|^2
\\ &\le 
(\frac{-s}{(s+1)n} + 2\alpha) \Phi_t +
(8+\frac{1}{\alpha})({R}^2+7)^2\gamma^2 +
\sum_i \frac{s}{n}(4+\frac{1}{2\alpha})\eta^2 \eta_i^2 \E_t\|\tih_{i, t}\|^2
\end{align*}

By setting $\alpha=\frac{3s-1}{n(8s+8)} \ge \frac{1}{8n}$, we get that:
\begin{align*}
    \E_t[\Delta_t] &\le - \frac{1}{4n}\Phi_t +
16n({R}^2+7)^2\gamma^2 +
\sum_i 8s\eta^2 \eta_i^2 \E_t\|\tih_{i, t}\|^2.
\end{align*}
Next we remove the conditioning , and use the definitions of $\Delta_i$ and $S_t^i$ (for $S_t^i$ we also use upper bound which come from the properties of quantization).
\begin{align*}
\E[\E_t[\Phi_{t+1}]] &= \E[\Delta_t+\Phi_t] \le (1 - \frac{1}{4n})\E[\Phi_t] +
 16n({R}^2+7)^2\gamma^2 +
8s\eta^2\sum_i \eta_i^2 \E\|\tih_{i, t}\|^2
\end{align*}
\end{proof}
\begin{lemma} \label{lem:PhiBoundPerStepAsync2}
For any time step $t$ we have:
\begin{equation*}
\E[\Phi_{t+1}] \le (1 - \frac{1}{5n})\E[\Phi_t]  +
 16n({R}^2+7)^2\gamma^2 + 
16nsK\eta^2((\frac{1}{n} \sum_{i=1}^n \eta_i^2)\sigma^2+2KG^2)
+32B^2nsK^2\eta^2\E\|\nabla f(\mu_t)\|^2
\end{equation*}
\end{lemma}
\begin{proof}
By combining Lemma \ref{lem:PhiBoundPerStepAsyncT} and \ref{lem:sumofstochasticG} we have:
\begin{align*}
\E[\Phi_{t+1}] &\le  (1 - \frac{1}{4n})\E[\Phi_t] +
 16n({R}^2+7)^2\gamma^2 + \\&\quad \quad \quad
8s\eta^2 \Big(2nK((\frac{1}{n} \sum_{i=1}^n \eta_i^2)\sigma^2+2KG^2)+8L^2K^2 \E[\Phi_t]
+4nK^2B^2 \E \|\nabla f(\mu_t)\|^2 \Big)
\\ &= (1 - \frac{1}{4n} + 64sL^2K^2\eta^2)\E[\Phi_t] +
 16n({R}^2+7)^2\gamma^2 + \\&\quad \quad \quad \quad \quad
16nsK\eta^2((\frac{1}{n} \sum_{i=1}^n \eta_i^2)\sigma^2+2KG^2)
+32B^2nsK^2\eta^2\E\|\nabla f(\mu_t)\|^2
\\ &\le (1 - \frac{1}{5n})\E[\Phi_t]  +
 16n({R}^2+7)^2\gamma^2 + \\&\quad \quad \quad \quad \quad
16nsK\eta^2((\frac{1}{n} \sum_{i=1}^n \eta_i^2)\sigma^2+2KG^2)
+32B^2nsK^2\eta^2\E\|\nabla f(\mu_t)\|^2
\end{align*}
\end{proof}

\begin{lemma} \label{lem:PhiBoundGlobal}
For the sum of potential functions in all $T$ steps we have:
\begin{equation*}
\sum_{t=0}^T \E[\Phi_t]
\le 80Tn^2({R}^2+7)^2\gamma^2 +
80Tn^2sK\eta^2((\frac{1}{n} \sum_{i=1}^n \eta_i^2)\sigma^2+2KG^2)
+160B^2n^2sK^2\eta^2 \sum_{t=0}^{T-1} \E\|\nabla f(\mu_t)\|^2
\end{equation*}
\end{lemma}
\begin{proof}
\begin{align*}
\sum_{t=0}^{T-1} \E[\Phi_{t+1}] &\le 
\sum_{t=0}^{T-1} \bigg(  (1 - \frac{1}{5n})\E[\Phi_t]  +
 16n({R}^2+7)^2\gamma^2 + \\&\quad \quad \quad \quad \quad
16nsK\eta^2((\frac{1}{n} \sum_{i=1}^n \eta_i^2)\sigma^2+2KG^2)
+32B^2nsK^2\eta^2\E\|\nabla f(\mu_t)\|^2 \bigg)
\\ &\le (1 - \frac{1}{5n}) \sum_{t=0}^{T-1} \E[\Phi_t]  +
 16Tn({R}^2+7)^2\gamma^2 +
16TnsK\eta^2((\frac{1}{n} \sum_{i=1}^n \eta_i^2)\sigma^2+2KG^2)
\\& \quad \quad \quad \quad \quad +32B^2nsK^2\eta^2 \sum_{t=0}^{T-1} \E\|\nabla f(\mu_t)\|^2
\\& \sum_{t=0}^{T} \E[\Phi_{t}] \le
 5n \big( 16Tn({R}^2+7)^2\gamma^2 +
16TnsK\eta^2((\frac{1}{n} \sum_{i=1}^n \eta_i^2)\sigma^2+2KG^2)
\\& \quad \quad \quad \quad \quad +32B^2nsK^2\eta^2 \sum_{t=0}^{T-1} \E\|\nabla f(\mu_t)\|^2 \big)
\\&=
80Tn^2({R}^2+7)^2\gamma^2 +
80Tn^2sK\eta^2((\frac{1}{n} \sum_{i=1}^n \eta_i^2)\sigma^2+2KG^2)
\\& \quad \quad \quad \quad \quad +160B^2n^2sK^2\eta^2 \sum_{t=0}^{T-1} \E\|\nabla f(\mu_t)\|^2
\end{align*}
\end{proof}

\begin{lemma} 
For any  step $t$, we have that 
\begin{align*}
 \sum_{i=1}^{n} \eta_i^2 \E \|\sum_{s=0}^{k-1}  \tih_{i,t}^{s}\|^2 \le 2nK(\frac{1}{n} \sum_{i=1}^n \eta_i^2)\sigma^2+8L^2K^2 \E[\Phi_t]
+4nK^2G^2
+4nK^2B^2 \E \|\nabla f(\mu_t)\|^2
\end{align*}
\end{lemma}
\begin{proof}
    Using lemma \ref{lem:sumofstochasticFixed}
\begin{align*}
 \sum_{i=1}^{n} \eta_i^2  \E \|\sum_{s=0}^{k-1}  \tih_{i,t}^{s}\|^2  &\le  nK(\frac{1}{n} \sum_{i=1}^n \eta_i^2)\sigma^2 + \sum_{i=1}^{n} \eta_i^2  \E \| \sum_{q=0}^{k-1}  h_{i,t}^{q}\|^2 \\&
\le nK(\frac{1}{n} \sum_{i=1}^n \eta_i^2)\sigma^2 + K \sum_{i=1}^{n} \eta_i^2   \sum_{q=0}^{k-1}  \E \|h_{i,t}^{q}\|^2 \\&
\le nK (\frac{1}{n} \sum_{i=1}^n \eta_i^2)\sigma^2 + K \sum_{i=1}^{n} \eta_i^2   \sum_{q=0}^{k-1} \Bigg(
\frac{\eta_i^2 \sigma^2}{K^2} + 8L^2 \E\|X_t^i-\mu_t\|^2 
+4\E\|\nabla f_i(\mu_t)\|^2
\Bigg)\\&
\le 2nK (\frac{1}{n} \sum_{i=1}^n \eta_i^2)\sigma^2 + \sum_{i=1}^n  \Bigg(
8K^2L^2 \E\|X_t^i-\mu_t\|^2 
+4K^2\E\|\nabla f_i(\mu_t)\|^2
\Bigg)\\&
\le
2nK (\frac{1}{n} \sum_{i=1}^n \eta_i^2)\sigma^2+8L^2K^2 \E[\Phi_t]
+4nK^2G^2
+4nK^2B^2 \E \|\nabla f(\mu_t)\|^2.
\end{align*}
\end{proof}

For the following lemmas we define parameter $\zeta_{i,k}$ as it is equal to 1 if the client $i$ did its kth step and otherwise is Zero.
\begin{lemma} 
For any step $t$
\begin{align*}
\E \| \sum_{i \in S} \eta_i \sum_{k=0}^{K}& \zeta_{i,k} \nabla f_i ( X_t^i - \sum_{s=0}^{k-1} \eta \tih_{i,t}^{s}) - \sum_{i \in S} \eta_i \sum_{k=0}^{K} \zeta_{i,k} \nabla f_i (X_t^i) \|^2 \\& 
 \le K ((\frac{1}{n} \sum_{i=1}^n \eta_i^2) \sigma^2+2KG^2)+ \frac{K^2L^2}{n} \E[\Phi_t]
+ 2K^2 B^2 \E \|\nabla f(\mu_t)\|^2.
\end{align*}
 \end{lemma}
\begin{proof}
 \begin{align*}
\E \| \sum_{i \in S}\eta_i & \sum_{k=0}^{K} \zeta_{i,k} \nabla f_i ( X_t^i - \sum_{s=0}^{k-1} \eta \tih_{i,t}^{s}) - \sum_{i \in S} \eta_i \sum_{k=0}^{K} \zeta_{i,k} \nabla f_i (X_t^i) \|^2   \\&  \le \E \Big[  sK \sum_{i \in S} \eta_i^2 \sum_{k=0}^{K} \zeta_{i,k}^2 \|\nabla f_i ( X_t^i - \sum_{s=0}^{k-1} \eta \tih_{i,t}^{s}) - \nabla f_i (X_t^i) \|^2 \Big]  \\&  \le \E \Big[  sKL^2 \sum_{i \in S} \eta_i^2 \sum_{k=0}^{K} \|X_t^i - \sum_{s=0}^{k-1} \eta \tih_{i,t}^{s} - X_t^i \|^2 \Big]
\\&  \le \E \Big[  sKL^2 \eta^2 \sum_{i \in S} \eta_i^2 \sum_{k=0}^{K} \|\sum_{s=0}^{k-1}  \tih_{i,t}^{s}\|^2 \Big] =   \frac{s^2KL^2 \eta^2}{n} \sum_{k=0}^{K} \sum_{i=1}^{n} \eta_i^2 \E \|\sum_{s=0}^{k-1}  \tih_{i,t}^{s}\|^2 \\&\le \frac{s^2K^2L^2 \eta^2}{n} (2nK(\frac{1}{n} \sum_{i=1}^n \eta_i^2)\sigma^2+8L^2K^2 \E[\Phi_t]
+4nK^2G^2
+4nK^2B^2 \E \|\nabla f(\mu_t)\|^2) \\& =  2s^2K^3L^2 \eta^2 ((\frac{1}{n} \sum_{i=1}^n \eta_i^2)\sigma^2+2KG^2)+ \frac{8s^2K^4L^4 \eta^2}{n} \E[\Phi_t]
+ 4s^2K^4L^2 \eta^2 B^2 \E \|\nabla f(\mu_t)\|^2
\\& \le K ((\frac{1}{n} \sum_{i=1}^n \eta_i^2)\sigma^2+2KG^2)+ \frac{K^2L^2}{n} \E[\Phi_t]
+ 2K^2 B^2 \E \|\nabla f(\mu_t)\|^2
\end{align*}
\end{proof}

\begin{lemma} 
For any step $t$
\begin{align*}
\E \| \sum_{i \in S} \eta_i \sum_{k=0}^{K} \zeta_{i,k}  (\nabla f_i(X_t^i) - \nabla f_i(\mu_t))\|^2 \le \frac{s^2K^2L^2}{n} \E[\Phi_t]
\end{align*}
 \end{lemma}
\begin{proof}
\begin{align*}
\E \| \sum_{i \in S} \eta_i \sum_{k=0}^{K} \zeta_{i,k}&  (\nabla f_i(X_t^i) - \nabla f_i(\mu_t))\|^2 \le sK \E \big[\sum_{i \in S} \sum_{k=0}^{K} \zeta_{i,k}^2 \|\nabla f_i(X_t^i) - \nabla f_i(\mu_t)\|^2] 
\\& \le \frac{s^2K^2}{n} \sum_{i=1}^{n} \E \|\nabla f_i(X_t^i) - \nabla f_i(\mu_t)\|^2
\\& \le \frac{s^2K^2L^2}{n} \sum_{i=1}^{n} \E \|X_t^i - \mu_t\|^2 =  \frac{s^2K^2L^2}{n} \E[\Phi_t]
\end{align*}
\end{proof}

\begin{lemma} 
For any step $t$
\begin{align*}
\E \| \sum_{i \in S} \eta_i \sum_{k=0}^{K} \zeta_{i,k} (\nabla f_i(\mu_t) - \nabla f(\mu_t)) \|^2 \le 2sK^2G^2 + 4sK^2B^2 \E\|\nabla f(\mu_t) \|^2
\end{align*}
 \end{lemma}
\begin{proof}
 \begin{align*}
 \E \| \sum_{i \in S} \eta_i & \sum_{k=0}^{K} \zeta_{i,k} (\nabla f_i(\mu_t) - \nabla f(\mu_t)) \|^2 
 = \E \Big[\sum_{i \in S} \eta_i^2 (\sum_{k=0}^{K} \zeta_{i,k})^2 \|(\nabla f_i(\mu_t) - \nabla f(\mu_t)) \|^2\Big]
\\& + \E \Big[\sum_{i\neq j \in S} \eta_i \eta_j(\sum_{k=0}^{K} \zeta_{i,k}) (\sum_{k=0}^{K} \zeta_{j,k}) \langle \nabla f_i(\mu_t) - \nabla f(\mu_t), \nabla f_i(\mu_t) - \nabla f(\mu_t) \rangle \Big] 
\\&= \E \Big[\sum_{i \in S} \eta_i^2 (\sum_{k=0}^{K} \zeta_{i,k})^2 \|(\nabla f_i(\mu_t) - \nabla f(\mu_t)) \|^2\Big]
\\&  \quad \quad  + \E \Big[\sum_{i\neq j \in S} \eta_i \eta_j \E[(\sum_{k=0}^{K} \zeta_{i,k})] \E[(\sum_{k=0}^{K} \zeta_{j,k})] \langle \nabla f_i(\mu_t) - \nabla f(\mu_t), \nabla f_i(\mu_t) - \nabla f(\mu_t) \rangle \Big]
\\&= \E \Big[\sum_{i \in S} \eta_i^2 (\sum_{k=0}^{K} \zeta_{i,k})^2 \|(\nabla f_i(\mu_t) - \nabla f(\mu_t)) \|^2\Big]
\\&  \quad \quad  + \E \Big[\sum_{i\neq j \in S} \eta_i H_i \eta_j H_j \langle \nabla f_i(\mu_t) - \nabla f(\mu_t), \nabla f_i(\mu_t) - \nabla f(\mu_t) \rangle \Big]
\\&= \E \Big[\sum_{i \in S} \eta_i^2 (\sum_{k=0}^{K} \zeta_{i,k})^2 \|(\nabla f_i(\mu_t) - \nabla f(\mu_t)) \|^2\Big]
\\&  \quad \quad  + H_{\min}^2 \E \Big[\sum_{i\neq j \in S} \langle \nabla f_i(\mu_t) - \nabla f(\mu_t), \nabla f_i(\mu_t) - \nabla f(\mu_t) \rangle \Big]
\\&\le \frac{{n-1 \choose s-1}}{{n \choose s}} \sum_{i=1}^{n}\E \Big[(\sum_{k=0}^{K} \zeta_{i,k})^2 \|(\nabla f_i(\mu_t) - \nabla f(\mu_t)) \|^2\Big]
\\&  \quad \quad  + H_{\min}^2 \frac{{n-2 \choose s-2}}{{n \choose s}}\sum_{i\neq j} \E \Big[\langle \nabla f_i(\mu_t) - \nabla f(\mu_t), \nabla f_i(\mu_t) - \nabla f(\mu_t) \rangle \Big]
\\&= \frac{s}{n} \sum_{i=1}^{n}\E \Big[(\sum_{k=0}^{K} \zeta_{i,k})^2 \|(\nabla f_i(\mu_t) - \nabla f(\mu_t)) \|^2\Big]
\\&  \quad \quad  - H_{\min}^2 \frac{{n-2 \choose s-2}}{{n \choose s}}\sum_{i=1}^{n} \E \Big[ \|(\nabla f_i(\mu_t) - \nabla f(\mu_t)) \|^2 \Big]
\\&= \frac{sK^2}{n} \sum_{i=1}^{n}\E \|(\nabla f_i(\mu_t) - \nabla f(\mu_t)) \|^2 \le \frac{2sK^2}{n} \sum_{i=1}^{n}\E \|\nabla f_i(\mu_t)\|^2 + 2sK^2 \E\|\nabla f(\mu_t) \|^2
\\&\le 2sK^2 (G^2 + B^2 \E\|\nabla f(\mu_t) \|^2 )+ 2sK^2 \E\|\nabla f(\mu_t) \|^2
\le 2sK^2G^2 + 4sK^2B^2 \E\|\nabla f(\mu_t) \|^2
\end{align*}
\end{proof}

\begin{lemma} 
For any step $t$
\begin{align*}
\E \|\sum_{i \in S} \eta_i \sum_{k=0}^{K} \zeta_{i,k} \nabla f_i (X_t^i)  \|^2 \le \frac{3s^2K^2L^2}{n} \E[\Phi_t] + 3sK((\frac{1}{n} \sum_{i=1}^n \eta_i^2) \sigma^2 + 2KG^2) + 15s^2K^2B^2 \E\|\nabla f(\mu_t) \|^2
\end{align*}
 \end{lemma}
\begin{proof}
 \begin{align*}
\E \|\sum_{i \in S} \eta_i \sum_{k=0}^{K} \zeta_{i,k}& \nabla f_i (X_t^i)  \|^2 = \E \| \sum_{i \in S} \eta_i \sum_{k=0}^{K} \zeta_{i,k}  (\nabla f_i(X_t^i) - \nabla f_i(\mu_t) + \nabla f_i(\mu_t) - \nabla f(\mu_t) + \nabla f(\mu_t) )\|^2 \\& \le 3 \E \| \sum_{i \in S} \eta_i \sum_{k=0}^{K} \zeta_{i,k}  (\nabla f_i(X_t^i) - \nabla f_i(\mu_t))\|^2 + 3 \E \| \sum_{i \in S} \eta_i \sum_{k=0}^{K} \zeta_{i,k} (\nabla f_i(\mu_t) - \nabla f(\mu_t)) \|^2  \\& \quad\quad\quad\quad +3 s^2K^2\E \| \nabla f(\mu_t) )\|^2
\\& \le  \frac{3s^2K^2L^2}{n} \E[\Phi_t] + 6sK^2G^2 + 12sK^2B^2 \E\|\nabla f(\mu_t) \|^2 + 3s^2K^2\E \| \nabla f(\mu_t) )\|^2 \\& \le \frac{3s^2K^2L^2}{n} \E[\Phi_t] + 6sK^2G^2 + 15s^2K^2B^2 \E\|\nabla f(\mu_t) \|^2
\\& \le \frac{3s^2K^2L^2}{n} \E[\Phi_t] + 3sK((\frac{1}{n} \sum_{i=1}^n \eta_i^2) \sigma^2 + 2KG^2) + 15s^2K^2B^2 \E\|\nabla f(\mu_t) \|^2
\end{align*}
\end{proof}

\begin{lemma} 
For any step $t$
\begin{align*}
\E  \|\sum_{i \in S} \eta_i h_{i,t}  \|^2 \le  \frac{8s^2K^2L^2}{n} \E[\Phi_t] + 8sK((\frac{1}{n} \sum_{i=1}^n \eta_i^2) \sigma^2 + 2KG^2) + 32s^2K^2B^2 \E\|\nabla f(\mu_t) \|^2 
\end{align*}
 \end{lemma}
\begin{proof}
 \begin{align*}
&\E  \|\sum_{i \in S} \eta_i h_{i,t}  \|^2 = \E \| \sum_{i \in S} \eta_i \sum_{k=0}^{K} \zeta_{i,k} h_{i,t}^{k} \|^2 = \E \| \sum_{i \in S} \eta_i \sum_{k=0}^{K} \zeta_{i,k} \nabla f_i ( X_t^i - \sum_{s=0}^{k-1} \eta \tih_{i,t}^{s}) \|^2 
\\& = \E \| \sum_{i \in S} \eta_i \sum_{k=0}^{K} \zeta_{i,k} \nabla f_i ( X_t^i - \sum_{s=0}^{k-1} \eta \tih_{i,t}^{s}) - \sum_{i \in S} \eta_i \sum_{k=0}^{K} \zeta_{i,k} \nabla f_i (X_t^i) + \sum_{i \in S} \eta_i \sum_{k=0}^{K} \zeta_{i,k} \nabla f_i (X_t^i)  \|^2 
\\& \le  2\E \| \sum_{i \in S} \eta_i \sum_{k=0}^{K} \zeta_{i,k} \nabla f_i ( X_t^i - \sum_{s=0}^{k-1} \eta \tih_{i,t}^{s}) - \sum_{i \in S} \eta_i \sum_{k=0}^{K} \zeta_{i,k} \nabla f_i (X_t^i) \|^2 + 2\E \|\sum_{i \in S} \eta_i \sum_{k=0}^{K} \zeta_{i,k} \nabla f_i (X_t^i)  \|^2 
\\&\le 2K ((\frac{1}{n} \sum_{i=1}^n \eta_i^2) \sigma^2+2KG^2)+ \frac{2K^2L^2}{n} \E[\Phi_t]
+ 2K^2 B^2 \E \|\nabla f(\mu_t)\|^2)  \\& + \frac{6s^2K^2L^2}{n} \E[\Phi_t] + 6sK((\frac{1}{n} \sum_{i=1}^n \eta_i^2) \sigma^2 + 2KG^2) + 30s^2K^2B^2 \E\|\nabla f(\mu_t) \|^2
\\&\le  \frac{8s^2K^2L^2}{n} \E[\Phi_t] + 8sK((\frac{1}{n} \sum_{i=1}^n \eta_i^2) \sigma^2 + 2KG^2) + 32s^2K^2B^2 \E\|\nabla f(\mu_t) \|^2
\end{align*}
\end{proof}

 \begin{lemma} 
For any step $t$
\begin{align*}
\E  \|\sum_{i \in S} \eta_i (\tih_{i, t} - h_{i,t})  \|^2 \le sK (\frac{1}{n} \sum_{i=1}^n \eta_i^2) \sigma^2
\end{align*}
 \end{lemma}
\begin{proof}
 \begin{align*}
\E  \|\sum_{i \in S} \eta_i (\tih_{i, t} - h_{i,t})  \|^2 &= \E \| \sum_{i \in S} \eta_i \sum_{k=0}^{K} \zeta_{i,k} (\tih_{i, t}^{k} - h_{i,t}^{k}) \|^2 = \E \Big[ \sum_{i \in S} \eta_i^2 \sum_{k=0}^{K} \zeta_{i,k}^2  \| (\tih_{i, t}^{k} - h_{i,t}^{k}) \|^2 \Big]  \\& +
 \E \Big[ \sum_{i,j \in S, (i,k) \neq (j,k')} \eta_i \eta_j \zeta_{i,k} \zeta_{j, k'}   \langle \tih_{i, t}^{k} - h_{i,t}^{k}, \tih_{j, t}^{k'} - h_{j,t}^{k'} \rangle \Big]  \\&  = \E \Big[ \sum_{i \in S} \eta_i^2 \sum_{k=0}^{K} \zeta_{i,k}^2  \| (\tih_{i, t}^{k} - h_{i,t}^{k}) \|^2 \Big] 
\le sK(\frac{1}{n} \sum_{i=1}^n \eta_i^2)\sigma^2
\end{align*}
\end{proof}

 \begin{replemma} {lem:sumofselectedstochasticG}
For any step $t$
\begin{align*}
 \E  \|\sum_{i \in S} \eta_i &\tih_{i, t}  \|^2 \le \frac{16s^2K^2L^2}{n} \E[\Phi_t] + 18sK((\frac{1}{n} \sum_{i=1}^n \eta_i^2)\sigma^2 + 2KG^2)  + 64s^2K^2B^2 \E\|\nabla f(\mu_t) \|^2  
 \end{align*}
 \end{replemma}
\begin{proof}
 \begin{align*}
 \E  \|\sum_{i \in S} \eta_i &\tih_{i, t}  \|^2 \le 
 2\E  \|\sum_{i \in S} \eta_i h_{i, t}  \|^2 + 2\E  \|\sum_{i \in S} \eta_i (\tih_{i, t} - h_{i,t})\|^2  
 \\& \le  2\Big( \frac{8s^2K^2L^2}{n} \E[\Phi_t] + 8sK((\frac{1}{n} \sum_{i=1}^n \eta_i^2)\sigma^2 + 2KG^2) + 32s^2K^2B^2 \E\|\nabla f(\mu_t) \|^2 \Big) + 2sK (\frac{1}{n} \sum_{i=1}^n \eta_i^2)\sigma^2 
 \\& \le \frac{16s^2K^2L^2}{n} \E[\Phi_t] + 18sK((\frac{1}{n} \sum_{i=1}^n \eta_i^2)\sigma^2 + 2KG^2)  + 64s^2K^2B^2 \E\|\nabla f(\mu_t) \|^2  
\end{align*}
\end{proof}

 \begin{lemma} \label {lem:mudifference1}
For any step $t$
\begin{align*}
\E\|\mu_{t+1}&-\mu_t\|^2 \le 
 \frac{2\eta^2}{(n+1)^2} \E  \|\sum_{i \in S} \eta_i \tih_{i, t}  \|^2  + \frac{2({R}^2+7)^2\gamma^2}{(n+1)^2} 
 \end{align*}
 \end{lemma}
\begin{proof}
 \begin{align*}
\E\|\mu_{t+1}&-\mu_t\|^2 = \frac{1}{(n+1)^2} \E \| -\eta \sum_{i \in S} \eta_i  \tih_{i, t} + \frac{Q(X_t)-X_t}{s+1}  \\&\quad \quad \quad +\frac{1}{s+1} \sum_{i \in S} (Q(X_t^i-\eta \eta_i \tih_{i, t}) - (X_t^i-\eta \eta_i \tih_{i, t})) \|^2 \\&\le 
 \frac{2}{(n+1)^2} \E  \| -\eta \sum_{i \in S} \eta_i \tih_{i, t}  \|^2 + \frac{2}{(n+1)^2}\E \| \frac{Q(X_t)-X_t}{s+1}  \\&\quad \quad \quad +\frac{1}{s+1} \sum_{i \in S} (Q(X_t^i-\eta \eta_i \tih_{i, t}) - (X_t^i-\eta \eta_i \tih_{i, t})) \|^2
\\& \le
 \frac{2}{(n+1)^2} \E  \| -\eta \sum_{i \in S}\eta_i \tih_{i, t}  \|^2  + \frac{2}{(n+1)^2} \Big( \frac{1}{s+1}\E\Big\| Q(X_t)-X_t  \Big\|^2 \\&\quad \quad \quad +\frac{1}{s+1} \sum_{i \in S} \E\Big\|  (Q(X_t^i-\eta \eta_i \tih_{i, t}) - (X_t^i-\eta \eta_i \tih_{i, t})) \Big\|^2 \Big) 
\\&\le 
 \frac{2\eta^2}{(n+1)^2} \E  \|\sum_{i \in S} \eta_i \tih_{i, t}  \|^2  + \frac{2({R}^2+7)^2\gamma^2}{(n+1)^2}  
\end{align*}
\end{proof}

 \begin{lemma} \label{lem:mudifference2}
For any step $t$
\begin{align*}
\E\|\mu_{t+1}&-\mu_t\|^2 \le 
 \frac{32\eta^2s^2K^2L^2}{n(n+1)^2} \E[\Phi_t] + \frac{36sK\eta^2}{(n+1)^2}((\frac{1}{n} \sum_{i=1}^n \eta_i^2) \sigma^2 + 2KG^2)  \\& \quad \quad \quad \quad \quad \quad
 +\frac{128\eta^2s^2K^2B^2 }{(n+1)^2}\E\|\nabla f(\mu_t) \|^2  +  \frac{2({R}^2+7)^2\gamma^2}{(n+1)^2} 
 \end{align*}
 \end{lemma}
\begin{proof}
 \begin{align*}
\E\|\mu_{t+1}&-\mu_t\|^2 \le 
 \frac{2\eta^2}{(n+1)^2} \E  \|\sum_{i \in S} \eta_i^2 \tih_{i, t}  \|^2  + \frac{2({R}^2+7)^2\gamma^2}{(n+1)^2} 
 \\& \le  \frac{2\eta^2}{(n+1)^2}\Big( \frac{16s^2K^2L^2}{n} \E[\Phi_t] + 18sK((\frac{1}{n} \sum_{i=1}^n \eta_i^2) \sigma^2 + 2KG^2) + 64s^2K^2B^2 \E\|\nabla f(\mu_t) \|^2 \Big) \\& +  \frac{2({R}^2+7)^2\gamma^2}{(n+1)^2} 
 \\& \le \frac{32\eta^2s^2K^2L^2}{n(n+1)^2} \E[\Phi_t] + \frac{36sK\eta^2}{(n+1)^2}((\frac{1}{n} \sum_{i=1}^n \eta_i^2) \sigma^2 + 2KG^2)  \\& \quad \quad \quad \quad \quad\quad
 +\frac{128\eta^2s^2K^2B^2 }{(n+1)^2}\E\|\nabla f(\mu_t) \|^2  +  \frac{2({R}^2+7)^2\gamma^2}{(n+1)^2} 
\end{align*}
\end{proof}

\subsection{Convergence} 
\label{appendix:full-convergence}
\begin{theorem} \label{thm:convergence}
For learning rate $\eta=\frac{n+1}{sH_{\min}\sqrt{T}}$, Algorithm \ref{algo:quafl} converges at rate:
\begin{align*}
  \frac{1}{T} \sum_{t=0}^{T-1} \E\|\nabla f(\mu_t)\|^2 &\le \frac{2(f(\mu_0)-f_*)}{\sqrt{T}}
    +\frac{800n KL^2({R}^2+7)^2\gamma^2}{H_{\min}}
    + \frac{6KL((\frac{1}{n} \sum_{i=1}^n \eta_i^2) \sigma^2 + 2KG^2)}{H_{\min}^2\sqrt{T}}
 \\& \quad\quad
    + \frac{808n(n+1)^2K^2L^2}{sH_{\min}^3T}((\frac{1}{n} \sum_{i=1}^n \eta_i^2) \sigma^2 + 2KG^2)
    +\frac{2({R}^2+7)^2\gamma^2L\sqrt{T}}{(n+1)^2sH_{\min}}
\end{align*}
\end{theorem}

\begin{proof}
Let $\E_t$ denote expectation conditioned on the entire history up to and including step $t$.
By $L$-smoothness we have that 
\begin{equation} \label{eqn:descentwithsecondmoment}
\E_t[f(\mu_{t+1})] \le f(\mu_t)+\E_t\langle\nabla f(\mu_t) , \mu_{t+1}-\mu_t\rangle+
    \frac{L}{2} \E_t\|\mu_{t+1}-\mu_t\|^2.
\end{equation} 

First we look at $\E_t\langle\nabla f(\mu_t) , \mu_{t+1}-\mu_t\rangle=
\langle\nabla f(\mu_t) , \E_t[\mu_{t+1}-\mu_t]\rangle$.
If set $S$ is chosen at step $t+1$,
We have that $$\mu_{t+1}-\mu_t=\frac{1}{n+1}(-\eta \sum_{i \in S} \eta_i \tih_{i, t} + \frac{Q(X_t)-X_t}{s+1} +\frac{1}{s+1} \sum_{i \in S} (Q(X_t^i-\eta \eta_i \tih_{i, t}) - X_t^i-\eta \eta_i \tih_{i, t})))$$
Thus, in this case:
\begin{align*}
\E_t[\mu_{t+1}-\mu_t]=-\frac{\eta}{n+1}\sum_{i \in S} \eta_i h_{i, t}.
\end{align*}
Where we used unbiasedness of quantization and stochastic gradients. 
We would like to note that even though we do condition on the entire history
up to and including step $t$ and this includes conditioning on $X_t^i$,
the algorithm has not yet used $\tih_{i, t}$ (it does not count towards computation of $\mu_t$), thus we can safely use all properties of stochastic gradients. Hence, we can proceed by taking into the account that each set of agents $S$ is chosen as initiator with probability $\frac{1}{{n \choose s}}$:
\begin{align*}
\E_t[\mu_{t+1}-\mu_t]=\sum_{S} \frac{1}{{n \choose s}} \sum_{i \in S} -\frac{\eta}{n+1} \eta_i h_{i, t} =  -\frac{s\eta}{n(n+1)}\sum_{i=1}^n \eta_i h_{i, t}.
\end{align*}
and subsequently 
\begin{align*}
\E_t\langle\nabla f(\mu_t) , \mu_{t+1}-\mu_t\rangle=\sum_{i=1}^n \frac{s\eta}{n(n+1)} \E_t\langle\nabla f(\mu_t) , -\eta_i h_{i, t}\rangle.
\end{align*}

Hence, we can rewrite (\ref{eqn:descentwithsecondmoment}) as:

\begin{align*}
\E_t[f(\mu_{t+1})] \le f(\mu_t)+\sum_{i=1}^n \frac{s\eta}{n(n+1)} \E_t\langle\nabla f(\mu_t) , -\eta_i h_{i, t}\rangle+
    \frac{L}{2} \E_t\|\mu_{t+1}-\mu_t\|^2.
\end{align*}
Next, we remove the conditioning
\begin{align*}
\E[(\mu_{t+1})]=\E[\E_t[f(\mu_{t+1})]] &\le \E[f(\mu_t)]+\sum_{i=1}^n \frac{s\eta}{n(n+1)} \E\langle\nabla f(\mu_t) , -\eta_i h_{i, t}\rangle
\\& \quad \quad \quad \quad \quad + \frac{L}{2} \E\|\mu_{t+1}-\mu_t\|^2.
\end{align*}
This allows us to use Lemmas \ref{lem:mudifference2} and \ref{lem:gradientdifference}:
\begin{align*}
\E[f(\mu_{t+1})] &- \E[f(\mu_t)] \le 
 \frac{s\eta}{n(n+1)} \bigg(  4KL^2\E[\Phi_t] + (-\frac{3H_{\min}n}{4} + 8B^2L^2\eta^2K^3n)\E\|\nabla f(\mu_t)\|^2\\& \quad \quad \quad \quad  \quad \quad +4nL^2\eta^2K^3((\frac{1}{n} \sum_{i=1}^n \eta_i^2) \sigma^2 + 2G^2) \bigg)
 \\& + \frac{L}{2}\bigg(\frac{36sK \eta^2 ((\frac{1}{n} \sum_{i=1}^n \eta_i^2) \sigma^2+2KG^2)}{(n+1)^2}
+\frac{32s^2L^2K^2\eta^2 \E[\Phi_t]}{n(n+1)^2} 
\\& \quad \quad \quad \quad \quad \quad
+\frac{128B^2s^2K^2 \eta^2 \E\|\nabla f(\mu_t)\|^2}{(n+1)^2}  +\frac{2({R}^2+7)^2\gamma^2}{(n+1)^2}
\bigg)
\\&= \big(\frac{4\eta sKL^2}{n(n+1)} + \frac{32s^2K^2L^3\eta^2}{n(n+1)^2}\big)\E[\Phi_t]
\\& \quad \quad \quad \quad +\big( \frac{4sL^2\eta^3K^3}{n+1} + \frac{18sK\eta^2L}{(n+1)^2}\big)((\frac{1}{n} \sum_{i=1}^n \eta_i^2) \sigma^2+2KG^2) + \frac{(R^2+7)^2\gamma^2L}{(n+1)^2} 
\\& \quad \quad \quad \quad \quad + \big(\frac{-3\eta sH}{4(n+1)} + \frac{8B^2L^2\eta ^3 sK^3}{n+1} + \frac{64B^2 s^2K^2 L \eta^2}{(n+1)^2} \big) \E\|\nabla f(\mu_t)\|^2
\end{align*}
By simplifying the above inequality we get: 
\begin{align*}
\E[f(\mu_{t+1})] &- \E[f(\mu_t)] \le \frac{5\eta sKL^2\E[\Phi_t]}{n(n+1)}+
\big( \frac{4sL^2\eta^3K^3}{n+1} + \frac{18sK\eta^2L}{(n+1)^2}\big)((\frac{1}{n} \sum_{i=1}^n \eta_i^2) \sigma^2+2KG^2)
\\& \quad \quad \quad \quad \quad  + \frac{(R^2+7)^2\gamma^2L}{(n+1)^2} 
+  \big(\frac{-3\eta sH_{\min}}{4(n+1)} + \frac{8B^2L^2\eta ^3 sK^3}{n+1} 
\\& \quad \quad \quad \quad \quad  \quad \quad + \frac{64B^2 s^2K^2 L \eta^2}{(n+1)^2} \big) \E\|\nabla f(\mu_t)\|^2
\end{align*}
\\
by summing the above inequality for $t=0$ to $t=T-1$, we get that
\begin{align*}
    \E[f(\mu_T)]-f(\mu_0) &\le \frac{5\eta sKL^2}{n(n+1)} \sum_{t=0}^{T-1} \E[\Phi_t] + \big( \frac{4sL^2\eta^3K^3T}{n+1} + \frac{18sK\eta^2LT}{(n+1)^2}\big)((\frac{1}{n} \sum_{i=1}^n \eta_i^2) \sigma^2+2KG^2)
\\&\quad\quad+\big(\frac{-3\eta sH_{\min}}{4(n+1)} + \frac{8B^2L^2\eta ^3 sK^3}{n+1} + \frac{64B^2 s^2K^2 L \eta^2}{(n+1)^2} \big) \sum_{t=0}^{T-1} \E\|\nabla f(\mu_t)\|^2
\\& \quad \quad  \quad \quad   +\frac{({R}^2+7)^2\gamma^2LT}{(n+1)^2} 
\end{align*}
Further, we use Lemma \ref{lem:PhiBoundGlobal}:
\begin{align*}
    \E[f(\mu_T)]-f(\mu_0) &\le \frac{5\eta sKL^2}{n(n+1)} \bigg(
    80Tn^2({R}^2+7)^2\gamma^2 +
80Tn^2sK\eta^2((\frac{1}{n} \sum_{i=1}^n \eta_i^2) \sigma^2+2KG^2)
\\& \quad\quad +160B^2n^2sK^2\eta^2  \sum_{t=0}^{T-1} \E\|\nabla f(\mu_t)\|^2
    \bigg)
    \\&\quad\quad  + \big( \frac{4sL^2\eta^3K^3T}{n+1} + \frac{18sK\eta^2LT}{(n+1)^2}\big)((\frac{1}{n} \sum_{i=1}^n \eta_i^2) \sigma^2+2KG^2)
 +\frac{({R}^2+7)^2\gamma^2LT}{(n+1)^2} +
    \\&\quad\quad    \big(\frac{-3\eta sH}{4(n+1)} + \frac{8B^2L^2\eta ^3 sK^3}{n+1} + \frac{64B^2 s^2K^2 L \eta^2}{(n+1)^2} \big) \sum_{t=0}^{T-1} \E\|\nabla f(\mu_t)\|^2
    \\& \le
    \frac{400\eta snK L^2T({R}^2+7)^2\gamma^2}{n+1} + \frac{404Tns^2K^2L^2\eta^3((\frac{1}{n} \sum_{i=1}^n \eta_i^2) \sigma^2+2KG^2)}{n+1}
 \\& \quad\quad
    + \frac{18sK\eta^2L T((\frac{1}{n} \sum_{i=1}^n \eta_i^2) \sigma^2 + 2KG^2)}{(n+1)^2}
    +\frac{({R}^2+7)^2\gamma^2LT}{(n+1)^2} \\&\quad\quad+\big(\frac{-3\eta sH_{\min}}{4(n+1)} + \frac{8B^2L^2\eta ^3 sK^3}{n+1}    \\&\quad\quad \quad\quad + \frac{64B^2 s^2K^2 L \eta^2}{(n+1)^2} + \frac{800B^2ns^2K^3\eta^3L^2}{n+1} \big)  \sum_{t=0}^{T-1} \E\|\nabla f(\mu_t)\|^2
\end{align*}
by assuming $\eta < \frac{1}{100B\sqrt{ns}k^2L}$ we get:
\begin{align*}
    \E[f(\mu_T)]-f(\mu_0) &\le 
    \frac{400\eta snK L^2T({R}^2+7)^2\gamma^2}{n+1} + 
 \\& \quad\quad
    + (\frac{18sK\eta^2L T}{(n+1)^2} + \frac{404Tns^2K^2L^2\eta^3}{n+1}) ((\frac{1}{n} \sum_{i=1}^n \eta_i^2) \sigma^2 + 2KG^2)
    +\frac{({R}^2+7)^2\gamma^2LT}{(n+1)^2} \\&\quad\quad+\frac{-\eta sH_{\min}}{2(n+1)}  \sum_{t=0}^{T-1} \E\|\nabla f(\mu_t)\|^2
\end{align*}
Next, we regroup terms,  multiply both sides by $\frac{2(n+1)}{\eta sH_{\min} T}$
and use the fact that $f(\mu_T) \ge f_*$:
\begin{align*}
  \frac{1}{T} \sum_{t=0}^{T-1} \E\|\nabla f(\mu_t)\|^2 &\le \frac{2(n+1)(f(\mu_0)-f_*)}{sH_{\min}\eta T}
    +\frac{800nK L^2({R}^2+7)^2\gamma^2}{H} + 
 \\& \quad\quad
    + (\frac{36K\eta L}{H_{\min}(n+1)} + \frac{808nsK^2L^2\eta^2}{H_{\min}})((\frac{1}{n} \sum_{i=1}^n \eta_i^2) \sigma^2 + 2KG^2)
    +\frac{2({R}^2+7)^2\gamma^2L}{(n+1)sH_{\min} \eta}
\end{align*}

Finally, we set $\eta=\frac{n+1}{H_{\min}\sqrt{sT}}$:

\begin{align} \label{eqn:finalboundonconvergence}
  \frac{1}{T} \sum_{t=0}^{T-1} \E\|\nabla f(\mu_t)\|^2 &\le \frac{2(f(\mu_0)-f_*)}{\sqrt{sT}}
    +\frac{800n KL^2({R}^2+7)^2\gamma^2}{H_{\min}}
    + \frac{36KL((\frac{1}{n} \sum_{i=1}^n \eta_i^2) \sigma^2 + 2KG^2)}{H_{\min}^2\sqrt{sT}}
 \\& \quad\quad
    + \frac{808n(n+1)^2K^2L^2}{H_{\min}^3T}((\frac{1}{n} \sum_{i=1}^n \eta_i^2) \sigma^2 + 2KG^2)
    +\frac{2({R}^2+7)^2\gamma^2L\sqrt{T}}{\sqrt{s}(n+1)^2}
\end{align}
\end{proof}

\begin{lemma} \label{lem:lemmamainconvergence}
For quantization parameters $(R^2 + 7)^2 \gamma^2 = 
\frac{(n+1)^2}{sH_{\min}^2T} ((\frac{1}{n} \sum_{i=1}^n \eta_i^2) \sigma^2 + 2KG^2 + \frac{f(\mu_0) - f_*}{L})$ we have:
\begin{align*}
  &\frac{1}{T} \sum_{t=0}^{T-1} \E\|\nabla f(\mu_t)\|^2 \le \frac{4(f(\mu_0)-f_*)}{\sqrt{sT}}
     + \frac{36KL((\frac{1}{n} \sum_{i=1}^n \eta_i^2) \sigma^2 + 2KG^2)}{H_{\min}^2\sqrt{sT}}
      \\& \quad\quad\quad\quad
    + \frac{1608n(n+1)^2K^2L^2((\frac{1}{n} \sum_{i=1}^n \eta_i^2) \sigma^2 + 2KG^2)}{H_{\min}^3T} 
          \\& \quad\quad\quad\quad\quad\quad
+ \frac{800n(n+1)^2KL(f(\mu_0) - f_*)}{sH_{\min}^3T} 
\end{align*}
\end{lemma}
\begin{proof}
\begin{align*}
   \frac{1}{T} \sum_{t=0}^{T-1} \E\|\nabla f(\mu_t)\|^2 &\le \frac{2(f(\mu_0)-f_*)}{\sqrt{sT}}
    +\frac{800n KL^2({R}^2+7)^2\gamma^2}{H_{\min}}
    + \frac{36KL((\frac{1}{n} \sum_{i=1}^n \eta_i^2) \sigma^2 + 2KG^2)}{H_{\min}^2\sqrt{sT}}
 \\& \quad\quad
    + \frac{808n(n+1)^2K^2L^2}{H^3T}((\frac{1}{n} \sum_{i=1}^n \eta_i^2) \sigma^2 + 2KG^2)
    +\frac{2({R}^2+7)^2\gamma^2L\sqrt{T}}{\sqrt{s}(n+1)^2}
    \\&= \frac{2(f(\mu_0)-f_*)}{\sqrt{sT}}
     + \frac{800nKL^2(n+1)^2}{sH_{\min}^3T} ((\frac{1}{n} \sum_{i=1}^n \eta_i^2) \sigma^2 + 2KG^2 + \frac{f(\mu_0) - f_*}{L})
  \\& \quad\quad
     + \frac{36KL((\frac{1}{n} \sum_{i=1}^n \eta_i^2) \sigma^2 + 2KG^2)}{H_{\min}^2\sqrt{sT}}
    + \frac{808n(n+1)^2K^2L^2}{H_{\min}^3T}((\frac{1}{n} \sum_{i=1}^n \eta_i^2) \sigma^2 + 2KG^2) \\&
     +\frac{2L}{s\sqrt{s}H_{\min}^2 \sqrt{T}} ((\frac{1}{n} \sum_{i=1}^n \eta_i^2) \sigma^2 + 2KG^2 + \frac{f(\mu_0) - f_*}{L})
     \\& \le \frac{4(f(\mu_0)-f_*)}{\sqrt{sT}}
     + \frac{36KL((\frac{1}{n} \sum_{i=1}^n \eta_i^2) \sigma^2 + 2KG^2)}{H_{\min}^2\sqrt{sT}}
      \\& \quad\quad
    + \frac{1608n(n+1)^2K^2L^2((\frac{1}{n} \sum_{i=1}^n \eta_i^2) \sigma^2 + 2KG^2)}{H_{\min}^3T} 
          \\& \quad\quad\quad\quad
+ \frac{800n(n+1)^2KL(f(\mu_0) - f_*)}{sH_{\min}^3T} 
\end{align*}
\end{proof}

\begin{lemma} \label{lem:upperBoundOnDistances}
We have:
\begin{align*}
5s \sum_{t=0}^{T-1}\E[\Phi_t] &+ 3\eta^2 \sum_{t=0}^{T-1} \sum_i \E\|\tih_{i, t}\|^2 \le 1000Tn^3s({R}^2+7)^2\gamma^2 
\\& +10000B^2n^3s^3H_{\min}^2K^2 LT ({R}^2+7)^2\gamma^2
\end{align*}
\end{lemma}
\begin{proof}
\begin{align*}
 &5s \sum_{t=0}^{T-1} \E[\Phi_t] + 3\eta^2 \sum_{t=0}^{T-1} \sum_i \E\|\tih_{i, t}\|^2 \\&\le 5s \sum_{t=0}^{T-1} \E[\Phi_t] + 3\eta^2 \sum_{t=0}^{T-1} \big ( 2nK((\frac{1}{n} \sum_{i=1}^n \eta_i^2) \sigma^2+2KG^2)+8L^2K^2 \E[\Phi_t]
+4nK^2B^2 \E \|\nabla f(\mu_t)\|^2 \big )
\\&\quad\quad \le 5s \sum_{t=0}^{T-1}\E[\Phi_t] + 6nT\eta^2 K((\frac{1}{n} \sum_{i=1}^n \eta_i^2) \sigma^2+2KG^2)+24\eta^2L^2K^2 \sum_{t=0}^{T-1} \E[\Phi_t]
\\& \quad \quad \quad \quad +12nB^2\eta^2K^2 \sum_{t=0}^{T-1} \E\|\nabla f(\mu_t)\|^2 
\\&\quad\quad \le 6s \sum_{t=0}^{T-1}\E[\Phi_t] + 6nT\eta^2 K((\frac{1}{n} \sum_{i=1}^n \eta_i^2) \sigma^2+2KG^2)
+12B^2n\eta^2K^2 \sum_{t=0}^{T-1} \E\|\nabla f(\mu_t)\|^2
\\&\quad\quad \le 6s \big(  80Tn^2({R}^2+7)^2\gamma^2 +
80Tn^2sK\eta^2((\frac{1}{n} \sum_{i=1}^n \eta_i^2) \sigma^2+2KG^2)
\\&\quad\quad\quad\quad+160B^2n^2sK^2\eta^2 \sum_{t=0}^{T-1} \E\|\nabla f(\mu_t)\|^2 \big) \\&\quad\quad\quad\quad + 6nT\eta^2 K((\frac{1}{n} \sum_{i=1}^n \eta_i^2) \sigma^2+2KG^2)
+12B^2n\eta^2K^2 \sum_{t=0}^{T-1} \E\|\nabla f(\mu_t)\|^2 
\\&\quad\quad \le 480Tn^2s({R}^2+7)^2\gamma^2 +
(480Tn^2s^2K\eta^2 + 6nT\eta^2 K)((\frac{1}{n} \sum_{i=1}^n \eta_i^2) \sigma^2+2KG^2)
\\&\quad\quad \quad\quad +(960n^2s^2K^2B^2\eta^2 + 
12B^2n\eta^2K^2) \sum_{t=0}^{T-1} \E\|\nabla f(\mu_t)\|^2    
\\&\quad\quad \le 480Tn^2s({R}^2+7)^2\gamma^2 +
486Tn^2s^2K\eta^2((\frac{1}{n} \sum_{i=1}^n \eta_i^2) \sigma^2+2KG^2)
\\&\quad\quad\quad\quad +1000B^2n^2s^2K^2\eta^2 \sum_{t=0}^{T-1} \E\|\nabla f(\mu_t)\|^2 
\\&\quad\quad \le 480Tn^2s({R}^2+7)^2\gamma^2 +
486Tn^2s^2K\eta^2((\frac{1}{n} \sum_{i=1}^n \eta_i^2) \sigma^2+2KG^2)
\\&\quad\quad \quad\quad 
 +1000B^2n^2s^2K^2\eta^2 \Big ( \frac{2(n+1)(f(\mu_0)-f_*)}{sH_{\min}\eta}
    +\frac{800TnK L^2({R}^2+7)^2\gamma^2}{H_{\min}} + 
 \\& \quad\quad
    + (\frac{36TK\eta L}{H_{\min}(n+1)} + \frac{808TnsK^2L^2\eta^2}{H_{\min}})((\frac{1}{n} \sum_{i=1}^n \eta_i^2) \sigma^2 + 2KG^2)
    +\frac{2T({R}^2+7)^2\gamma^2L}{(n+1)sH_{\min} \eta} \Big )
    \end{align*}
\begin{align*}
 &\quad\quad \le 480Tn^2s({R}^2+7)^2\gamma^2 +
486Tn^2s^2K\eta^2((\frac{1}{n} \sum_{i=1}^n \eta_i^2) \sigma^2+2KG^2)
\\&\quad\quad \quad\quad 
 +\frac{2000B^2n^2(n+1)sK^2\eta (f(\mu_0)-f_*)}{H_{\min}}
    +\frac{800000TB^2n^3s^2K^3\eta^2  L^2({R}^2+7)^2\gamma^2}{H_{\min}} + 
 \\& \quad\quad\quad\quad \quad\quad 
    + (\frac{36000TB^2n^2s^2K^3\eta^3L}{H_{\min}(n+1)} + \frac{808000TB^2n^3s^3K^4\eta^4L^2}{H_{\min}})((\frac{1}{n} \sum_{i=1}^n \eta_i^2) \sigma^2 + 2KG^2)
    \\&\quad\quad \quad\quad \quad\quad  +\frac{4000TB^2nsK^2\eta L  ({R}^2+7)^2\gamma^2}{H_{\min}}
\\&\quad\quad \le 1000Tn^3s({R}^2+7)^2\gamma^2 
 +\frac{2000B^2n^2(n+1)^2sK^2}{\sqrt{T}} (f(\mu_0)-f_*)
   \\&\quad\quad \quad\quad + \frac{10000Tn^3(n+1)^2s^2KL}{T}((\frac{1}{n} \sum_{i=1}^n \eta_i^2) \sigma^2 + 2KG^2)
\end{align*}
Therefore we have
\begin{align*}
&5s \sum_{t=0}^{T-1} \E[\Phi_t] + 3\eta^2 \sum_{t=0}^{T-1} \sum_i \E\|\tih_{i, t}\|^2 \le 1000Tn^3s({R}^2+7)^2\gamma^2 
   \\& + 10000B^2n^3(n+1)^2s^2K^2L(\frac{f(\mu_0)-f_*)}{L} + (\frac{1}{n} \sum_{i=1}^n \eta_i^2) \sigma^2 + 2KG^2)
\\&\quad\quad 
= 1000Tn^3s({R}^2+7)^2\gamma^2 
 +10000B^2n^3s^3H_{\min}^2K^2 LT ({R}^2+7)^2\gamma^2
\end{align*}
\end{proof}

\begin{lemma} \label{lem:quantfailure}
Let $T \ge O(n^3)$, then for quantization parameters $R=2+T^{\frac{3}{d}}$
and $\gamma^2=\frac{(n+1)^2 ((\frac{1}{n} \sum_{i=1}^n \eta_i^2) \sigma^2 + 2KG^2 + \frac{f(\mu_0)-f_*)}{L}}{sH_{\min}^2T(R^2+7)^2}$ we have that the probability 
of quantization never failing during the entire run of the Algorithm \ref{algo:quafl} is
at least $1-O\left(\frac{1}{T}\right)$.
\end{lemma}

\begin{proof}

Let $\mathcal{L}_t$ be the event that quantization does not fail during step $t$.
Our goal is to show that $Pr[\cup_{t=1}^T \mathcal{L}_{t}] \ge 1-O\left(\frac{1}{T}\right)$. In order to do this, we first prove that $Pr[\lnot \mathcal{L}_{t+1}|\mathcal{L}_1, \mathcal{L}_2, ..., \mathcal{L}_{t}] \le O\left(\frac{1}{T^2}\right)$ (O is with respect to $T$ here).

We need need to lower bound probability that :
\begin{align}
\forall i \in S: &\|X_t-X_t^i\|^2 \le ({R^{R}}^d\gamma)^2 \label{eqn:criterium1}\\
&\|X_t - (X_t^i -\eta \tih_{i, t})\|^2 \le ({R^{R}}^d\gamma)^2 \label{eqn:criterium2}\\
&\|X_t-X_t^i\|^2 = O\left(\frac{\gamma^2(poly(T))^2}{R^2}\right)\\
&\|X_t - (X_t^i -\eta \tih_{i, t})\|^2 = O\left(\frac{\gamma^2(poly(T))^2}{R^2}\right)
\end{align}
We would like to point out that these conditions are necessary for decoding to succeed,
we ignore encoding since it will be counted when someone will try to decode it. Since, $R=2+T^{\frac{3}{d}}$ this means that  $({R^{R}}^d)^2 \ge 2^{2{T^3}} \ge T^{30}$, for large enough $T$.
Hence, it is suffices to upper bound the probability that  
$\sum_{i \in S} \|X_t - X_t^i\|^2 + \sum_{i \in S} \|X_t - (X_t^i -\eta \tih_{i, t})\|^2 \ge T^{30}\gamma^2$. To prove this, we have:

\begin{align*}
&\sum_{i \in S} \|X_t - X_t^i\|^2 + \sum_{i \in S} \|X_t - (X_t^i -\eta \tih_{i, t})\|^2 
\le \sum_{i \in S} (5\|X_t - \mu_t\|^2 \\& \quad \quad + 5\|\mu_t - X_t^i\|^2 + 3\eta^2 \|\tih_{i, t}\|^2) 
 \le 5s \Phi_t + 3\eta^2 \sum_i \|\tih_{i, t}\|^2
\end{align*}

Now, we use Markov's inequality, and Lemma \ref{lem:upperBoundOnDistances}:
\begin{align*}
&Pr[5s \Phi_t + 3\eta^2 \|\tih_{i, t}\|^2 \ge T^{30}\gamma^2|\mathcal{L}_1, \mathcal{L}_2, ..., \mathcal{L}_{t}] \le \frac{\E[5s\Phi_t + 3\eta^2 \sum_{i} \|\tih_{i, t}\|^2|\mathcal{L}_1, \mathcal{L}_2, ..., \mathcal{L}_{t}]}{T^{30}\gamma^2} 
\\&\le \frac{1000Tn^3s({R}^2+7)^2\gamma^2 
 +10000B^2n^3s^3H_{\min}^2K^2 LT ({R}^2+7)^2\gamma^2}{T^{30}\gamma^2} 
 \le O(\frac{1}{T^2})
\end{align*}

Thus, the failure probability due to the models not being close enough
for quantization to be applied is at most $O\left(\frac{1}{T^2}\right)$.
Conditioned on the event that $\|X_t-X_t^i\|$ and
$\|X_t - (X_t^i -\eta \tih_{i, t})\|$ are upper bounded by $T^{15}\gamma$ (This is what we actually lower bounded the probability for using Markov),
we get that the probability of quantization algorithm failing is at most 
\begin{align*}
&\sum_{i \in S} \log\log(\frac{1}{\gamma} \|X_t -X_t^i\|)\cdot O(R^{-d}) \\&\quad\quad\quad\quad\quad\quad\quad+
\sum_{i \in S}\log\log(\frac{1}{\gamma} \|X_t - (X_t^i -\eta \tih_{i, t})\|)\cdot O(R^{-d}) \\&\quad\quad\quad\quad\quad\quad\quad
\le O\left(\frac{s \log\log{{T}}}{T^3}\right)\le O\left(\frac{1}{T^2}\right).
\end{align*}
By the law of total probability (to remove conditioning)  and the union bound we get that the total probability of failure, either due to not being able to apply quantization or by failure of quantization algorithm itself is at most $O\left(\frac{1}{T^2}\right)$.
Finally we use chain rule to get that 
\begin{align*}
Pr[\cup_{t=1}^T \mathcal{L}_{t}] &= \prod_{t=1}^T 
Pr[\mathcal{L}_{t}|\cup_{s=0}^{t-1} \mathcal{L}_{s}]=
\prod_{t=1}^T 
\Big(1-Pr[\neg \mathcal{L}_{t}|\cup_{s=0}^{t-1} \mathcal{L}_{s}]
\Big) \\&\ge 1-\sum_{t=1}^T Pr[\neg \mathcal{L}_{t}|\cup_{s=0}^{t-1} \mathcal{L}_{s}] \ge 1-O\left(\frac{1}{T}\right).
\end{align*}
\end{proof}

\begin{lemma} \label{lem:bitsquantized}
Let $T \ge O(n^3)$, then for quantization parameters $R=2+T^{\frac{3}{d}}$
and $\gamma^2=\frac{\eta^2}{(R^2+7)^2}((\frac{1}{n} \sum_{i=1}^n \eta_i^2) \sigma^2 + 2KG^2 + \frac{f(\mu_0) - f_*}{L})$ we have that the expected number of bits used by Algorithm \ref{algo:quafl} per communication
is $O(d \log (n) + \log(T))$.
\end{lemma}

\begin{proof}
 At step $t+1$, by Corollary \ref{cor:quant}, we know that the total number of bits used
is at most 
\begin{align*}
\sum_{i \in S} O\Big(d \log (\frac{R}{\gamma} \| X_t^i-X_t\|)\Big)+
O\Big(d \log (\frac{R}{\gamma}\| X_t-(X_t^i - \eta \tih_{i, t})\|\Big)
\end{align*}
By taking the randomness of agent interaction at step $t+1$ into the account,
we get that the expected number of bits used is at most:
\begin{align*} \label{eqn:randomnumberofbits}
&\sum_{S} \frac{1}{{n \choose s}} \sum_{i \in S} \Bigg(
O\Big(d \log (\frac{R}{\gamma} \| X_t^i-X_t\|)\Big)+
O\Big(d \log (\frac{R}{\gamma}\| X_t-(X_t^i - \eta \tih_{i, t})\|\Big)
\Bigg)
\\&=\sum_{i} \frac{s}{n} \Bigg(
O\Big(d \log (\frac{R}{\gamma} \| X_t^i-X_t\|)\Big)+
O\Big(d \log (\frac{R}{\gamma}\| X_t-(X_t^i - \eta \tih_{i, t})\|\Big)
\Bigg)
\\&=\le\sum_{i} \frac{s}{n} \Bigg(
O\Big(d \log (\frac{R^2}{\gamma^2} \| X_t^i-X_t\|^2)\Big)+
O\Big(d \log (\frac{R^2}{\gamma^2}\| X_t-(X_t^i - \eta \tih_{i, t})\|^2\Big)
\Bigg)
\\&\overset{Jensen}{\le}  s \Bigg(
O\Big(d \log (\frac{R^2}{\gamma^2} \sum_{i} \frac{1}{n} ( \| X_t^i-X_t\|^2 + \| X_t-(X_t^i - \eta \tih_{i, t})\|^2)\Big)
\Bigg)
\\ &{\le} 
s \Bigg(
O\Big(d \log (\frac{R^2}{\gamma^2} \sum_{i} \frac{1}{n} (\| X_t-\mu_t\|^2 + \| X_t^i-\mu_t\|^2 + \eta^2\| \tih_{i, t}\|^2)\Big)
\Bigg) 
\\&\le 
s \Bigg(
O\Big(d \log (\frac{R^2}{\gamma^2} (\Phi_t+ \frac{\eta^2}{n}\sum_{i}\| \tih_{i, t}\|^2)\Big)
\Bigg) 
\end{align*} 
So the expected number of bits per communication in all rounds is at most:
\begin{align*}
&\frac{1}{sT} \sum_{t=0}^{T-1} s \Bigg(
O\Big(d \log (\frac{R^2}{\gamma^2} (\Phi_t+ \frac{\eta^2}{n}\sum_{i}\| \tih_{i, t}\|^2)\Big)
\Bigg) 
\\& \quad\quad \quad\quad \le \Bigg(
O\Big(d \log (\frac{R^2}{\gamma^2} (\frac{1}{T} \sum_{t=0}^{T-1} \Phi_t + \frac{1}{T} \sum_{t=0}^{T-1} \frac{\eta^2}{n}\sum_{i}\| \tih_{i, t}\|^2)\Big)
\Bigg)
\end{align*} 
Next, By Jensen inequality and Lemma \ref{lem:upperBoundOnDistances},
We get that the expected number of bits used is at most,
\begin{align*}
&O\Big(d \E\Big[\log (\frac{R^2}{\gamma^2} (\frac{1}{T} \sum_{t=0}^{T-1} \Phi_t + \frac{1}{T} \sum_{t=0}^{T-1} \frac{\eta^2}{n}\sum_{i}\| \tih_{i, t}\|^2)\Big]\Big)
 \\&\overset{Jensen}{\le}
O\Big(d \log (\frac{R^2}{\gamma^2} (\frac{1}{T} \sum_{t=0}^{T-1} \E[\Phi_t] + \frac{1}{T} \sum_{t=0}^{T-1} \frac{\eta^2}{n}\sum_{i}\E\| \tih_{i, t}\|^2)\Big)
\\&\le
O\Big(d \log (\frac{R^2}{\gamma^2} (\frac{1}{T} (1000Tn^3s({R}^2+7)^2\gamma^2 
 +10000B^2n^3s^3H^2K^2 LT ({R}^2+7)^2\gamma^2)))\Big)
\\& \le
O\Big(d \log (R^2 (1000n^3s({R}^2+7)^2
 +10000B^2n^3s^3H^2K^2 L ({R}^2+7)^2))\Big)
= O(d \log (n) + \log(T))
\end{align*}
\end{proof}

\begin{reptheorem} {thm:quantized}
Assume the total number of steps $T \ge \Omega(n^3)$, the learning rate $\eta=\frac{n+1}{H_{\min}\sqrt{sT}}$, and quantization parameters $R=2+T^{\frac{3}{d}}$
and $\gamma^2=\frac{\eta^2}{(R^2+7)^2}\left((\frac{1}{n} \sum_{i=1}^n \eta_i^2) \sigma^2 + 2KG^2 + \frac{f(\mu_0) - f_*}{L}\right)$. Let $H_{\min} > 0$ be the minimum $H_i$. Then, with probability at least $1-O(\frac{1}{T})$ we have that  Algorithm \ref{algo:quafl} converges at the following rate
\begin{align*}
\frac{1}{T} \sum_{t=0}^{T-1} \E\|\nabla f(\mu_t)\|^2 &\le  \frac{4(f(\mu_0)-f_*)}{\sqrt{sT}}  
    +\frac{36KL((\frac{1}{n} \sum_{i=1}^n \eta_i^2) \sigma^2 + 2KG^2)}{H_{\min}^2\sqrt{sT}}
    + \\& O\left(\frac{n^3K^2L^2 ((\frac{1}{n} \sum_{i=1}^n \eta_i^2) \sigma^2 + 2KG^2)}{H_{\min}^3T}\right).
\end{align*}
and uses $O\left(sT(d\log{n} + \log T\right))$ expected communication bits in total.
\end{reptheorem}
\begin{proof}
The proof simply follows from combining Lemmas \ref{lem:lemmamainconvergence}, \ref{lem:quantfailure} and \ref{lem:bitsquantized}
\end{proof}

\begin{lemma} \label{lem:mainlemmaserver}
For the convergence of the server, we have:
\begin{align*}
\frac{1}{T} \sum_{t=0}^{T-1} \E\|\nabla f(X_t)\|^2 &\le \frac{12(f(X_0)-f_*)}{\sqrt{sT}}   + \frac{108KL((\frac{1}{n} \sum_{i=1}^n \eta_i^2) \sigma^2 + 2KG^2)}{H_{\min}^2\sqrt{sT}} \\& \quad\quad 
    + (\frac{4824n(n+1)^2K^2L^2}{H_{\min}^3T} + \frac{320n^2(n+1)^2KL^2}{H_{\min}^2T})((\frac{1}{n} \sum_{i=1}^n \eta_i^2) \sigma^2 + 2KG^2)
          \\& \quad\quad     \quad\quad
+ (\frac{2400n(n+1)^2KL}{sH_{\min}^3T} + \frac{160n^2(n+1)^2L^2}{sH_{\min}^2T}) (f(X_0) - f_*) 
\end{align*}
\end{lemma}
\begin{proof}
\begin{align*}
  &\frac{1}{T} \sum_{t=0}^{T-1} \E\|\nabla f(X_t)\|^2 \le  
  \frac{1}{T} \sum_{t=0}^{T-1} \E\|\nabla f(X_t) - \nabla f(\mu_t) +\nabla f(\mu_t)\|^2
  \\&\le  \frac{2}{T} \sum_{t=0}^{T-1} \E\|\nabla f(X_t) - \nabla f(\mu_t)\|^2 + \frac{2}{T} \sum_{t=0}^{T-1} \|\nabla f(\mu_t)\|^2
   \\&\le \frac{2L^2}{T} \sum_{t=0}^{T-1} \E\|X_t - \mu_t\|^2 + \frac{2}{T} \sum_{t=0}^{T-1} \|\nabla f(\mu_t)\|^2 
   \\&\le \frac{2L^2}{T} \sum_{t=0}^{T-1} \E[\Phi_t] + \frac{2}{T} \sum_{t=0}^{T-1} \|\nabla f(\mu_t)\|^2 
   \\&\le 2L^2 \big( 80n^2({R}^2+7)^2\gamma^2 +
80n^2sK\eta^2((\frac{1}{n} \sum_{i=1}^n \eta_i^2) \sigma^2+2KG^2)
\\& \quad \quad \quad +160B^2n^2sK^2\eta^2 \frac{1}{T}\sum_{t=0}^{T-1} \E\|\nabla f(\mu_t)\|^2 \big) + \frac{2}{T} \sum_{t=0}^{T-1} \|\nabla f(\mu_t)\|^2 
  \\&\le  160n^2L^2({R}^2+7)^2\gamma^2 +
160n^2sKL^2\eta^2((\frac{1}{n} \sum_{i=1}^n \eta_i^2) \sigma^2+2KG^2) +  \frac{3}{T} \sum_{t=0}^{T-1} \|\nabla f(\mu_t)\|^2 
  \\&\le  160n^2L^2({R}^2+7)^2\gamma^2 +
160n^2sKL^2\eta^2((\frac{1}{n} \sum_{i=1}^n \eta_i^2) \sigma^2+2KG^2) + \frac{12(f(\mu_0)-f_*)}{\sqrt{sT}}
      \\& \quad\quad\quad\quad + \frac{108KL((\frac{1}{n} \sum_{i=1}^n \eta_i^2) \sigma^2 + 2KG^2)}{H_{\min}^2\sqrt{sT}}
      \\& \quad\quad\quad\quad
    + \frac{4824n(n+1)^2K^2L^2((\frac{1}{n} \sum_{i=1}^n \eta_i^2) \sigma^2 + 2KG^2)}{H_{\min}^3T} 
+ \frac{2400n(n+1)^2KL(f(\mu_0) - f_*)}{sH_{\min}^3T} 
  \\&\le  160n^2L^2\eta^2 ((\frac{1}{n} \sum_{i=1}^n \eta_i^2) \sigma^2 + 2KG^2 + \frac{f(\mu_0) - f_*}{L} )+
160n^2sKL^2\eta^2((\frac{1}{n} \sum_{i=1}^n \eta_i^2) \sigma^2+2KG^2) +    \\& \quad\quad\quad\quad \frac{12(f(\mu_0)-f_*)}{\sqrt{sT}}
     + \frac{108KL((\frac{1}{n} \sum_{i=1}^n \eta_i^2) \sigma^2 + 2KG^2)}{H_{\min}^2\sqrt{sT}}
    \\& \quad\quad\quad\quad + \frac{4824n(n+1)^2K^2L^2((\frac{1}{n} \sum_{i=1}^n \eta_i^2) \sigma^2 + 2KG^2)}{H_{\min}^3T}
\\& \quad\quad\quad\quad \quad\quad 
    + \frac{2400n(n+1)^2KL(f(\mu_0) - f_*)}{sH_{\min}^3T} 
  \\&\le  \frac{12(f(X_0)-f_*)}{\sqrt{sT}}   + \frac{108KL((\frac{1}{n} \sum_{i=1}^n \eta_i^2) \sigma^2 + 2KG^2)}{H_{\min}^2\sqrt{sT}} \\& \quad\quad 
    + (\frac{4824n(n+1)^2K^2L^2}{H_{\min}^3T} + \frac{320n^2(n+1)^2KL^2}{H_{\min}^2T})((\frac{1}{n} \sum_{i=1}^n \eta_i^2) \sigma^2 + 2KG^2)
          \\& \quad\quad     \quad\quad
+ (\frac{2400n(n+1)^2KL}{sH_{\min}^3T} + \frac{160n^2(n+1)^2L^2}{sH_{\min}^2T}) (f(X_0) - f_*) 
   \end{align*}
\end{proof}

Finally, the
\textbf{proof of Corollary \ref{cor:quant}}
 follows from combining Lemmas \ref{lem:mainlemmaserver}, \ref{lem:quantfailure} and \ref{lem:bitsquantized}

\vfill

\end{document}